\documentclass{article}

\usepackage{iclr2027_conference,times}

\usepackage{amsmath,amssymb,amsthm,mathtools,bm}
\usepackage{graphicx}
\graphicspath{{figures/}{figures_paper/}{./}}
\usepackage{booktabs}
\usepackage{microtype}
\usepackage{enumitem}
\usepackage{xcolor}
\usepackage{url}
\usepackage{hyperref}

\newtheorem{theorem}{Theorem}
\newtheorem{proposition}[theorem]{Proposition}
\newtheorem{corollary}[theorem]{Corollary}
\newtheorem{lemma}[theorem]{Lemma}
\newtheorem{fact}[theorem]{Fact}
\theoremstyle{definition}

\theoremstyle{remark}

\newcommand{\E}{\mathbb E}
\newcommand{\R}{\mathbb R}
\newcommand{\cE}{\mathcal E}
\newcommand{\cA}{\mathcal A}
\newcommand{\eps}{\varepsilon}
\newcommand{\dd}{\,\mathrm d}
\newcommand{\KL}{D_{\mathrm{KL}}}
\newcommand{\pred}{\mathrm{pred}}

\title{The Cost of Long Memory: State, Context, and Stability Complexity in Sequence Models}

\newif\ifpreprint
\preprinttrue

\ifpreprint
  \iclrfinalcopy
  \author{%
    Yuheng Song \\
    School of Management, Huazhong University of Science and Technology \\
    \texttt{slark4560@gmail.com} \\
  }
\else
  \author{Anonymous Authors\\
  Paper under double-blind review}
\fi

\begin{document}
\maketitle
\ifpreprint\lhead{Preprint.}\fi

\begin{abstract}
Long-range temporal dependence poses a resource question for sequence models: for a specified predictive-memory law, how much state, context, or dynamical criticality is required in order to forecast accurately?
We study this question directly in forecasting risk.
For algebraically decaying predictive memory, we prove matching upper and lower approximation bounds for exponential and finite-state modes.
The best $r$-mode forecast error decays as $e^{-\Theta(\sqrt r)}$, so reaching forecast error $\tau$ needs $r=\Theta(\log^2(1/\tau))$ states or modes.
Earlier curse-of-memory results establish broad limitations of stable recurrent models under different approximation notions; here both sides match for one canonical predictive target in forecast risk, which fixes the optimal resource exponent for that target.
We then show that genuine fractional long memory changes the geometry itself.
In particular, forecast error is measured after fractional integration, prediction from a finite context of length $L$ has an exact $1/L$ leading order, and a fixed fractional strength $d$ keeps the square-log state-complexity law.
Near the short-memory boundary, we identify the relevant $d^2$ and $d^4$ scales and give a uniform constructive law in the intermediate regime.
For nonlinear contextual recurrences with uniformly contractive state dynamics, we derive an exponential first-chaos envelope and an explicit necessary condition that relates forecast accuracy to the contraction margin.
Vanishing forecasting error on an algebraic target forces the recurrence quantitatively toward criticality, a condition that is necessary and not by itself sufficient.
Finite-sample Kullback--Leibler calculations further connect the predictive geometry to statistical information.
Theorem-matched experiments with contractive state-space, gated recurrent, and attention models reproduce the state and stability predictions.
Matched exponential-memory controls locate the floor in the compatibility between the target tail and the contraction margin rather than in algebraic decay as a functional form, since an exponential target matched to the contraction is fit almost exactly.
\end{abstract}

\vspace{-1.0em}
\section{Introduction}
\label{sec:intro}

Long-context sequence modeling is often discussed in qualitative terms.
Some architectures are said to remember longer, some recurrent parameterizations are more stable, and attention can expose an arbitrarily large finite context.
For forecasting, long memory is not a binary property: a target process specifies a concrete predictive-memory law, and an architecture should be judged by the forecasting error it achieves with a finite amount of representational resource.

This paper studies three such resources.
\emph{State complexity} asks how many exponential modes or recurrent states are required to approximate a target memory to error $\tau$.
\emph{Context complexity} asks how much observed history is needed before finite-past prediction approaches the infinite-past Bayes risk.
\emph{Stability complexity} asks, for a nonlinear recurrent model constrained to stay uniformly contractive, how close to the stability boundary it must move before an algebraic predictive memory can be approximated accurately.

A growing literature already studies the curse of memory in recurrent and state-space models, robust long-memory parameterizations, and latent-state resources.
Given the well-known broad fact that uniformly stable recurrence tends to forget exponentially, we ask a sharper target-specific question.
\begin{center}
\emph{Given a predictive-memory law, what is the optimal forecasting error achievable with a fixed state dimension, context length, or stability margin?}
\end{center}
This target-specific view yields matching resource laws, so the exponent itself is fully determined.

\vspace{-0.7em}
\paragraph{Our results.}
We obtain four results.
Throughout, $(X_t)_{t\in\mathbb Z}$ is a centered stationary process and the task is one-step forecasting from its own past: a memory kernel $b=(b_j)_{j\ge1}$ gives the predictor $m_t(b)=\sum_{j\ge1}b_jX_{t+1-j}$, and the best such predictor is the Bayes forecast $m_t^\star=\E[X_{t+1}\mid X_s,\,s\le t]$, whose kernel $a^\star=(a^\star_j)_{j\ge1}$ we call the predictive memory.
The excess forecast risk of a candidate $b$ is the extra one-step mean squared error it incurs,
\begin{equation}
 \cE(b)
 =\E\bigl(X_{t+1}-m_t(b)\bigr)^{2}-\E\bigl(X_{t+1}-m_t^\star\bigr)^{2}
 =\E\Bigl(\sum_{j\ge1}(b_j-a^\star_j)X_{t+1-j}\Bigr)^{2},
 \label{eq:intro-risk}
\end{equation}
where the second equality holds because the Bayes error is orthogonal to the past.
This is a quadratic form in the coefficient error, and its geometry is fixed by the spectrum of $X$.
The architecture class $\cA_r$ consists of the memory kernels generated by an $r$-dimensional linear state and truncated to a finite context, with the $r$-mode exponential kernels as its diagonal subclass.

For stable algebraic predictive memory $a^\star_j=cj^{-\alpha}$ with decay exponent $\alpha>1$, we prove matching upper and lower bounds
\[
 c_1e^{-C_1\sqrt r}\le \inf_{b\in\cA_r}\cE(b)\le C_2e^{-c_2\sqrt r},
\]
and therefore, within this architecture class and for this target, the state or head complexity $r^\star(\tau)$, the smallest $r$ that reaches forecast error $\tau$, is $\Theta(\log^2(1/\tau))$ (Thm.~\ref{thm:state-complexity}).
The upper bound comes from exponential-sum compression and the lower bound from the Hankel spectrum of the target, so the resource question for this canonical long-memory target is settled on both sides.
For fractional long memory with a fixed fractional strength $d$, the coefficient error is no longer measured in the ordinary $\ell_2$-norm, since the forecast-risk norm applies a fractional-integration operator (Thm.~\ref{thm:fractional-geometry}).
In this geometry, the gap between prediction from the most recent $L$ observations and prediction from the infinite past has the sharp leading order $V_L-V_\infty\sim V_\infty d^2/L$, where $V_L$ and $V_\infty$ are the corresponding Bayes prediction variances.
A fixed $d$ keeps the square-log state-complexity law, and the short-memory limit $d\to0$ exposes the $d^2$ and $d^4$ boundary scales and admits a uniform constructive law between them (Corollaries~\ref{cor:context-law} and~\ref{cor:frac-state} and Thm.~\ref{thm:small-d}).
For an innovation-driven genuinely nonlinear contextual recurrence whose state Jacobian contracts uniformly with factor $\rho<1$ and whose input and readout derivatives are bounded by $K$ and $M$, the lag-$j$ first-chaos memory coefficient $\beta_j$, the projection of the predictor onto the $j$th past innovation, satisfies $|\beta_j|\le MK\rho^{j-1}$ for $j\ge2$, which yields an explicit forecasting floor on algebraic targets and the necessary condition $-\log\rho\lesssim\tau^{1/(2\alpha-1)}\log(e/\tau)$ (Thm.~\ref{thm:first-chaos} and Corollary~\ref{cor:critical}).
This is a forecasting-specific quantitative refinement of the established curse of memory.
Finally, in Gaussian causal models, the trajectory-level conditional Kullback--Leibler divergence is exactly the accumulated forecast risk, and for fractional stationary blocks we control the finite-section Fisher--Hartwig boundary and extend the logarithmic correction to any fixed number of reflection-symmetric moving or merging fractional poles under a cluster-integrability margin (Prop.~\ref{prop:kl-risk} and discussions).

\vspace{-0.3em}
\subsection{Related works}
\label{sec:related-main}

\vspace{-0.3em}
\paragraph{The curse of memory in recurrent and state-space models.}
A substantial literature shows that stable recurrent and state-space models face long-memory limitations.
These results formalize related obstructions under different notions rather than one common metric.
\citet{li2021curse} work with worst-case functional approximation and the optimization dynamics that accompany it, \citet{wang2024inverse} establish inverse-approximation results in that same worst-case sense for nonlinear recurrent networks, and \citet{wang2023ssmnonlinear} characterize exponential decay for stacked state-space models through a derivative-based memory function.
All three are distinct from the stationary forecast-risk geometry used here.
StableSSM relates stable approximation to memory decay and proposes reparameterizations that ease optimization \citep{wang2024stablessm}.
HOPE uses Hankel coordinates to obtain a robust long-memory state-space parameterization \citep{yu2025hope}, initialization and timescale choices are shown to control the effective memory of a diagonal state \citep{liu2025autocorrelation}, and the transfer function of a linear time-invariant layer carries a frequency bias that can be tuned \citep{yu2025frequency}.
Selective architectures are known to have a receptive field shorter than their nominal context.
DeciMamba and LongMamba enlarge that receptive field in pretrained models without additional training \citep{benkish2025decimamba,ye2025longmamba}, whereas Block-Biased Mamba instead modifies the S6 unit itself, combining block-wise selective dynamics with a channel-specific bias to improve expressiveness, inductive bias, and training stability on long-range tasks \citep{yu2025blockbiased}.
Our contribution, however, is a forecast-risk envelope whose constants carry both the target tail and the contraction margin, so that the obstruction can be inverted into an explicit accuracy-to-criticality requirement.

\vspace{-0.7em}
\paragraph{State resources and long-context modeling.}
Recent theory asks how the latent-state dimension should scale with long-context information.
Information-theoretic arguments lower bound the latent state a model needs in order to carry the mutual information present in long contexts \citep{chen2025l2m}, and the architectures for which that question is asked are the modern state-space and gated recurrent families \citep{gu2022s4,smith2023s5,gu2023mamba,dao2024mamba2,yang2024gla,yang2025gateddelta}.
Our formulation fixes the target predictive-memory law instead and asks for the optimal approximation error as a function of state dimension, which complements those results.
This produces the explicit root-exponential curve and its square-log inverse.

\vspace{-0.7em}
\paragraph{Power-law and fractional architectures.}
Power-law memories can be implemented through sums of exponentials and have recently appeared in long-context architectures.
For example, VORT uses adaptive power-law memory together with recurrent exponential compression \citep{mlaiki2026vort}.
We study a different metric and a different asymptotic question, in which the horizon is not fixed externally, and approximation is evaluated in forecast risk over the infinite predictive law.
Classical exponential-sum approximation results \citep{beylkin2005exp,beylkin2010revisited,braess2005oneoverx} supply our constructive tools, together with quadrature designed for fractional kernels \citep{baffet2019gaussjacobi}.
The matching lower bound combines the Eckart--Young principle for unconstrained low-rank approximation in Frobenius norm \citep{eckart1936rank} with the singular-value asymptotics of Hankel operators \citep{widom1966hankel}, for which \citet{aak1971schmidt,peller2003hankel} provide the structured approximation background.

\vspace{-0.7em}
\paragraph{Fractional prediction and Fisher--Hartwig theory.}
Fractionally differenced processes are classical \citep{grangerjoyeux1980,hosking1981fractional,beran1994longmemory}, and the finite-past prediction and partial-autocorrelation theory we use is developed by \citet{inoue2002pacf,inoue2006finite}.
Asymptotic estimation of the memory parameter is a separate line \citep{foxtaqqu1986,giraitis1997rateoptimal,lieberman2004fractional,hualde2011fractional}.
When the location of the pole is itself unknown, the problem changes character \citep{giraitis2001unknownpole,hidalgo2005unknownpole,arteche2020multiple}.
Classical Toeplitz theory supplies the operator and determinant background \citep{grenander1958toeplitz}, while Fisher--Hartwig asymptotics treat symbols with singularities \citep{ehrhardt1997onefh,deift2011fisherhartwig}.
The uniform and merging-singularity regimes are what our fixed-pole-count statement consumes \citep{claeys2015merging,fahs2021uniform}.
Our contribution is the sequence-model resource interpretation, in which the singular spectrum changes the predictive norm, the context law, and the state or head complexity relevant to forecasting.

\vspace{-0.5em}
\section{Predictive Memory as a Resource Problem}
\label{sec:setup}
\vspace{-0.3em}
\subsection{Forecast-risk geometry}
\vspace{-0.3em}
Let the Bayes one-step predictor have memory $a^\star=(a_j^\star)_{j\ge1}$ and let $\cE(a)$ denote the excess forecast risk \eqref{eq:intro-risk}.
Write $A(z)=\sum_{j\ge1}a_jz^j$ for the transfer function of a coefficient sequence $a$, and $A^\star$ for that of $a^\star$.
For a stationary process with spectral density $f_X$, Parseval turns \eqref{eq:intro-risk} into the spectral form
\begin{equation}
 \cE(a)=\frac1{2\pi}\int_{-\pi}^{\pi}
 |A(e^{-i\omega})-A^\star(e^{-i\omega})|^2f_X(\omega)\dd\omega.
 \label{eq:risk}
\end{equation}
We write $\|u\|_{\pred,f_X}^2$ for the same quadratic form applied to a coefficient perturbation $u$, so that $\cE(a)=\|a-a^\star\|_{\pred,f_X}^2$, and call it the predictive norm.
Under a regular spectrum, meaning $0<m\le f_X\le M<\infty$, this equals constant coefficient $\ell_2^2$ error.
Under genuine fractional long memory, the upper equivalence fails, and Sec.~\ref{sec:fractional} identifies the correct singular geometry.

\vspace{-0.3em}
\subsection{Resources, targets, and regimes}
\vspace{-0.3em}
Fix the architecture class $\cA_r$ of $r$-state kernels truncated to a finite context, taken as a union over contexts, and write its best forecast error as $\inf_{b\in\cA_r}\cE(b)$. 
Thm.~\ref{thm:state-complexity} states the class precisely, and App.~\ref{app:prelim} verifies its Hankel-rank property.
The state-complexity function is
\[
 r^\star(\tau)=\min\{r:\inf_{b\in\cA_r}\cE(b)\le\tau\},
\]
and $r_d^\star(\tau)$ denotes the same quantity with $\cE$ replaced by the fractional risk $\cE_d$ of Sec.~\ref{sec:fractional}.
For finite context, let $V_L$ be the Bayes prediction variance from the most recent $L$ observations and $V_\infty$ its infinite-past limit.
For a nonlinear recurrence, the stability resource is the contraction margin $\delta=-\log\rho$, where $\|D_sF\|_{\rm op}\le\rho<1$.
Two target families carry these questions: algebraic memory $a^\star_j=cj^{-\alpha}$ with amplitude $c$ and exponent $\alpha>1$, and fractional memory of strength $d\in(0,1/2)$.
Sec.~\ref{sec:state}, \ref{sec:fractional}, and~\ref{sec:stability} analyze the three resources in turn, and Table~\ref{tab:resources} states the resulting laws.

The three questions are asked in different regimes.
Sec.~\ref{sec:state} and~\ref{sec:stability} assume a stationary purely nondeterministic process with a regular spectrum, which is what keeps forecast risk comparable to coefficient error.
Sec.~\ref{sec:fractional} deliberately drops that assumption for a single spectral pole with a bounded regular factor.
Sec.~\ref{sec:stability} assumes the derivative bounds on a forward-invariant domain containing the stationary solution.
Full statements are in App.~\ref{app:prelim}, \ref{app:fractional}, and~\ref{app:stability}.

\begin{table}[t]
\centering
\footnotesize
\begin{tabular}{@{}p{1.05in}p{2.05in}p{1.90in}@{}}
\toprule
Resource & Our statement & Our contribution \\
\midrule
State dimension $r$
 & $\inf_{b\in\cA_r}\cE(b)=e^{-\Theta(\sqrt{r})}$, so $r^\star(\tau)=\Theta(\log^2(1/\tau))$, and a fractional pole leaves this intact, $r_d^\star(\tau)=\Theta_d(\log^2(1/\tau))$ (Thm.~\ref{thm:state-complexity}, Corollary~\ref{cor:frac-state})
 & matching two-sided bounds for one specified target, measured in forecast risk, and transported to the singular geometry \\
\addlinespace
Context length $L$
 & regular algebraic: $\cE(a^{(L)})=\Theta(L^{1-2\alpha})$; fractional: $V_L-V_\infty=V_\infty d^2/L+o(1/L)$ (Corollary~\ref{cor:context-law})
 & the singular predictive norm changes the context law; root-exponential state accuracy requires and admits $L=e^{\Theta(\sqrt r)}$ up to constants \\
\addlinespace
Stability margin $\delta=-\log\rho$
 & $\delta\lesssim\tau^{1/(2\alpha-1)}\log(e/\tau)$, necessary only (Corollary~\ref{cor:critical})
 & an explicit accuracy-to-criticality law for a fixed algebraic target \\
\bottomrule
\end{tabular}
\caption{The three resource laws. The state law is two-sided, the context law is exact at leading order, and the stability law is necessary and not sufficient. The related-work discussion of the introduction places each against the nearest existing line.}
\label{tab:resources}
\vspace{-1.2em}
\end{table}

\section{Sharp State Complexity of Algebraic Memory}
\label{sec:state}

A recurrent layer carrying an $r$-dimensional linear state, with transition matrix $A\in\R^{r\times r}$, input map $v\in\R^{r}$, and readout $u\in\R^{r}$, generates the memory coefficients $b_j=u^\top A^{j-1}v$, and therefore finite Hankel matrices of rank at most $r$; the resource $r$ can be read as the number of exponential temporal heads, diagonal recurrent modes, or linear state dimensions inside that class.
A positive algebraic memory, on the other side, admits a Laplace representation and can be compressed by quadrature into a sum of exponentials.
These two facts are what make the two directions meet, the rank constraint bounding every candidate from below and the quadrature construction supplying a candidate from above.

\begin{theorem}[Sharp state complexity of algebraic predictive memory]
\label{thm:state-complexity}
Consider a stable regular-spectrum forecasting model with target memory $a_j^\star=cj^{-\alpha}$, $\alpha>1$.
Let $\cA_r$ denote the union, over finite contexts, of $r$-dimensional linear-state kernels restricted to that context, the $r$-mode exponential kernels forming the diagonal subclass.
There exist constants $c_1,c_2,C_1,C_2>0$ and $r_0$, depending only on the fixed model, such that for $r\ge r_0$,
\begin{equation}
 c_1e^{-C_1\sqrt r}
 \le \inf_{b\in\cA_r}\cE(b)
 \le C_2e^{-c_2\sqrt r}.
 \label{eq:rootexp}
\end{equation}
Equivalently, we have
\begin{equation}
 r^\star(\tau)=\Theta_\alpha\!\left(\log^2\frac1\tau\right)
 \qquad(\tau\to0).
 \label{eq:log2}
\end{equation}
The same exponent holds for the analytic algebraic Laplace-mixture class stated in App.~\ref{app:state}.
\end{theorem}

Curse-of-memory and inverse-approximation results establish limitations of stable recurrence on slowly decaying targets \citep{li2021curse,wang2024inverse}, and modern state-space work develops parameterizations aimed at improving long-memory representation and optimization \citep{wang2024stablessm,yu2025hope}.
Those statements are made under different approximation notions and over different target classes, and App.~\ref{app:state} separates the three notions before proving anything.
Thm.~\ref{thm:state-complexity} settles a narrower question completely: for one canonical algebraic predictive target measured in forecast risk, both sides of the approximation curve match, which fixes the resource exponent rather than establishing only that the resource must grow.
The class is the linear-state one just described, and Sec.~\ref{sec:stability} treats a broader nonlinear class under an explicit contraction hypothesis.

For the upper bound, the Laplace representation of $j^{-\alpha}$ and positive exponential quadrature give a finite-context relative approximation with $r=O(\log^2(1/\tau))$ modes.
For the lower bound, every $r$-state candidate has Hankel rank at most $r$, while finite sections of the algebraic Hankel operator have singular values on a root-exponential scale.
A corrected Ky Fan comparison transfers that spectral obstruction into forecast risk.
App.~\ref{app:state} contains the full proof and constants.

\section{Singular Long Memory and Context Complexity}
\label{sec:fractional}

\subsection{Fractional benchmark and predictive geometry}
Consider FARIMA$(0,d,0)$, where $B$ is the backshift operator $(BX)_t=X_{t-1}$ and $(1-B)^d$ is its binomial fractional power,
\begin{equation}
 (1-B)^dX_t=\eps_t,
 \qquad \eps_t\stackrel{\rm iid}{\sim}N(0,\sigma^2),
 \qquad 0<d<\tfrac12,
 \label{eq:farima}
\end{equation}
with spectral density $f_d(\omega)=\sigma^2|1-e^{-i\omega}|^{-2d}$.
Its autoregressive predictive-memory coefficients satisfy
\[
 a_j(d)=\frac{d\,\Gamma(j-d)}{\Gamma(1-d)\Gamma(j+1)}
 \sim \frac{d}{\Gamma(1-d)}j^{-1-d}.
\]
Let $I_d$ be the discrete fractional-integration map with transfer function $(1-z)^{-d}$, and write $\cE_d$ for the excess forecast risk~\eqref{eq:risk} taken under the fractional model~\eqref{eq:farima}.

\begin{theorem}[Fractional predictive geometry]
\label{thm:fractional-geometry}
Suppose
\[
 f_X(\omega)=|1-e^{-i\omega}|^{-2d}g(\omega),
 \qquad 0<m_g\le g(\omega)\le M_g<\infty.
\]
For every finitely supported coefficient perturbation $u$,
\begin{equation}
 m_g\|I_du\|_2^2
 \le \|u\|_{\pred,f_X}^2
 \le M_g\|I_du\|_2^2.
 \label{eq:frac-geometry}
\end{equation}
For pure fractional noise the identity is exact up to the factor $\sigma^2$,
\[
 \cE_d(a(d)+u)=\sigma^2\|I_du\|_2^2.
\]
\end{theorem}

Under a spectral pole the relevant geometry departs from ordinary coefficient error, because low-frequency perturbations are amplified by fractional integration.
The two-sided bound transports the boundedness of $g$ through the fractional-integration isometry, and the pure-noise identity is exact because the fractional density carries no additional regular factor.
The full argument is in App.~\ref{app:fractional}.

\subsection{Consequences for the two resources}
The change of geometry moves one resource law and leaves the other in place.
\begin{corollary}[Sharp finite-context prediction law]
\label{cor:context-law}
For the fractional model~\eqref{eq:farima},
\begin{equation}
 V_L-V_\infty
 =V_\infty\frac{d^2}{L}+o(L^{-1}).
 \label{eq:one-over-L}
\end{equation}
For a fixed finite set of separated fractional poles with a sufficiently smooth positive short-memory envelope, the leading coefficient becomes the corresponding sum of squared pole strengths.
\end{corollary}
The $1/L$ law differs from the naive unweighted coefficient-tail calculation, because the predictive metric itself is singular.
The leading order comes from the partial autocorrelations of the fractional process evaluated in the predictive norm, and App.~\ref{app:fractional} gives the proof.

\begin{corollary}[Fractional state complexity]
\label{cor:frac-state}
For every fixed $d\in(0,1/2)$, the best $r$-dimensional linear-state approximation in fractional forecast risk satisfies
\begin{equation}
 c_de^{-C_d\sqrt r}
 \le \inf_{b\in\cA_r}\cE_d(b)
 \le C_de^{-c_d\sqrt r},
 \qquad
 r_d^\star(\tau)=\Theta_d\!\left(\log^2\frac1\tau\right).
 \label{eq:frac-log2}
\end{equation}
\end{corollary}

The upper bound is attained inside the diagonal subclass, by a positive $r$-mode exponential construction obtained from exponential quadrature applied to the fractionally reweighted kernel, and the lower bound reads off the fractional Hankel spectrum; both are carried out in App.~\ref{app:fractional}.

\subsection{The short-memory boundary}
The fixed-$d$ constants in~\eqref{eq:frac-log2} cannot be held fixed as $d\to0$.
In absolute forecast risk the memory itself disappears,
\begin{equation}
 \cE_d(0)=\sigma^2\left\{\frac{\Gamma(1-2d)}{\Gamma(1-d)^2}-1\right\}
 =\frac{\pi^2\sigma^2}{6}d^2+O(d^3).
 \label{eq:zero-boundary}
\end{equation}
Hence $O(1)$ memory suffices once $\tau\gtrsim\sigma^2d^2$.
In the intermediate window we have a uniform constructive refinement.
\begin{theorem}[Uniform intermediate small-$d$ upper bound]
\label{thm:small-d}
Fix $d_0<1/2$.
There are constants $c_0,C_0,C$ such that, whenever
\[
 0<d\le d_0,
 \qquad
 C_0\sigma^2d^4\le\tau\le c_0\sigma^2d^2,
\]
there is a truncated positive exponential kernel with forecast risk at most $\tau$ and
\begin{equation}
 r_d^\star(\tau)
 \le C\left[1+\log\left(\frac{\sigma^2d^2}{\tau}\right)\right]^2.
 \label{eq:small-d}
\end{equation}
\end{theorem}
The construction linearizes the fractional coefficients as $a_j(d)=d/j+O(d^2)$ and truncates a positive harmonic exponential kernel, with the details in App.~\ref{app:fractional}.
A matching uniform lower law in this boundary layer, and the deeper regime $\tau\ll\sigma^2d^4$, remain open.

\section{Quantitative Stability Barriers for Recurrent Memory}
\label{sec:stability}

\subsection{Relation to existing curse-of-memory theory}
Several formal results constrain stable recurrent approximation of slowly decaying targets, under different notions and for different model classes.
\citet{li2021curse} study worst-case approximation and the optimization limitations that accompany it for linear recurrent models, \citet{wang2024inverse} give inverse-approximation results in the same worst-case sense for nonlinear recurrent networks, and \citet{wang2023ssmnonlinear} establish exponential memory decay for stacked state-space models through a derivative-based memory function.
The notions are not interchangeable, and none of them is the stationary forecast risk used here.
StableSSM and HOPE are representative modern responses to this difficulty \citep{wang2024stablessm,yu2025hope}.
We derive a forecast-risk lower bound whose constants expose the target tail and the contraction margin, which lets the qualitative obstruction be inverted into an explicit requirement linking accuracy and criticality.

\subsection{Contextual nonlinear recurrence}
The predictors in this section are no longer given by a coefficient kernel, so for a general predictor $\widehat m_t$ we keep the same symbol and write $\cE(\widehat m):=\E[(m_t^\star-\widehat m_t)^2]$ for its excess forecasting risk in this innovation-driven noiseless Gaussian benchmark, where $m_t^\star$ is a fixed target functional of the innovation sequence, specialized to the algebraic form in Thm.~\ref{thm:first-chaos}.
Let $Z_t\stackrel{\rm iid}{\sim}N(0,1)$ and consider
\[
 S_t=F(S_{t-1},Z_t),
 \qquad
 \widehat m_t=G(S_t,Z_t),
\]
with a stationary solution and derivative budgets on its forward-invariant state domain
\[
 \|D_sF(s,z)\|_{\rm op}\le\rho<1,
 \qquad
 \|D_zF(s,z)\|\le K,
 \qquad
 \|D_sG(s,z)\|\le M.
\]

\begin{theorem}[Uniform contraction gives an exponential first-chaos envelope]
\label{thm:first-chaos}
Define the lag-$j$ first-chaos coefficient
\[
 \beta_j=\E[\widehat m_t Z_{t+1-j}].
\]
Under the regularity conditions in App.~\ref{app:stability},
\begin{equation}
 |\beta_j|\le MK\rho^{j-1},\qquad j\ge2.
 \label{eq:first-chaos}
\end{equation}
For the algebraic Gaussian target $m_t^\star=\sum_{j\ge1}cj^{-\alpha}Z_{t+1-j}$ with $c>0$ and $\alpha>1$, orthogonal projection onto the first Wiener chaos gives
\begin{equation}
 \cE(\widehat m)\ge\sum_{j\ge1}(cj^{-\alpha}-\beta_j)^2.
 \label{eq:chaos-lower}
\end{equation}
\end{theorem}

The envelope follows from a Gaussian Stein identity together with contraction of the state Jacobian across lags, and the projection bound is orthogonality in the first Wiener chaos.
App.~\ref{app:stability} gives the proof.

\subsection{Accuracy forces near-critical dynamics}
\begin{corollary}[Necessary quantitative near-criticality]
\label{cor:critical}
Under the assumptions of Thm.~\ref{thm:first-chaos}, in particular $\alpha>1$, with the derivative budgets $K$ and $M$ and the target amplitude $c$ held fixed:
For sufficiently small $\delta=-\log\rho$,
\begin{equation}
 \cE(\widehat m)
 \ge C\left[\frac{\delta}{\log(e/\delta)}\right]^{2\alpha-1}.
 \label{eq:critical-floor}
\end{equation}
Consequently, for all sufficiently small $\tau$, if $\cE(\widehat m)\le\tau$, then necessarily
\begin{equation}
 -\log\rho
 \le C'\tau^{1/(2\alpha-1)}\log\frac e\tau.
 \label{eq:critical-rate}
\end{equation}
\end{corollary}
The condition is necessary and not by itself sufficient: approaching criticality removes the fixed-contraction obstruction, and it does not on its own guarantee an accurate representation of algebraic memory.
The lower envelope inverts the exponential first-chaos tail against the algebraic target, and the sharp form is proved in App.~\ref{app:stability}.

\section{Finite-Sample Information Consequences}
\label{sec:information}

\subsection{Conditional trajectory information equals forecast risk}
Let $\theta\mapsto a_\theta$ index causal Gaussian models as in Sec.~\ref{sec:setup} with a common innovation variance $\sigma^2$, let $P_{\theta,T}$ be the law of $X_1,\dots,X_T$ under memory $a_\theta$, and write $\cE_\theta(b)$ for the excess risk~\eqref{eq:risk} of a candidate $b$ when the truth is $a_\theta$.
\begin{proposition}[Conditional Gaussian information and forecast identity]
\label{prop:kl-risk}
For two Gaussian causal predictors with common innovation variance $\sigma^2$, conditioning on the infinite past gives
\begin{equation}
\E_\theta\,\KL\!\left(P_{\theta,T}(\cdot\mid\mathcal F_0)\,\|\,P_{\vartheta,T}(\cdot\mid\mathcal F_0)\right)
 =\frac{T}{2\sigma^2}\cE_\theta(a_\vartheta).
 \label{eq:kl-risk}
\end{equation}
\end{proposition}
The same predictive geometry therefore controls both approximation and statistical distinguishability.
The identity is the conditional Gaussian relative-entropy computation given in App.~\ref{app:information}.

\subsection{Fractional finite-section corrections}
Write $P_d^{(T)}$ for the law of a length-$T$ block of the fractional process~\eqref{eq:farima}, $P_0^{(T)}$ for the white-noise law, and $\overline D(d\|0)$ for the Kullback--Leibler divergence rate per observation between the two stationary laws.
For fractional noise against white noise,
\begin{equation}
 0\le T\overline D(d\|0)-\KL(P_d^{(T)}\|P_0^{(T)})
 \le C_{d_0}d^2\log(eT),
 \label{eq:finite-kl}
\end{equation}
so the finite stationary block has the same local information constant as the entropy-rate experiment on the $d=h/\sqrt T$ scale.
For any fixed number of pure fractional poles with a \emph{fixed} strength vector satisfying a uniform cluster-integrability margin, meaning that the strengths summed over every merging cluster stay bounded away from $1/2$ so that each cluster limit remains inside the stationary fractional class, uniform Fisher--Hartwig asymptotics give an $O(\log T)$ finite-section correction that is uniform over reflection-symmetric pole locations even when those locations move with $T$ and merge.
The uniformity is over locations only; we do not claim the corresponding strength-uniform refinement for triangular arrays of vanishing pole strengths.
App.~\ref{app:information} gives the exact determinant formulas and states precisely what fails to be uniform when the number or strengths of the singularities vary with $T$.

\section{Experiments}
\label{sec:experiments}
The experiments are matched to the theorems.
App.~\ref{app:experiments} gives the architecture definitions, the width-independent budget enforcement, the training and seed protocol, and the complete numerical tables.
Fig.~\ref{fig:resource-laws} numerically evaluates the deterministic resource quantities from Sec.~\ref{sec:state} and~\ref{sec:fractional}, and Fig.~\ref{fig:e11} tests the width and stability consequences in controlled sequence models.

\begin{figure}[t]
\centering
\includegraphics[width=\linewidth]{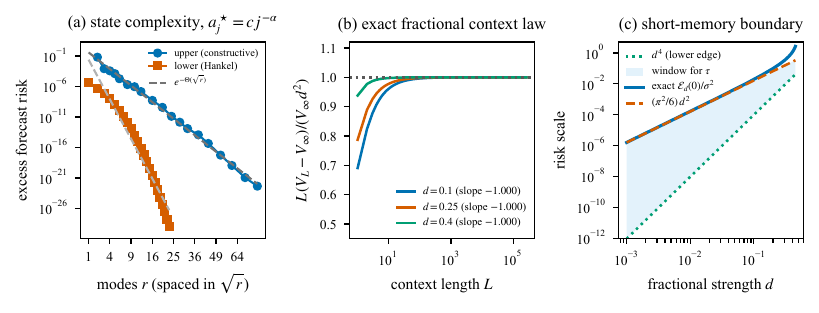}
\vspace{-2.0em}
\caption{Computed resource laws.
Every point is computed directly from deterministic theoretical quantities, through closed-form formulas, spectral quadrature, or controlled deterministic optimization.
Panel (a) brackets the best $r$-mode algebraic-memory error between a constructive exponential-quadrature upper bound and the Hankel rank-$r$ lower bound, on an axis spaced in $\sqrt{r}$ so that $e^{-\Theta(\sqrt{r})}$ appears as a straight line.
The lower curve is evaluated on finite Hankel sections and is therefore a numerical surrogate for the operator singular values, not the operator spectrum itself.
The two exponent constants differ, which is what Thm.~\ref{thm:state-complexity} asserts.
Panel (b) shows the exact fractional finite-past prediction gap normalized by $d^2$, whose collapse onto one confirms the sharp $1/L$ law of Corollary~\ref{cor:context-law}.
Panel (c) shows the exact zero-memory fractional risk against $d$, compared with $(\pi^2\sigma^2/6)d^2$, with the intermediate resource window of Thm.~\ref{thm:small-d} shaded.
The shaded region indicates the schematic $d^4$-to-$d^2$ scaling window only; the theorem's window is $C_0\sigma^2d^4\le\tau\le c_0\sigma^2d^2$ and its constants are suppressed in the plot.}
\label{fig:resource-laws}
\vspace{-1em}
\end{figure}

\begin{figure}[t]
\centering
\includegraphics[width=\linewidth]{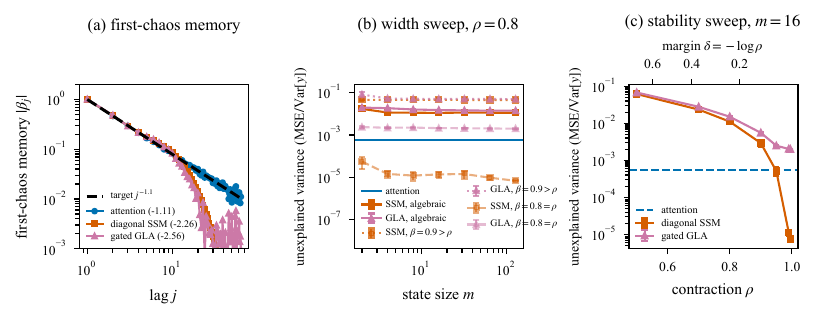}
\vspace{-2.0em}
\caption{Theorem-matched recurrent and attention experiment.
Panel (a) shows the first-chaos memory shape for the algebraic target, a contractive diagonal SSM, a bounded nonlinear gated recurrence, and attention.
The numbers in the legend are log-log slopes fitted over the plotted lag range; they summarize the shape on that range and are not estimates of a theoretical exponent, since Thm.~\ref{thm:first-chaos} predicts an exponential envelope rather than a power law for the contractive models.
Panel (b) shows a width sweep at fixed contraction and width-independent derivative budget, together with two exponential-target controls.
Because the three targets carry different variances, the vertical axis is the unexplained fraction of target variance, meaning held-out mean squared error divided by the variance of the target.
One control is matched to the contraction scale and is fit almost exactly.
The other is still exponential and decays more slowly than the contraction, and it leaves a floor above the algebraic one.
For the diagonal state-space model, that floor is essentially width-independent, while the remaining recurrent curves improve at small widths and then plateau well above zero.
An exponential target is therefore both the easiest and the hardest case here, which shows that functional form alone does not explain the floor in this benchmark and points instead to compatibility with the contraction envelope.
Panel (c) shows a stability sweep as $\rho\to1$, with the upper axis reparameterized by $\delta=-\log\rho$, in which the diagonal model keeps improving while the gated model retains a residual floor.}
\label{fig:e11}
\vspace{-1em}
\end{figure}

\paragraph{Resource-law calculations.}
Panels (b) and (c) of Fig.~\ref{fig:resource-laws} are closed-form Gamma-function identities, and panel (a) combines a numerical evaluation of the spectral-risk identity for each candidate kernel with a controlled deterministic search over the quadrature scalings at a hard budget of $r$ modes.
The fractional context panel uses the exact finite-past prediction variance, and the short-memory panel uses the Gamma-function expression~\eqref{eq:zero-boundary}.

\paragraph{Architecture-facing experiment.}
At fixed $\rho$ the models satisfy the stated budgets; for the gated model these hold on its invariant domain, while the $\rho$-sweep is qualitative.
At fixed $\rho$ we vary the state width, and at fixed width we vary $\rho$.
Two exponential targets are included as specificity controls, holding the architecture, the budgets, the optimizer, the data splits, and the matched initialization seeds fixed and changing only the target, the algebraic sweep additionally including one further seed.
The first is matched to the contraction scale, with decay equal to $\rho$.
The diagonal eigenvalues are $\rho$ times a sigmoid and therefore lie strictly below $\rho$, which places this target on the closure of the model class rather than in its interior, approachable to high accuracy without being realized exactly at finite width.
The second decays more slowly than the contraction, so its tail leaves the envelope of Thm.~\ref{thm:first-chaos} even though it is still exponential.
The three targets carry different variances, so we report the unexplained fraction of target variance, meaning held-out mean squared error divided by the variance of the target.
For the diagonal state-space model at the largest width, the boundary-matched control is fit to about $7\times10^{-6}$, the algebraic target to about $1.1\times10^{-2}$, and the slower exponential to about $4.5\times10^{-2}$, with the last of these insensitive to width.
The gated recurrence shows the same ordering with a much smaller gap on the boundary-matched control.
The controls therefore isolate tail compatibility rather than functional form, since an exponential target is fit essentially exactly when it is matched to the contraction and leaves a width-insensitive floor when it decays more slowly, with the algebraic target lying between the two.
This argues against generic optimization or finite-width capacity as the sole explanation for the observed algebraic-memory floor.
The targets here are noiseless linear functionals of the input, so the target functional is directly identified, in contrast with the stochastic forecasting setting of \citet{song2026learning}, where trained Transformers attain forecasting error close to the matched autoregressive benchmark while recovering the generating memory law only loosely.
In the stability sweep the diagonal state-space model improves steadily as $\rho$ approaches one, while the gated model stops improving beyond a moderate contraction and retains a residual floor.
At $L=64$, the largest-$\rho$ points are finite-context illustrations rather than asymptotic tests; the gated floor may also contain higher-chaos variance.

\section{Discussion and Limitations}
\label{sec:discussion}
The sharp state-complexity theorem concerns the specified finite-state and exponential architecture class, and it is not an unrestricted bound for arbitrary neural networks.
The nonlinear stability theorem requires uniform contraction and derivative budgets, and selective or noncontractive architectures can fall outside this class.
The near-criticality law is necessary only.
All three resource questions are posed for a scalar target.
App.~\ref{app:multivariate} gives a multivariate shared-factor extension, related to tensor and factor vector autoregression \citep{wang2022tensorvar,huang2025supervised}, together with the additional singularities that head collisions create, which are specific to that parameterization.
The small-$d$ architecture boundary is not fully resolved below the $d^4$ scale, and a growing number of moving Fisher--Hartwig singularities remains outside the finite-section information theory developed here.
These boundaries matter because modern selective state-space models can change the relevant representation and stability classes, and they need not stay inside a fixed one.
Two changes of setting would alter the conclusions rather than weaken them.
A contraction budget that varies across the state would replace the single margin of Thm.~\ref{thm:first-chaos} by a spectrum of margins, and a readout outside the derivative budgets would remove the uniform sensitivity control needed to derive the geometric envelope on the first-chaos coefficients.
Both are realistic descriptions of selective architectures \citep{gu2023mamba,benkish2025decimamba}, and both sit outside the class treated here.

\section*{Acknowledgment}
We thank Xuan Li for helpful discussion during the preparation of this work.
The code, which includes the synthetic generators, the model-selection procedures, the train, validation, and test split logic, the initialization controls, the checkpoint metadata, and scripts that reproduce every main-text and appendix figure, is available from the corresponding author on reasonable request.
This manuscript was prepared with AI-assisted drafting and editing.
The authors are responsible for verifying every mathematical claim, citation, experiment, and final wording.

\section*{Reproducibility statement.}
All assumptions are stated in the appendices, together with proofs of the main results and precise statements of every external analytic input used.
The supplementary code provides the exact resource-law calculations, the recurrent and attention architecture definitions, the width-independent budget enforcement, the training seeds, the checkpoint metadata, the first-chaos diagnostics, and the matched exponential-target control.

\section*{AI Use Statement}
Generative AI tools were used for drafting and editing.
The research questions, the results, and the proof strategies are the authors' own. 
AI assistance in the mathematical parts took the form of auditing and revising arguments rather than generating them.
The authors reviewed every AI-assisted claim, proof, citation, and experimental statement, and take responsibility for the final content.

\bibliography{predictive_memory_refs}
\bibliographystyle{iclr2027_conference}

\appendix

\section{Predictive-Risk and Hankel Preliminaries}
\label{app:prelim}

\subsection{Stationary forecast-risk identity}
Let $(X_t)_{t\in\mathbb Z}$ be centered, stationary, and purely nondeterministic, with $\mathcal F_t=\sigma(X_s:s\le t)$ and Bayes forecast $m_t^\star=\E[X_{t+1}\mid\mathcal F_t]=\sum_{j\ge1}a_j^\star X_{t+1-j}$.
The innovation $\eps_{t+1}=X_{t+1}-m_t^\star$ satisfies $\E[\eps_{t+1}\mid\mathcal F_t]=0$ and is therefore orthogonal to every square-integrable $\mathcal F_t$-measurable variable.
For a candidate kernel $a$ write $m_t(a)=\sum_{j\ge1}a_jX_{t+1-j}$, and let $\cE(a)$ be the excess of $\E[(X_{t+1}-m_t(a))^2]$ over its value at $a^\star$.

With the convention $\gamma(h)=(2\pi)^{-1}\int_{-\pi}^{\pi}e^{ih\omega}f_X(\omega)\dd\omega$ and $\Delta_j=a_j-a_j^\star$, the orthogonality of the innovation and Parseval's identity give \eqref{eq:risk}, namely
\begin{equation}
 \cE(a)=\frac1{2\pi}\int_{-\pi}^{\pi}|\Delta(e^{-i\omega})|^2f_X(\omega)\dd\omega .
 \label{eq:app-risk-identity}
\end{equation}
The argument is the standard one.
Partial sums $\sum_{j\le J}\Delta_jX_{t+1-j}$ are Cauchy in $L_2$ whenever $f_X$ is bounded above and $\Delta\in\ell_2$, the limit is $\mathcal F_t$-measurable and therefore orthogonal to $\eps_{t+1}$, and the spectral representation converts the second moment of the limit into the displayed integral.
For a perturbation $u\in\ell_2$ we write
\begin{equation}
 \|u\|_{\pred,f_X}^2:=\frac1{2\pi}\int_{-\pi}^{\pi}|U(e^{-i\omega})|^2f_X(\omega)\dd\omega ,
 \label{eq:app-pred-norm}
\end{equation}
which is the seminorm appearing in Thm.~\ref{thm:fractional-geometry}.

\subsection{Regular and singular spectra}
Under the regular condition $0<m\le f_X\le M<\infty$, inserting the bounds into \eqref{eq:app-risk-identity} and applying Parseval gives
\begin{equation}
 m\|a-a^\star\|_2^2\le\cE(a)\le M\|a-a^\star\|_2^2 ,
 \label{eq:app-regular}
\end{equation}
so forecast risk and coefficient error are equivalent up to constants, and the truncation of the truth at context $L$ has risk between $m$ and $M$ times $\sum_{j>L}(a_j^\star)^2$.

The two halves of \eqref{eq:app-regular} are used by different results and should be separated.
The lower half needs only $f_X\ge m>0$ and is what converts an approximation obstruction into a risk lower bound, so Thm.~\ref{thm:state-complexity} survives on spectra that are unbounded above.
The upper half needs $f_X\le M$ and is what converts a construction into a risk upper bound.
Sec.~\ref{sec:fractional} deliberately violates the upper half, since a spectral pole at the origin makes $M$ infinite, and App.~\ref{app:fractional} shows that the correct replacement is not a weaker constant but a different norm.

\subsection{Finite-state kernels and Hankel rank}
For a sequence $a$ and an integer $n$ define the Hankel block $H_n(a)=[a_{i+j-1}]_{i,j=1}^{n}$, which involves only the first $2n-1$ coefficients, and write $s_1\ge s_2\ge\cdots$ for singular values.
An $r$-dimensional linear state with transition $A\in\R^{r\times r}$ and maps $u,v\in\R^{r}$ produces $b_j=u^\top A^{j-1}v$, and the factorization $[H_n(b)]_{ij}=(u^\top A^{i-1})(A^{j-1}v)$ exhibits $H_n(b)$ as a product of an $n\times r$ and an $r\times n$ matrix, so
\begin{equation}
 \operatorname{rank}H_n(b)\le r\qquad\text{for every }n .
 \label{eq:app-hankel-rank}
\end{equation}
The same bound holds for any $r$-mode exponential kernel $b_j=\sum_{q\le r}w_q\rho_q^{\,j-1}$, since $H_n(b)=\sum_q w_qz_qz_q^\top$ with $z_q=(1,\rho_q,\dots,\rho_q^{n-1})^\top$.
The architecture class must record the finite context as well as the state dimension, because a hard truncation of a state-space sequence is in general no longer a state-space sequence and no longer has globally bounded Hankel rank.
We therefore let $\cA_r$ denote the union over finite contexts of the truncations of $r$-state kernels,
\begin{equation}
 \cA_r:=\bigcup_{L\ge1}\Bigl\{b:\ b_j=u^\top A^{j-1}v\,\mathbf{1}\{j\le L\},\ A\in\R^{r'\times r'},\ u,v\in\R^{r'},\ r'\le r\Bigr\},
 \label{eq:app-Ar}
\end{equation}
and we call $L$ the context of that member.
The $r$-mode exponential kernels form the diagonal-state subclass, obtained by taking $A$ diagonal with entries $\rho_q$ and absorbing the weights into $u$.

The only property of \eqref{eq:app-Ar} used below is the following restricted form of \eqref{eq:app-hankel-rank}.
For a member with context $L$ and any $n$ with $2n-1\le L$, the block $H_n(b)$ involves only coefficients with index at most $L$, so it coincides with the corresponding block of the underlying untruncated state-space sequence and therefore satisfies $\operatorname{rank}H_n(b)\le r$.
No rank claim is made for blocks that reach past the context, and none is needed.
A lower bound over $\cA_r$ applies simultaneously to $r$ exponential modes, $r$ diagonal recurrent states, and $r$ general linear states, in each case at every finite context.

\subsection{Ky Fan and low-rank approximation inequalities}
\begin{lemma}[Low-rank Frobenius approximation]
\label{lem:app-eckart}
If $\operatorname{rank}(B)\le r<n$ then $\|M-B\|_F^2\ge\sum_{q>r}s_q(M)^2$ for every $M\in\R^{n\times n}$.
\end{lemma}

\begin{proof}
Let $P$ project onto the column space of $B$, so $(I-P)B=0$ and $\operatorname{rank}P\le r$.
Projection does not increase the Frobenius norm, so $\|M-B\|_F^2\ge\|(I-P)M\|_F^2=\operatorname{tr}(MM^\top(I-P))$.
Diagonalizing $MM^\top$ and writing $p_q$ for the diagonal of $P$ in that basis gives $\sum_qs_q^2(1-p_q)$ with $0\le p_q\le1$ and $\sum_qp_q\le r$, and since the $s_q$ decrease the sum is at least $\sum_{q>r}s_q^2$.
\end{proof}

The second ingredient is the multiplicity count that converts a Hankel discrepancy back to coefficient error.
Writing $w_{n,k}$ for the number of pairs $(i,j)$ in $\{1,\dots,n\}^2$ with $i+j-1=k$,
\begin{equation}
 \|H_n(e)\|_F^2=\sum_{k=1}^{2n-1}w_{n,k}e_k^2,
 \qquad 1\le w_{n,k}\le n ,
 \label{eq:app-hankel-weights}
\end{equation}
so $\sum_{k\le2n-1}e_k^2\ge n^{-1}\|H_n(e)\|_F^2$.
A coefficient appears on up to $n$ antidiagonals of the block, so a comparison that ignores $w_{n,k}$ counts that mass more than once, and the factor $n^{-1}$ is the correction.
This is the step referred to in App.~\ref{app:state} as the corrected Ky Fan indexing.

\subsection{Gaussian first-chaos projection}
Let $Z=(Z_t)$ be independent standard Gaussians and let $\widehat m_t$ be a square-integrable function of $Z$.
The Wiener chaos decomposition writes $L_2(\sigma(Z))$ as an orthogonal sum $\bigoplus_{p\ge0}\mathcal H_p$, and the projection onto the first chaos is
\begin{equation}
 P_1\widehat m_t=\sum_{j\ge1}\beta_jZ_{t+1-j},
 \qquad \beta_j=\E[\widehat m_tZ_{t+1-j}] ,
 \label{eq:app-first-chaos}
\end{equation}
the coefficients being well defined because the $Z_{t+1-j}$ are orthonormal.

If the target is itself in the first chaos, $m_t^\star=\sum_{j\ge1}a_j^\star Z_{t+1-j}$, then orthogonality of the chaoses gives
\begin{equation}
 \E[(m_t^\star-\widehat m_t)^2]
 \ \ge\ \|m_t^\star-P_1\widehat m_t\|_2^2
 \ =\ \sum_{j\ge1}(a_j^\star-\beta_j)^2 ,
 \label{eq:app-chaos-lower}
\end{equation}
since the higher chaoses of $\widehat m_t$ are orthogonal to $m_t^\star$ and only add to the error.
Inequality \eqref{eq:app-chaos-lower} is what makes a statement about the linear part of a nonlinear predictor into a lower bound on its full risk, and it is the reason Thm.~\ref{thm:first-chaos} can control an arbitrary nonlinear recurrence through a single sequence $\beta$.
Nothing in \eqref{eq:app-chaos-lower} asserts that the higher chaoses are small, only that they cannot help.

\section{Sharp State Complexity of Algebraic Memory}
\label{app:state}

Throughout this appendix the target is $a_j^\star=cj^{-\alpha}$ with $\alpha>1$ inside a stable regular-spectrum forecasting model, so that \eqref{eq:app-regular} holds with constants $m_\alpha,M_\alpha$ determined by $(\alpha,c,\sigma)$.

\subsection{Comparison with existing approximation metrics}
The literature contains several distinct notions of memory approximation, and the constants and even the exponents depend on which one is used.
Since Thm.~\ref{thm:state-complexity} is a statement about one of them, we separate the three before proving anything.

\subsubsection{Functional-memory approximation in curse-of-memory theory}
Related curse-of-memory results formalize different notions of memory obstruction rather than one common metric.
\citet{li2021curse} and the inverse-approximation results of \citet{wang2024inverse} use worst-case functional approximation over a class of inputs, whereas \citet{wang2023ssmnonlinear} characterize exponential decay through a derivative-based memory function for stacked state-space models.
None of these notions is the stationary forecast risk used here, and the separation is the one already drawn in Sec.~\ref{sec:stability}.
Those results give explicit rates in their own metrics, and what they do not provide is a two-sided resource curve for one fixed predictive target in forecast risk, which is the question asked here.
The constructions used for their upper directions are likewise not optimized for a single fixed target.
We do not claim to sharpen those statements on their own terms, since neither the metric nor the target class coincides with ours.

\subsubsection{Finite-horizon uniform approximation in power-law recurrent constructions}
A second notion fixes a horizon $[1,T]$ and asks for uniform approximation of a power-law kernel on that window by a sum of exponentials, which is the setting of recent architectural constructions \citep{mlaiki2026vort}.
The horizon is externally imposed and the error is measured on it, so the resulting complexity depends on $T$ and does not by itself say anything about the infinite predictive law.

\subsubsection{Infinite predictive-law approximation in forecast risk}
Our metric is \eqref{eq:risk}, the horizon is not externally fixed, and the infimum in $\inf_{b\in\cA_r}\cE(b)$ is taken over the context as well as over the state.
This is the notion under which both sides match, and it is the one relevant to forecasting because it is the risk actually incurred.
A construction that is good on a finite window is not automatically good here, since the unseen tail contributes to \eqref{eq:risk} whether or not it lies inside the window.

\subsection{Laplace representation of algebraic kernels}
\begin{lemma}[Power laws as continuous exponential mixtures]
\label{lem:laplace}
For every $\alpha>0$ and $x>0$,
\begin{equation}
 x^{-\alpha}=\frac1{\Gamma(\alpha)}\int_0^{\infty}t^{\alpha-1}e^{-xt}\dd t .
 \label{eq:laplace}
\end{equation}
In particular $j^{-\alpha}$ is a positive continuous mixture of the exponential atoms $e^{-tj}$.
\end{lemma}

\begin{proof}
Substitute $s=xt$ in the definition $\Gamma(\alpha)=\int_0^\infty s^{\alpha-1}e^{-s}\dd s$.
\end{proof}

Lem.~\ref{lem:laplace} is the whole reason a state can represent algebraic memory at all.
The mixture is positive, so the discretization below produces positive weights and stable modes rather than a signed cancellation, and this is what keeps the construction inside $\cA_r$ with a controlled norm.

\subsection{Positive exponential quadrature}
Substituting $t=e^{u}$ in \eqref{eq:laplace} gives
\begin{equation}
 j^{-\alpha}=\frac1{\Gamma(\alpha)}\int_{-\infty}^{\infty}e^{\alpha u}e^{-je^{u}}\dd u ,
 \label{eq:log-laplace}
\end{equation}
whose integrand decays like $e^{\alpha u}$ as $u\to-\infty$ and doubly exponentially as $u\to+\infty$, and is analytic in the strip $|\operatorname{Im}u|<\pi/2$ where $\operatorname{Re}e^{u}>0$.
Applying the trapezoidal rule with step $h$ on a truncated range therefore incurs a discretization error of order $e^{-\pi^2/h}$ and a left-truncation error of order $e^{\alpha u_{\min}}$.
Each node contributes one exponential mode with positive weight, since $\lambda=e^{u}>0$ and the weight $h\,e^{\alpha u}/\Gamma(\alpha)$ is positive.

The upper bound uses one external approximation result, which we state exactly as it is used.

\begin{fact}[Positive exponential approximation of a power law]
\label{fact:bm}
Fix $\alpha>0$.
For every $0<\varrho<1$ and every relative tolerance $0<\eta\le e^{-1}$ there are an integer $R$, positive weights $w_q$, and positive exponents $p_q$ with
\begin{equation}
 \Bigl|x^{-\alpha}-\sum_{q=1}^{R}w_qe^{-p_qx}\Bigr|\le\eta\,x^{-\alpha}
 \qquad(\varrho\le x\le1),
 \label{eq:bm-relative}
\end{equation}
and, for fixed $\alpha$, the construction may be chosen with
\begin{equation}
 R\le C_\alpha\bigl(1+\log\eta^{-1}\bigr)\bigl(1+\log\varrho^{-1}+\log\eta^{-1}+\log\log\eta^{-1}\bigr),
 \label{eq:bm-count}
\end{equation}
where $C_\alpha$ depends only on $\alpha$.
This is the exponential-sum approximation of \citet{beylkin2005exp,beylkin2010revisited}, in the relative-error form on a bounded interval bounded away from the origin.
\end{fact}

\begin{theorem}[Finite-context exponential compression]
\label{thm:compression}
Fix $\alpha>0$, an integer $L\ge2$, and $0<\eta\le e^{-1}$.
There are positive $b_q$ and $\lambda_q$, $1\le q\le R$, with
\begin{equation}
 \Bigl|j^{-\alpha}-\sum_{q=1}^{R}b_qe^{-\lambda_q(j-1)}\Bigr|\le\eta\,j^{-\alpha}
 \qquad(1\le j\le L),
 \label{eq:relative-compression}
\end{equation}
and
\begin{equation}
 R\le C_\alpha\bigl(1+\log\eta^{-1}\bigr)\bigl(1+\log L+\log\eta^{-1}+\log\log\eta^{-1}\bigr).
 \label{eq:head-count}
\end{equation}
\end{theorem}

\begin{proof}
Apply Fact~\ref{fact:bm} with $\varrho=1/L$, so that \eqref{eq:bm-relative} holds on $[1/L,1]$ and \eqref{eq:bm-count} reads $R\le C_\alpha(1+\log\eta^{-1})(1+\log L+\log\eta^{-1}+\log\log\eta^{-1})$.
Set $x=j/L$, which lies in $[1/L,1]$ for $1\le j\le L$.
Then $j^{-\alpha}=L^{-\alpha}x^{-\alpha}$, and substituting the approximant gives
\[
 \Bigl|j^{-\alpha}-L^{-\alpha}\sum_q w_qe^{-p_qj/L}\Bigr|\le\eta\,j^{-\alpha}
 \qquad(1\le j\le L).
\]
Putting $\lambda_q=p_q/L>0$ and $b_q=L^{-\alpha}w_qe^{-\lambda_q}>0$ turns the right-hand sum into $\sum_qb_qe^{-\lambda_q(j-1)}$, which is \eqref{eq:relative-compression} with positive weights and positive rates.
The count is unchanged, giving \eqref{eq:head-count}.
\end{proof}

The error is relative rather than absolute, which matters because the risk weights every lag and an absolute guarantee would be vacuous at large $j$.
Balancing the two error sources in \eqref{eq:log-laplace} at $R$ nodes gives $\log\eta^{-1}\asymp\pi\sqrt{\alpha R}$, which is the origin of the square root and is consistent with \eqref{eq:head-count}.

\subsection{Finite-context upper bound}
\begin{corollary}[Risk of $R$ exponential modes at context $L$]
\label{cor:risk-upper}
With $R$ as in \eqref{eq:head-count} there is $b\in\cA_R$ supported on $\{1,\dots,L\}$ with
\begin{equation}
 \cE(b)\le M_\alpha c^2\Bigl\{\zeta(2\alpha)\eta^2+\frac{L^{1-2\alpha}}{2\alpha-1}\Bigr\}.
 \label{eq:risk-upper}
\end{equation}
\end{corollary}

\begin{proof}
By \eqref{eq:app-regular} it suffices to bound $\|b-a^\star\|_2^2$.
On $\{1,\dots,L\}$ the relative bound \eqref{eq:relative-compression} gives a squared error at most $\eta^2c^2\sum_{j\le L}j^{-2\alpha}\le\eta^2c^2\zeta(2\alpha)$, and beyond $L$ the candidate vanishes so the error is the tail $c^2\sum_{j>L}j^{-2\alpha}\le c^2L^{1-2\alpha}/(2\alpha-1)$.
\end{proof}

\begin{corollary}[Constructive root-exponential curve]
\label{cor:constructive}
There are $C_L,C_{\rm up},c_{\rm up}>0$ depending only on the fixed model such that for every sufficiently large $r$ there is $L_r\le\exp(C_L\sqrt r)$ and $b\in\cA_r$ supported on $\{1,\dots,L_r\}$ with $\cE(b)\le C_{\rm up}e^{-c_{\rm up}\sqrt r}$.
Equivalently, reaching risk $\tau$ needs at most $C_{\rm mode}\log^2(C_{\rm up}/\tau)$ modes for another fixed constant $C_{\rm mode}>0$.
\end{corollary}

\begin{proof}
Choose $\eta$ and $L$ so that the two terms of Corollary~\ref{cor:risk-upper} are comparable, that is $L^{1-2\alpha}\asymp\eta^2$, so $\log L\asymp\log\eta^{-1}/(2\alpha-1)$.
Then \eqref{eq:head-count} gives $R\le C(1+\log\eta^{-1})^2$, and inverting yields $\log\eta^{-1}\asymp\sqrt R$ and hence $\cE\le C_{\rm up}e^{-c_{\rm up}\sqrt R}$.
\end{proof}

\subsection{Hankel operator formulation}
Let $\mathbf H_\alpha$ be the Hankel operator on $\ell_2(\mathbb N_0)$ with symbol $(i+j+1)^{-\alpha}$, and let $G^{(\alpha)}_N$ be its leading $N\times N$ section.
For $\alpha>1$ the entries are square summable, so $\mathbf H_\alpha$ is Hilbert--Schmidt, hence compact and self-adjoint, and being a moment matrix of a positive measure it is positive semidefinite.
By the Eckart--Young principle the best unconstrained rank-$r$ approximation in Frobenius norm is governed by the tail of the singular spectrum \citep{eckart1936rank}, and classical Hankel-operator theory supplies the structured spectral and approximation background used here \citep{widom1966hankel,aak1971schmidt,peller2003hankel}.
By \eqref{eq:app-hankel-rank} every member of $\cA_r$ has Hankel rank at most $r$.

\subsection{Finite-section singular-value lower scale}
\begin{lemma}[Finite sections retain the relevant singular value]
\label{lem:finite-section}
Let $P_N$ project onto the first $N$ coordinates and put $\mathbf H_{\alpha,N}=P_N\mathbf H_\alpha P_N$.
Then $\|\mathbf H_\alpha-\mathbf H_{\alpha,N}\|_{\rm op}\le C_\alpha N^{1-\alpha}$.
Consequently there are $\kappa_\alpha,c_\alpha,C_\alpha>0$ and $r_0$ such that, with $N_r=\lceil e^{\kappa_\alpha\sqrt r}\rceil$, one has $N_r>r$ and
\begin{equation}
 s_{r+1}\bigl(G^{(\alpha)}_{N_r}\bigr)\ \ge\ c_\alpha e^{-C_\alpha\sqrt r}
 \qquad(r\ge r_0).
 \label{eq:finite-section-sv}
\end{equation}
\end{lemma}

\begin{proof}
For the operator bound, $\mathbf H_\alpha-\mathbf H_{\alpha,N}$ is supported on the indices with $\max(i,j)\ge N$, whose entries are at most $(i+j+1)^{-\alpha}$.
Bounding the operator norm by the Hilbert--Schmidt norm and comparing the resulting sum with $\int_N^\infty x^{1-2\alpha}\dd x$, which converges because $\alpha>1$, gives $\|\mathbf H_\alpha-\mathbf H_{\alpha,N}\|_{\rm op}\le C_\alpha N^{1-\alpha}$.

For the singular value we use one external input, namely the classical root-exponential scale for the singular values of the Hankel operator with power-law symbol, in the form
\begin{equation}
 s_{r+1}(\mathbf H_\alpha)\ \ge\ c_\alpha e^{-C_\alpha\sqrt r}
 \qquad\text{for all large }r,
 \label{eq:widom-scale}
\end{equation}
which follows from the spectral theory of Hankel matrices developed by \citet{widom1966hankel}, specifically the singular-value asymptotics for Hankel operators with power-law symbols established there; \citet{peller2003hankel} restates that theory systematically and is the more convenient modern entry point.
Only the stretched-exponential scale in \eqref{eq:widom-scale} is used, and no claim is made about the exact constant $C$ or about the leading asymptotic coefficient.
Weyl's inequality for the perturbation $\mathbf H_{\alpha,N}=\mathbf H_\alpha-(\mathbf H_\alpha-\mathbf H_{\alpha,N})$ gives
\[
 s_{r+1}\bigl(G^{(\alpha)}_N\bigr)\ \ge\ s_{r+1}(\mathbf H_\alpha)-\|\mathbf H_\alpha-\mathbf H_{\alpha,N}\|_{\rm op}
 \ \ge\ c_\alpha e^{-C_\alpha\sqrt r}-C_\alpha' N^{1-\alpha}.
\]
Choose $\kappa_\alpha:=2C_\alpha/(\alpha-1)$ and $N_r=\lceil e^{\kappa_\alpha\sqrt r}\rceil$, so that
$C_\alpha' N_r^{1-\alpha}\le C_\alpha' e^{-(\alpha-1)\kappa_\alpha\sqrt r}=C_\alpha' e^{-2C_\alpha\sqrt r}$,
which is below $\tfrac12c_\alpha e^{-C_\alpha\sqrt r}$ for all large $r$.
The displayed difference is then at least $\tfrac12c_\alpha e^{-C_\alpha\sqrt r}$; decreasing the constant $c_\alpha$ in the statement by a factor of two gives \eqref{eq:finite-section-sv}, and $N_r>r$ holds for large $r$ since $e^{\kappa_\alpha\sqrt r}$ grows faster than $r$.
\end{proof}

The two-step structure is what the argument needs.
The singular values of the infinite operator decay root-exponentially, which is the classical Widom scale, but a lower bound on $\cE$ must be expressed through a finite section, because only finitely many coefficients enter a Hankel block.
Lem.~\ref{lem:finite-section} says the section can be taken exponentially large in $\sqrt r$ while still retaining the $(r+1)$st singular value up to a constant, and this is exactly the size at which the truncation error $N^{1-\alpha}$ is below the singular value being retained.

\subsection{Corrected Ky Fan indexing in the singular-value split}
Combining Lem.~\ref{lem:app-eckart} with a Hankel block of order $N$ bounds $\|H_N(b-a^\star)\|_F^2$ below by $\sum_{q>r}s_q(G^{(\alpha)}_N)^2$, but the passage back to coefficients must use \eqref{eq:app-hankel-weights} and therefore loses a factor $N$.
The block of order $N$ involves the first $2N-1$ coefficients, so a context of at least $2N-1$ is required for the block to be a genuine restriction of the candidate, and the admissibility condition is $2N-1\le L$.
Writing the two steps in the wrong order, that is applying Ky Fan to a block larger than the context allows, would compare against coefficients the candidate has not been asked to produce.

The loss of the factor $N=N_r=\lceil e^{\kappa_\alpha\sqrt r}\rceil$ is what makes the constant in the exponent of the lower bound differ from the constant in the upper bound.
Both are of the form $e^{-\Theta(\sqrt r)}$ because $\log N_r\asymp\sqrt r$, so dividing by $N_r$ changes the constant in front of $\sqrt r$ and not the exponent.
Thm.~\ref{thm:state-complexity} asserts matching rates and not matching constants, and this is where the discrepancy enters.

\subsection{Proof of Thm.~\ref{thm:state-complexity}}
\begin{proof}[Proof of Thm.~\ref{thm:state-complexity}]
The upper bound is Corollary~\ref{cor:constructive}.

For the lower bound, fix $r\ge r_0$, let $N=N_r$ from Lem.~\ref{lem:finite-section}, and let $b\in\cA_r$ be arbitrary, with context $L$ as in \eqref{eq:app-Ar}.
If $L<2N-1$ then the coefficients of $b$ beyond $L$ vanish while those of $a^\star$ are of order $j^{-\alpha}$, so the tail term alone gives $\cE(b)\ge m_\alpha c^2\sum_{j>L}j^{-2\alpha}\ge c_{\rm tail}e^{-C_{\rm tail}\sqrt r}$ with $C_{\rm tail}=(2\alpha-1)\kappa_\alpha$ after absorbing the ceiling and the factor $2N-1$ into $c_{\rm tail}$.
Otherwise $2N-1\le L$, so the block $H_N(b)$ lies inside the context and, by the restricted rank statement following \eqref{eq:app-Ar}, has rank at most $r$.
Lem.~\ref{lem:app-eckart} applied with $M=cG^{(\alpha)}_N$ gives
\[
 \|H_N(b-a^\star)\|_F^2\ \ge\ c^2\sum_{q>r}s_q\bigl(G^{(\alpha)}_N\bigr)^2\ \ge\ c^2s_{r+1}\bigl(G^{(\alpha)}_N\bigr)^2 ,
\]
and \eqref{eq:app-hankel-weights} converts this to $\|b-a^\star\|_2^2\ge c^2N^{-1}s_{r+1}(G^{(\alpha)}_N)^2$.
Inserting \eqref{eq:finite-section-sv} and $N=\lceil e^{\kappa_\alpha\sqrt r}\rceil$, and using the lower half of \eqref{eq:app-regular}, gives
\[
 \cE(b)\ \ge\ m_\alpha c^2c_\alpha^2e^{-\kappa_\alpha\sqrt r}e^{-2C_\alpha\sqrt r}.
\]
Taking $C_1=\max\{C_{\rm tail},\kappa_\alpha+2C_\alpha\}$ and decreasing $c_1>0$ to cover both context branches gives $\cE(b)\ge c_1e^{-C_1\sqrt r}$.
Taking the infimum over $b\in\cA_r$ and over $L$ gives \eqref{eq:rootexp}.

For \eqref{eq:log2}, apply the elementary inversion that if $c_1e^{-C_1r^\gamma}\le\inf_{b\in\cA_r}\cE(b)\le C_2e^{-c_2r^\gamma}$ for large $r$, then $r^\star(\tau)=\Theta(\log^{1/\gamma}(1/\tau))$.
Here $\gamma=1/2$, which gives $\Theta(\log^2(1/\tau))$.
\end{proof}

\subsection{Analytic algebraic Laplace-mixture extension}
The same exponent holds for a broader class.
Fix $\alpha>1$ and let the analytic algebraic Laplace-mixture class consist of the positive sequences
\begin{equation}
 a_j=\int_0^{\infty}e^{-(j-1)t}t^{\alpha-1}h(t)\dd t,
 \qquad j\ge1,
 \label{eq:cm-representation}
\end{equation}
where $h$ is positive on $[0,\infty)$ with $h(0)>0$, extends holomorphically to an open set containing both a neighbourhood $\{|z|<r_0\}$ of the origin for some $r_0>0$ and a sector $\Sigma_\vartheta=\{z\ne0:|\arg z|<\vartheta\}$ with $0<\vartheta<\pi/2$, and satisfies
\begin{equation}
 |h(z)|\le C_he^{-\mu\Re z}
 \qquad(z\in\Sigma_\vartheta)
 \label{eq:cm-decay}
\end{equation}
for constants $C_h,\mu>0$.
The three requirements are used separately below, so none is decorative.
Positivity of $h$ makes the mixture a moment representation and keeps the constructed modes positive, the value $h(0)>0$ is what the local Loewner step consumes, and the exponential decay \eqref{eq:cm-decay} together with sectorial analyticity is what the dyadic Bernstein-ellipse quadrature consumes.
Every such $a$ satisfies $a_j\sim h(0)\Gamma(\alpha)j^{-\alpha}$ and is squeezed between two positive multiples of $j^{-\alpha}$, so its tail obeys $\sum_{j>L}a_j^2=\Theta(L^{1-2\alpha})$.
Ordinary algebraic memory and the fractional kernel of App.~\ref{app:fractional} both belong to this class.

Two statements carry over.
On the upper side, the geometry that the composite quadrature needs is uniform across dyadic Laplace intervals, and it is worth stating exactly, since it is the only place where sectorial analyticity is consumed.

\begin{lemma}[Dyadic Bernstein-ellipse geometry]
\label{lem:dyadic-ellipse}
Let $[a,2a]\subset(0,\infty)$ and let $t=a(3+x)/2$ map $[-1,1]$ onto it.
For every Bernstein ellipse $E_\rho$ with $1<\rho<3+2\sqrt2$, the image satisfies
\[
 \min_{E_\rho}\Re t=\frac a4\bigl(6-\rho-\rho^{-1}\bigr)>0 ,
\]
so the image lies in the open right half-plane, and for $\rho>1$ small enough it lies inside any prescribed sector $\Sigma_\vartheta$.
Consequently $|e^{-(j-1)t}|\le e^{-c_\rho(j-1)a}$ uniformly on $E_\rho$.
On the endpoint interval $[0,L^{-1}]$ the affine image of a fixed ellipse obeys $|t|\le C_\rho/L$, hence for all sufficiently large $L$ it lies inside the fixed holomorphic neighbourhood of the origin; the finitely many remaining $L$ are absorbed into the constants. Moreover $|(j-1)t|\le C_\rho$ and $|e^{-(j-1)t}|\le e^{C_\rho}$ for every $j\le L$.
\end{lemma}

\begin{proof}
A point of $E_\rho$ is $x=\tfrac12(\zeta+\zeta^{-1})$ with $|\zeta|\in[\rho^{-1},\rho]$, so $\Re x\ge-\tfrac12(\rho+\rho^{-1})$ and $\Re t\ge\tfrac a2\{3-\tfrac12(\rho+\rho^{-1})\}$, which is the stated quantity and is positive exactly when $\rho+\rho^{-1}<6$, that is $\rho<3+2\sqrt2$.
The sector statement follows because $E_\rho$ shrinks to $[-1,1]$ as $\rho\to1$ and the image of $[-1,1]$ is the compact interval $[a,2a]$ on the positive axis.
The exponential bound is $|e^{-(j-1)t}|=e^{-(j-1)\Re t}$ with the displayed lower bound on $\Re t$, and the endpoint statement is the triangle inequality on the image of a fixed ellipse under a map of size $O(1/L)$.
\end{proof}

\noindent
The compression then follows in three steps, each contributing one factor to the head count.
First, the decay \eqref{eq:cm-decay} lets the Laplace integral be truncated at $T_\eta=C(1+\log\eta^{-1})$ with relative error at most $\eta/3$ uniformly in $j$, with no dependence on $L$ at all.
The point is that the discarded mass carries the same power of $j$ as $a_j$ itself, so the comparison is scale free.
Writing $q_j:=j-1+\mu$, the bound $|h(t)|\le C_he^{-\mu t}$ gives
\[
 R_j(T):=\int_T^\infty e^{-(j-1)t}t^{\alpha-1}|h(t)|\dd t
 \ \le\ C_h\int_T^\infty e^{-q_jt}t^{\alpha-1}\dd t
 \ =\ C_hq_j^{-\alpha}\Gamma(\alpha,q_jT),
\]
after the substitution $u=q_jt$.
The upper incomplete Gamma function obeys $\Gamma(\alpha,x)\le C_\alpha e^{-x/2}$ for $x\ge1$, the polynomial prefactor being absorbed into the weaker exponential, and $q_j\ge\mu>0$ with $q_j\asymp_\mu j$, so $R_j(T)\le C_{\alpha,\mu}e^{-\mu T/2}j^{-\alpha}$.
Since $h(0)>0$ gives $a_j\ge cj^{-\alpha}$ on the same class, the ratio obeys $R_j(T)/a_j\le Ce^{-cT}$ uniformly in $j$, and $T_\eta=C(1+\log\eta^{-1})$ suffices.
Second, partition $[L^{-1},T_\eta]$ dyadically and treat $[0,L^{-1}]$ separately; the number of intervals is $M=O(1+\log L+\log\log\eta^{-1})$.
Third, apply the analytic Gauss--Jacobi estimate of \citet{baffet2019gaussjacobi} on each interval with $J$ nodes.
Lem.~\ref{lem:dyadic-ellipse} supplies the uniform ellipse geometry that estimate requires. On a dyadic interval $[a,2a]$, the ellipse bound gives an absolute quadrature error at most $C\rho^{-2J}e^{-c(j-1)a}a^\alpha\le C\rho^{-2J}j^{-\alpha}$, and the endpoint interval obeys the same $C\rho^{-2J}j^{-\alpha}$ bound after the preceding neighbourhood argument. Summing the $M$ interval errors therefore gives at most $CM\rho^{-2J}j^{-\alpha}\le C_{\rm quad}M\rho^{-2J}a_j$.
Choosing $J=O(1+\log\eta^{-1}+\log M)$ makes that at most $\eta/3$, and since $M=O(1+\log L+\log\log\eta^{-1})$ the total node count $R=MJ$ remains $O\{(1+\log L+\log\eta^{-1})^2\}$, which is the analogue of Thm.~\ref{thm:compression}.

The lower side needs care, and a global Loewner comparison is not available under these hypotheses.
A positive analytic bounded $h$ may decay to zero as $t\to\infty$, so the two mixing measures need not be comparable on the whole half-line, and we do not assume $h$ bounded away from zero.
The correct argument is local.
The hypothesis $h(0)>0$ from \eqref{eq:cm-representation}, together with continuity, gives $t_0>0$ and $c_0>0$ with $h(t)\ge c_0$ on $[0,t_0]$.
Writing $v(t)=(1,e^{-t},e^{-2t},\dots)^\top$ and splitting the Laplace representation of $\mathbf H_\alpha$ at $t_0$ as $\mathbf H_\alpha=\mathbf H_\alpha^{(0)}+\mathbf T_\alpha$, the bound $e^{-t}\le1$ on $[0,t_0]$ gives the Loewner inequality $\mathbf H(a)\succeq c_0\Gamma(\alpha)\mathbf H_\alpha^{(0)}$ on the low-$t$ piece alone.
The tail operator is exponentially compressible, since $\|(I-P_N)v(t)\|_2^2\le C_{t_0}e^{-2Nt_0}$ for $t\ge t_0$ gives $\|\mathbf T_\alpha-P_N\mathbf T_\alpha P_N\|_{\rm op}\le Ce^{-cN}$ and hence $s_{N+1}(\mathbf T_\alpha)\le Ce^{-cN}$.
Applying the Ky Fan inequality $s_{2N+1}(A+B)\le s_{N+1}(A)+s_{N+1}(B)$ to $\mathbf H_\alpha=\mathbf H_\alpha^{(0)}+\mathbf T_\alpha$, together with Widom's scale for $\mathbf H_\alpha$ and the exponentially small tail term, transfers the root-exponential lower scale to $\mathbf H_\alpha^{(0)}$ and therefore, through the Loewner inequality, to $\mathbf H(a)$, giving $s_n(\mathbf H(a))\ge c_ae^{-C_a\sqrt n}$.
The doubled index in the Ky Fan step is what changes the constant in the exponent and not the exponent itself.

Combining the two sides gives $\inf_{b\in\cA_r}\cE(b)=e^{-\Theta(\sqrt r)}$ and $r^\star(\tau)=\Theta(\log^2(1/\tau))$ over the whole class, so the law is not an artifact of the exact power.

\subsection{Current-token affine-gating benchmark}
The theorem restricts the architecture class, so it is worth recording one enlargement that does not change the exponent.
Consider a recurrence whose transition is affine in the previous state with coefficients depending on the current input, $S_t=A(Z_t)S_{t-1}+b(Z_t)$ and $\widehat m_t=c(Z_t)^\top S_{t-1}+d(Z_t)$, on the noiseless iid Gaussian benchmark with target in the first chaos.
Assume a stationary causal solution with
\[
 \E\|S_t\|^2<\infty,
 \qquad
 \E\widehat m_t^2<\infty ,
\]
taken as hypotheses rather than derived: a random affine recursion can have a stationary law under $\E\log\|A(Z)\|_{\rm op}<0$ while that law has an infinite second moment, so a moment condition on the maps alone does not deliver them.
Assume in addition
\[
 \E\bigl[(1+Z^2)\bigl\{\|A(Z)\|_{\rm op}^2+\|b(Z)\|^2+\|c(Z)\|^2+d(Z)^2\bigr\}\bigr]<\infty .
\]
This last condition is stronger than necessary, but it ensures that the Bochner expectations $\bar A$, $\bar c$ and $v$ below and the first-chaos covariances are well defined.
The readout is taken before the current update.
A gate that reads after the update, $\widehat m_t=\widetilde c(Z_t)^\top S_t+\widetilde d(Z_t)$, is the same predictor with $c(z)=A(z)^\top\widetilde c(z)$ and $d(z)=\widetilde c(z)^\top b(z)+\widetilde d(z)$, and is therefore covered whenever those induced maps satisfy the moment condition above; that is a mild restriction rather than an automatic consequence, since the induced $c$ involves a product of two input-dependent factors.
Let $\mu_S=\E S_t$, $\widetilde S_t=S_t-\mu_S$, $\bar A=\E A(Z_t)$, $\bar c=\E c(Z_t)$, and define the centered injection $\xi(z)=A(z)\mu_S+b(z)-\mu_S$, for which $\E\xi(Z_t)=0$. Then $\widetilde S_t=\bar A\widetilde S_{t-1}+\{A(Z_t)-\bar A\}\widetilde S_{t-1}+\xi(Z_t)$. Writing $v=\E[Z_t\xi(Z_t)]$, independence gives $\E[\widetilde S_{t-1}Z_{t-k}]=\bar A^{k-1}v$, and hence the first-chaos coefficients satisfy
\begin{equation}
 \beta_j=\bar c^\top\bar A^{\,j-2}v\qquad(j\ge2).
 \label{eq:gate-beta}
\end{equation}
The shifted sequence therefore has Hankel rank at most the state dimension, by \eqref{eq:app-hankel-rank} applied to \eqref{eq:gate-beta}.

Consequently the entire lower-bound argument applies to the averaged transition, and the best $m$-state gated risk on this benchmark is again bracketed by $c_1e^{-C_1\sqrt m}$ and $C_2e^{-c_2\sqrt m}$, so $m^\star(\tau)=\Theta(\log^2(1/\tau))$.
A gated linear attention head with key width $d_k$ contributes at most $d_k$ such modes through $\rho_\ell=\E\alpha_\ell(Z_t)$, so $H$ heads give at most $\sum_hd_{k,h}$ linear predictive modes and the same bound with $r$ replaced by that budget.
Data-dependent gating of this affine kind therefore buys no better exponent, and Sec.~\ref{sec:stability} is needed precisely because it treats recurrences that are not affine in the state.

\section{Fractional and Singular Predictive Geometry}
\label{app:fractional}

\subsection{Predictive Hilbert-space completion}
App.~\ref{app:prelim} assumed $f_X$ bounded above, which fails at a spectral pole.
The identity \eqref{eq:app-risk-identity} survives, but its right-hand side is no longer comparable to a coefficient norm and the space of admissible perturbations must be constructed rather than assumed.

Let $c_{00}$ be the finitely supported sequences and let $\mathcal H_{f_X}$ be the completion of $c_{00}$ under the seminorm \eqref{eq:app-pred-norm}, with elements written as equivalence classes $\mathbf u=[u^{(k)}]$ of Cauchy sequences.

\begin{proposition}[Forecast risk on the singular predictive space]
\label{prop:singular-space}
Do not assume $f_X$ bounded above.
For $\mathbf u\in\mathcal H_{f_X}$ the filters $U_t^{(k)}=\sum_{j\ge1}u_j^{(k)}X_{t+1-j}$ are Cauchy in $L_2$ with a limit $U_t(\mathbf u)$ depending only on the class, and the perturbed predictor $m_t^{\mathbf u}=m_t^\star+U_t(\mathbf u)$ satisfies
\begin{equation}
 \cE(\mathbf u)=\E\{X_{t+1}-m_t^{\mathbf u}\}^2-\E\{X_{t+1}-m_t^\star\}^2=\|\mathbf u\|_{\pred,f_X}^2 .
 \label{eq:singular-risk}
\end{equation}
\end{proposition}

\begin{proof}
For finitely supported $u$ the spectral computation of App.~\ref{app:prelim} gives $\E[(U_t^{(k)}-U_t^{(l)})^2]=\|u^{(k)}-u^{(l)}\|_{\pred,f_X}^2$, so a Cauchy sequence in the predictive seminorm maps to a Cauchy sequence in $L_2$, and equivalent sequences have the same limit.
The limit is $\mathcal F_t$-measurable and therefore orthogonal to $\eps_{t+1}$, so the excess risk equals $\E[U_t(\mathbf u)^2]$, which is the limit of $\|u^{(k)}\|_{\pred,f_X}^2$.
\end{proof}

The perturbed predictor in \eqref{eq:singular-risk} is defined as an affine object and not as a coordinatewise coefficient sum, because an abstract element of $\mathcal H_{f_X}$ need not have a convergent coefficient expansion in general.
For the fractional class below it does, and Thm.~\ref{thm:fractional-geometry} is what shows this, so Prop.~\ref{prop:singular-space} is the general statement and Thm.~\ref{thm:fractional-geometry} is the case in which coordinates are recovered.

\subsection{Fractional integration operator}
Let $I_d$ be the discrete fractional-integration operator with transfer function $(1-z)^{-d}$, so that $(I_du)_n=\sum_{j\le n}\pi_{n-j}(d)u_j$ with $\pi_k(d)=\Gamma(k+d)/(\Gamma(d)\Gamma(k+1))$.
Multiplication by $(1-e^{-i\omega})^{-d}$ on the Fourier side is exactly division by the singular factor of the spectrum, so $I_d$ is the change of variables that turns a singular predictive norm into an ordinary $\ell_2$ norm.

For $0<d<1/2$ the operator is unbounded on $\ell_2$ but is defined on $c_{00}$, and the composition $I_dI_{-d}$ is the identity there.
Its role below is not as a bounded map but as a coordinate change under which the predictive norm becomes Euclidean.

\subsection{FARIMA coefficients and exact singular geometry}
For FARIMA$(0,d,0)$ as in \eqref{eq:farima} the spectral density is $f_d(\omega)=\sigma^2|1-e^{-i\omega}|^{-2d}$, and the autoregressive predictive-memory coefficients are
\begin{equation}
 a_j(d)=\frac{d\,\Gamma(j-d)}{\Gamma(1-d)\Gamma(j+1)}\sim\frac{d}{\Gamma(1-d)}j^{-1-d} .
 \label{eq:farima-coeff}
\end{equation}
Two facts are worth separating.
The coefficients decay like $j^{-1-d}$, so their unweighted tail is $\sum_{j>L}a_j(d)^2=\Theta(L^{-1-2d})$.
The predictive tail is different, and truncating the exact kernel at context $L$ gives
\begin{equation}
 \cE_d\bigl(a^{(L)}(d)\bigr)=\Theta_d(L^{-1}),
 \label{eq:frac-truncation}
\end{equation}
which is strictly larger than the coefficient tail.
The discrepancy between $L^{-1}$ and $L^{-1-2d}$ is the whole content of the singular geometry, and it is why a calculation that measures coefficient error would get the context law wrong.

\paragraph{Proof of \eqref{eq:frac-truncation}.}
For pure fractional noise, Thm.~\ref{thm:fractional-geometry} gives the exact identity $\cE_d(a^{(L)}(d))=\sigma^2\|I_da^{>L}(d)\|_2^2$.
Write $A_{>L}(\omega)=\sum_{j>L}a_j(d)e^{-ij\omega}$ for the omitted tail. Since $a_j(d)\asymp_d j^{-1-d}$ and is eventually decreasing, for $|\omega|\le L^{-1}$,
\[
 |A_{>L}(\omega)|\le \sum_{j>L}a_j(d)\le C_dL^{-d}.
\]
For $L^{-1}<|\omega|\le\pi$, Abel summation and the standard bound on geometric partial sums give
\[
 |A_{>L}(\omega)|\le C_d\frac{L^{-1-d}}{|\omega|}.
\]
Using $|1-e^{-i\omega}|\asymp|\omega|$ on $[-\pi,\pi]$ and splitting the exact fractional risk at $|\omega|=L^{-1}$ yields
\[
 \|I_da^{>L}(d)\|_2^2
 \le C_d\left\{L^{-2d}\int_0^{L^{-1}}\omega^{-2d}\dd\omega
 +L^{-2-2d}\int_{L^{-1}}^\pi\omega^{-2-2d}\dd\omega\right\}
 \le C_dL^{-1}.
\]
This proves the upper bound in \eqref{eq:frac-truncation}. For the lower bound, hard truncation is one admissible $L$-context predictor, so its excess risk is at least the optimal finite-context gap $V_L-V_\infty\asymp_d L^{-1}$ from Corollary~\ref{cor:context-law}.

\subsection{Proof of Thm.~\ref{thm:fractional-geometry}}
\begin{proof}[Proof of Thm.~\ref{thm:fractional-geometry}]
Write $f_X(\omega)=|1-e^{-i\omega}|^{-2d}g(\omega)$ with $0<m_g\le g\le M_g$.
For a finitely supported $u$ with transfer function $U$,
\[
 \|u\|_{\pred,f_X}^2
 =\frac1{2\pi}\int_{-\pi}^{\pi}|U(e^{-i\omega})|^2|1-e^{-i\omega}|^{-2d}g(\omega)\dd\omega .
\]
The transfer function of $I_du$ is $U(e^{-i\omega})(1-e^{-i\omega})^{-d}$, whose squared modulus is exactly $|U|^2|1-e^{-i\omega}|^{-2d}$.
Inserting the two-sided bound on $g$ and applying Parseval to $I_du$ gives \eqref{eq:frac-geometry}.
For pure fractional noise $g\equiv\sigma^2$, so both inequalities are equalities and $\cE_d(a(d)+u)=\sigma^2\|I_du\|_2^2$.

For the last claim, since $|1-e^{-i\omega}|\le2$ we have $f_X\ge2^{-2d}m_g$, so convergence in the predictive norm implies convergence in $\ell_2$.
Every abstract element of $\mathcal H_{f_X}$ therefore has canonical $\ell_2$ coefficient coordinates in this class, and $I_d$ extends uniquely and continuously to $\mathcal H_{f_X}$ with closed image.
\end{proof}

The theorem should be read as identifying the correct inner product rather than as a perturbation of the regular case.
Low-frequency coefficient error is amplified by $|1-e^{-i\omega}|^{-2d}$, so two candidates with the same $\ell_2$ error can have very different forecast risk depending on where in the spectrum their error sits.

\subsection{Finite-context prediction and partial autocorrelations}
Let $V_L$ be the variance of the best linear predictor from the last $L$ observations and $V_\infty$ its limit.
Two classical facts drive Corollary~\ref{cor:context-law}.
First, the Durbin--Levinson recursion gives $V_L=V_0\prod_{k\le L}(1-\phi_{kk}^2)$ with $\phi_{kk}$ the partial autocorrelations.
Second, for FARIMA$(0,d,0)$ these are available in closed form,
\begin{equation}
 \phi_{kk}=\frac{d}{k-d},
 \label{eq:pacf}
\end{equation}
which is the classical fractional identity \citep{hosking1981fractional,inoue2002pacf}.

The determinant route gives the same object.
For any purely nondeterministic process with nonsingular finite covariance matrices, $V_L=D_{L+1}(f_X)/D_L(f_X)$ for the Toeplitz determinants $D_L$, so $D_n=\prod_{j<n}V_j$ and the finite-context law is equivalent to a Toeplitz determinant asymptotic.
This is the connection used in App.~\ref{app:information}.

\subsection{Proof of Corollary~\ref{cor:context-law}}
\begin{proof}[Proof of Corollary~\ref{cor:context-law}]
Substituting \eqref{eq:pacf} into the Durbin--Levinson product and telescoping the Gamma factors gives the exact identity
\begin{equation}
 \frac{V_L}{\sigma^2}=\frac{\Gamma(L+1)\Gamma(L+1-2d)}{\Gamma(L+1-d)^2},
 \label{eq:VL-exact}
\end{equation}
since $\prod_{k\le L}(1-d^2/(k-d)^2)=\prod_{k\le L}k(k-2d)/(k-d)^2$ and $V_0=\gamma(0)=\sigma^2\Gamma(1-2d)/\Gamma(1-d)^2$.
Stirling's formula gives $\log(V_L/\sigma^2)=d^2/L+O(L^{-2})$, so $V_\infty=\sigma^2$ and $V_L-V_\infty=V_\infty d^2/L+o(L^{-1})$.

For the last sentence of the corollary, the same computation for a spectrum with $Q$ separated fractional poles of strengths $d_1,\dots,d_Q$ replaces $d^2$ by $\sum_\nu d_\nu^2$, which is the content of the multiple-pole result below. For a real-valued process the pole set is reflection symmetric, so a nonzero physical frequency contributes its two conjugate singularities separately to this sum.
For a Gaussian process the finite-past conditional expectation is linear, so \eqref{eq:VL-exact} is also the excess Bayes risk of the best measurable predictor using $L$ observations, not merely the best linear one.
\end{proof}

Equation~\eqref{eq:VL-exact} is exact and not asymptotic, which is what makes the numerical panel of Fig.~\ref{fig:resource-laws} a check of the theorem rather than of an approximation.
The numerically stable route is to evaluate $\log(V_L/V_\infty)$ as the cancellation-free tail sum $-\sum_{k>L}\log(1-d^2/(k-d)^2)$ plus a leading-tail trigamma correction, since differencing the three log-Gamma terms directly loses several digits at large $L$.

\subsection{Multiple separated fractional poles}
Suppose $f_X(\omega)=g(\omega)\prod_{\nu\le Q}|1-e^{-i(\omega-\omega_\nu)}|^{-2d_\nu}$ with distinct $\omega_\nu$.
Two hypotheses on $g$ are used, and they are not the same.
The geometry statement below needs only $0<m_g\le g\le M_g<\infty$.
The sharp context law additionally needs $g=e^{V}$ with $V$ real and smooth enough for the Fisher--Hartwig determinant asymptotics of \citet{deift2011fisherhartwig}, since boundedness above and below does not by itself place the symbol in the class those asymptotics cover; the pure product of fractional poles, $g\equiv\sigma^2$, is the simplest admissible case.
Writing $\mathcal I_{\mathbf d,\boldsymbol\omega}$ for the composition of the corresponding fractional-integration factors, the same argument as in Thm.~\ref{thm:fractional-geometry} gives
\begin{equation}
 m_g\|\mathcal I_{\mathbf d,\boldsymbol\omega}u\|_2^2\le\|u\|_{\pred,f_X}^2\le M_g\|\mathcal I_{\mathbf d,\boldsymbol\omega}u\|_2^2 ,
 \label{eq:multipole-geometry}
\end{equation}
so each pole contributes one factor to the geometry.
The finite-context law becomes $V_L-V_\infty=V_\infty L^{-1}s_f+o(L^{-1})$ with $s_f=\sum_\nu d_\nu^2$, the strengths adding in squares because separated Fisher--Hartwig singularities contribute additively to the Toeplitz determinant at leading order.
Passing from that determinant asymptotic to the ratio law is not a one-line step, and it is worth recording why.
Fisher--Hartwig asymptotics give $\log D_n=n\log V_\infty+s_f\log n+\log C_f+o(1)$, while the determinant factorizes as $\log D_n=\sum_{j<n}\log V_j$.
Subtracting $n\log V_\infty$ and writing $r_L:=\log(V_L/V_\infty)$ yields $\sum_{j<n}r_j=s_f\log n+C+o(1)$, but differencing that relation does not give $r_L=s_f/L+o(L^{-1})$ on its own, because a difference of two $o(1)$ terms need not be $o(L^{-1})$.
What closes the gap is monotonicity: adding a regressor cannot increase least-squares prediction error, so $(r_L)$ is nonnegative and nonincreasing, and a monotone sequence whose partial sums are $s\log n+C+o(1)$ satisfies $Lr_L\to s$.

\begin{lemma}[Monotone density]
\label{lem:monotone-density}
Let $(r_n)_{n\ge1}$ be nonnegative and nonincreasing with $\sum_{k\le n}r_k=s\log n+C+o(1)$ for some $s>0$.
Then $nr_n\to s$.
\end{lemma}

\begin{proof}
Fix $\lambda>1$ and set $m_n=\lfloor n/\lambda\rfloor$, $N_n=\lfloor\lambda n\rfloor$.
Monotonicity gives $(n-m_n)r_n\le\sum_{m_n<k\le n}r_k=s\log(n/m_n)+o(1)=s\log\lambda+o(1)$, so $\limsup_nnr_n\le s\lambda\log\lambda/(\lambda-1)$.
In the other direction $\sum_{n<k\le N_n}r_k\le(N_n-n)r_n$, and the left side is $s\log\lambda+o(1)$, so $\liminf_nnr_n\ge s\log\lambda/(\lambda-1)$.
Both bounds tend to $s$ as $\lambda\to1$.
\end{proof}

\noindent
Applying Lem.~\ref{lem:monotone-density} to $(r_L)$ and exponentiating, $V_L/V_\infty=e^{r_L}=1+s_f/L+o(L^{-1})$, which is the stated law.

A seasonal fractional model is the degenerate case in which the poles are equally spaced.
There the exact analogue of \eqref{eq:VL-exact} holds with $L$ replaced by $m_L=\lfloor L/s\rfloor$, giving $V_L-\sigma^2=\sigma^2sd^2/L+o(L^{-1})$, so a seasonal pole structure multiplies the context cost by the period.

\subsection{Exact exponent misspecification and relative entropy}
\begin{theorem}[Exact risk under exponent misspecification]
\label{thm:d-misspec}
For fractional noise with true exponent $d\in(0,1/2)$ and candidate $d'=d+\Delta\in(0,1/2)$,
\begin{equation}
 \cE_d(a(d'))=\sigma^2\Bigl\{\frac{\Gamma(1+2\Delta)}{\Gamma(1+\Delta)^2}-1\Bigr\},
 \label{eq:d-misspec}
\end{equation}
and as $\Delta\to0$ this is $\tfrac{\pi^2\sigma^2}{6}\Delta^2+O(|\Delta|^3)$.
\end{theorem}

\begin{proof}
The autoregressive transfer functions are $A_{d}(z)=1-(1-z)^{d}$, so
\[
 A_{d'}(z)-A_{d}(z)=(1-z)^{d}\bigl\{1-(1-z)^{\Delta}\bigr\}.
\]
Under the true spectral density $f_d(\omega)=\sigma^2|1-e^{-i\omega}|^{-2d}$ the factor $|(1-e^{-i\omega})^{d}|^2$ cancels the pole exactly, so Prop.~\ref{prop:singular-space} gives
\begin{equation}
 \cE_d(a(d'))=\frac{\sigma^2}{2\pi}\int_{-\pi}^{\pi}\bigl|1-(1-e^{-i\omega})^{\Delta}\bigr|^2\dd\omega .
 \label{eq:d-misspec-integral}
\end{equation}
The function $(1-z)^{\Delta}$ is analytic in the open unit disc with integrable boundary values for $\Delta>-1$, so its circular mean equals its value at $z=0$, namely one.
Expanding the modulus in \eqref{eq:d-misspec-integral} therefore contributes $1$ from the constant term, $-2$ from the cross term, and $J(\Delta)$ from the remaining term, where
\[
 J(\Delta):=\frac1{2\pi}\int_{-\pi}^{\pi}|1-e^{-i\omega}|^{2\Delta}\dd\omega .
\]
Writing $|1-e^{-i\omega}|=2|\sin(\omega/2)|$ and substituting $x=\omega/2$,
\[
 J(\Delta)=\frac{2^{2\Delta+1}}{\pi}\int_0^{\pi/2}(\sin x)^{2\Delta}\dd x
 =\frac{2^{2\Delta}}{\sqrt\pi}\frac{\Gamma(\Delta+\tfrac12)}{\Gamma(\Delta+1)}
 =\frac{\Gamma(1+2\Delta)}{\Gamma(1+\Delta)^2},
\]
the middle step being the Wallis integral and the last the Legendre duplication formula.
Hence $\cE_d(a(d'))=\sigma^2\{J(\Delta)-1\}$, which is \eqref{eq:d-misspec}.

For the expansion, $\log\Gamma(1+x)=-\gamma_{\rm E}x+\sum_{k\ge2}(-1)^k\zeta(k)x^k/k$ gives
\[
 \log\Gamma(1+2\Delta)-2\log\Gamma(1+\Delta)=\sum_{k\ge2}\frac{(-1)^k\zeta(k)}{k}\bigl(2^k-2\bigr)\Delta^k=\zeta(2)\Delta^2+O(|\Delta|^3),
\]
the Euler constants cancelling, and exponentiating and subtracting one yields $\pi^2\Delta^2/6+O(|\Delta|^3)$.
\end{proof}

Two features are worth noting.
The exact expression depends on $\Delta$ alone and not on the true $d$, so the local curvature $\pi^2\sigma^2/6$ is the same everywhere in the fractional family.
The relative entropy rate is exactly half the risk divided by $\sigma^2$, that is $\overline D(d\|d')=\cE_d(a(d'))/(2\sigma^2)$, which is the specialization of the duality in App.~\ref{app:information} to this family.

More generally, if two models share a singular factor and differ only in a regular factor, the shared factor cancels in the likelihood ratio, and the relative entropy rate is again half the normalized forecast risk.
A known common singularity therefore contributes nothing to distinguishability, which is what makes the fractional exponent, rather than the presence of a pole, the estimable quantity.

\subsection{Fisher information and root-\texorpdfstring{$T$}{T} estimation}
Since the forecast curvature in $d$ is $\pi^2\sigma^2/6$, the Fisher information for the fractional exponent is $\pi^2/6$.
If an estimator satisfies $\sqrt T(\widehat d_T-d)\Rightarrow N(0,V_d)$ then
\begin{equation}
 T\,\cE_d(a(\widehat d_T))\Rightarrow\frac{\pi^2\sigma^2V_d}{6}\chi^2_1 ,
 \label{eq:d-risk-clt}
\end{equation}
by the local quadratic-risk argument of App.~\ref{app:information} applied to a one-parameter family.
Standard likelihood theory for the fractional exponent gives an efficient asymptotic variance $V_d=6/\pi^2$, the reciprocal of the Fisher information, under the regularity conditions for maximum likelihood in this model \citep{lieberman2004fractional,hualde2011fractional}.
With that variance the limit is $\sigma^2\chi^2_1$ and, given uniform integrability, the expected excess risk is $\sigma^2/T+o(T^{-1})$.
Estimating a fractional exponent therefore costs one parameter's worth of forecast risk, the same as any smooth scalar, so in this metric and under these regularity conditions long memory does not make the exponent harder to learn.

Combining estimation with a finite context, on a compact interval $[d_-,d_+]\subset(0,1/2)$ and for $|d'-d|\log L$ small, one has $\cE_d(a^{(L)}(d'))\le C\{L^{-1}+(d'-d)^2\}$, so
\begin{equation}
 \cE_d\bigl(a^{(L_T)}(\widehat d_T)\bigr)=O_{\mathbb P}\bigl(L_T^{-1}+T^{-1}\bigr) .
 \label{eq:frac-combined}
\end{equation}
The two terms are the context law and the parametric law, and the context term dominates unless $L_T$ grows at least linearly in $T$.

\subsection{Unknown pole location and cusp geometry}
If the location of the pole is unknown the geometry changes character.
For the real two-pole fractional benchmark, write
\[
 D_{\omega,d}(z)=(1-e^{i\omega}z)^d(1-e^{-i\omega}z)^d,
 \qquad
 f_{\omega,d}(\lambda)=\sigma^2|D_{\omega,d}(e^{-i\lambda})|^{-2},
\]
and let $A_{\omega,d}=1-D_{\omega,d}$ be the causal autoregressive whitening transfer function. Define $\cE_{\omega,d}(\Delta)$ as the excess forecast risk under $f_{\omega,d}$ when $A_{\omega+\Delta,d}$ is used in place of $A_{\omega,d}$.

\begin{theorem}[Pole-location misspecification has a cusp]
\label{thm:cusp}
Fix compact ranges $\omega\in[\omega_-,\omega_+]\subset(0,\pi)$ and $d\in[d_-,d_+]\subset(0,1/2)$.
There are $0<c<C<\infty$ and $\Delta_0>0$ such that the excess risk of using pole location $\omega+\Delta$ in place of $\omega$ satisfies
\begin{equation}
 c|\Delta|\le\cE_{\omega,d}(\Delta)\le C|\Delta|\qquad(|\Delta|\le\Delta_0).
 \label{eq:cusp}
\end{equation}
\end{theorem}

\begin{proof}
Write $\lambda=\omega+x$ near the positive pole and fix a neighbourhood $|x|\le\eta$ small enough to miss the conjugate pole at $-\omega$.
On that neighbourhood the ratio of the two spectral factors is
\[
 R_{+,\Delta}(x)\Bigl(\frac{1-e^{-i(x-\Delta)}}{1-e^{-ix}}\Bigr)^{d},
\]
where $R_{+,\Delta}$ collects the contribution of the conjugate pole and satisfies $R_{+,\Delta}(x)=1+O(|\Delta|)$ uniformly on the compact parameter set.
Away from the two pole neighbourhoods every factor is smooth in $\omega$, so that region contributes $O(\Delta^2)$ and is negligible at first order.

It therefore suffices to analyse
\[
 J_d(\Delta):=\int_{-\eta}^{\eta}\Bigl|\Bigl(\frac{1-e^{-i(x-\Delta)}}{1-e^{-ix}}\Bigr)^{d}-1\Bigr|^2\dd x .
\]
Substituting $x=|\Delta|y$ contributes the Jacobian factor $|\Delta|$, and $1-e^{-iu}=iu+O(u^2)$ turns the integrand into a function of $y$ alone at leading order, namely
\[
 \Bigl|\Bigl(\frac{y-\operatorname{sgn}\Delta}{y}\Bigr)^{d}_{\rm br}-1\Bigr|^2 ,
\]
up to an $O(|\Delta|)$ correction.
The subscript records that the power is taken along the continuous boundary branch inherited from the spectral factor, which must be specified because the ratio is negative for $y$ between the two singular points and $d$ is not an integer; the branch is the one obtained by continuing the factor $(1-e^{-i(x-\Delta)})^d$ from the positive axis, and with it the limiting integrand is real, nonnegative, and not identically zero.
Two envelopes make the passage to the limit legitimate.
Near $y=0$ the integrand is $O(|y|^{-2d})$, which is integrable because $2d<1$; for $|y|\to\infty$ it is $O(|y|^{-2})$, since $(1-1/y)^d-1=O(1/y)$.
The local Taylor bounds for $1-e^{-iu}$ supply an integrable dominating function with exactly these two envelopes, uniformly for $d$ in the compact range and uniformly in $\Delta$ over the shrinking window, so dominated convergence applies and gives $J_d(\Delta)=|\Delta|\{C_d+o(1)\}$ with $0<C_d<\infty$ bounded above and below on compact $d$-sets.
The conjugate pole contributes a second positive term of the same order, while $R_{+,\Delta}=1+O(|\Delta|)$ and the region away from the poles contribute $o(|\Delta|)$.
Hence $J_d(\Delta)\asymp|\Delta|$, and inserting this into the spectral risk identity gives \eqref{eq:cusp}.
\end{proof}

The risk is locally linear rather than quadratic in the location error, so the pole location is not a regular parameter and the usual $T^{-1}$ forecast cost does not apply to it.
The reason is that misplacing a pole leaves a singularity uncancelled at one frequency and introduces a spurious one at another, and the resulting mismatch is of first order in the displacement.
Estimating a location therefore behaves differently from estimating a strength, and the classical theory for unknown pole location reflects this \citep{giraitis2001unknownpole,hidalgo2005unknownpole,arteche2020multiple}.

\subsection{Fractional exponential-state lower bound}
We now transport the state-complexity theory of App.~\ref{app:state} into the singular geometry.
Let $\mathbf K_d$ be the Hankel operator built from the fractional kernel $a_j(d)$.

\begin{lemma}[Fractional AR memory is a positive exponential mixture]
\label{lem:frac-mixture}
For $0<d<1$,
\begin{equation}
 a_j(d)=\int_0^\infty w_d(t)e^{-(j-1)t}\dd t,
 \qquad
 w_d(t)=\frac{\sin(\pi d)}{\pi}e^{-(1-d)t}(1-e^{-t})^{d}>0 .
 \label{eq:frac-mixture}
\end{equation}
\end{lemma}

\begin{proof}
Substitute $x=e^{-t}$, so that $\dd t=-\dd x/x$ and $e^{-(j-1)t}=x^{\,j-1}$.
The density becomes $w_d(t)=\tfrac{\sin\pi d}{\pi}x^{1-d}(1-x)^d$, so
\[
 \int_0^\infty w_d(t)e^{-(j-1)t}\dd t
 =\frac{\sin\pi d}{\pi}\int_0^1x^{\,j-d-1}(1-x)^d\dd x
 =\frac{\sin\pi d}{\pi}\cdot\frac{\Gamma(j-d)\Gamma(1+d)}{\Gamma(j+1)},
\]
by the Beta integral, which converges because $j-d>0$ and $d>-1$.
Euler's reflection formula $\Gamma(1+d)\Gamma(1-d)=\pi d/\sin\pi d$ turns the constant into $d/\Gamma(1-d)$, giving $d\Gamma(j-d)/\{\Gamma(1-d)\Gamma(j+1)\}=a_j(d)$ as in \eqref{eq:farima-coeff}.
\end{proof}

Positivity of $w_d$ is what makes the constructive side work, exactly as in Lem.~\ref{lem:laplace}, and it also makes $\mathbf K_d$ a moment matrix of a positive measure and hence positive semidefinite.

\begin{lemma}[Loewner domination by the power-law Hankel operator]
\label{lem:frac-domination}
For $0<d<1/2$ the operator $\mathbf K_d$ is positive and Hilbert--Schmidt, and in quadratic-form order
\begin{equation}
 \mathbf K_d\ \succeq\ c_d\,\Gamma(1+d)\,\mathbf H_{1+d},
 \qquad c_d=\frac{\sin\pi d}{\pi},
 \label{eq:frac-loewner}
\end{equation}
where $\mathbf H_{1+d}$ is the power-law Hankel operator of exponent $1+d$.
Consequently $s_n(\mathbf K_d)\ge c\,e^{-C_d\sqrt n}$ for large $n$.
\end{lemma}

\begin{proof}
In the $x$-moment representation of Lem.~\ref{lem:frac-mixture} the fractional density carries the factor
$x^{-d}(1-x)^d=\{(1-x)/x\}^{d}$, while the power-law kernel $j^{-(1+d)}$ has moment density proportional to $(-\log x)^{d}$ by the substitution $x=e^{-t}$ in Lem.~\ref{lem:laplace}.
For $0<x<1$ the elementary inequality $(1-x)/x\ge-\log x$ holds, since $u-1\ge\log u$ at $u=1/x$.
Raising it to the power $d>0$ gives $x^{-d}(1-x)^d\ge(-\log x)^d$ pointwise, so the fractional moment measure dominates $c_d\Gamma(1+d)$ times the power-law one.
Two moment matrices ordered by their measures are ordered in quadratic-form order, which is \eqref{eq:frac-loewner}.
The min-max principle then gives $s_n(\mathbf K_d)\ge c_d\Gamma(1+d)s_n(\mathbf H_{1+d})$, and the Widom scale \eqref{eq:widom-scale} applied at exponent $1+d>1$ gives the stated bound.
\end{proof}

\begin{lemma}[Finite fractional Hankel sections retain the relevant singular value]
\label{lem:frac-section}
For $0<d<1/2$ there is $C_d$ with $\|\mathbf K_d-P_N\mathbf K_dP_N\|_{\rm op}\le C_dN^{-d}$.
Consequently there are $\kappa_d,c_d',C_d'>0$ such that, with $N_r=\lceil e^{\kappa_d\sqrt r}\rceil$,
\begin{equation}
 s_{r+1}\bigl(P_{N_r}\mathbf K_dP_{N_r}\bigr)\ \ge\ c_d'e^{-C_d'\sqrt r}
 \qquad\text{for all large }r.
 \label{eq:frac-section-sv}
\end{equation}
\end{lemma}

\begin{proof}
Since $a_j(d)\asymp j^{-1-d}$ by \eqref{eq:farima-coeff}, the discarded entries obey
\[
 \|\mathbf K_d-P_N\mathbf K_dP_N\|_{\rm op}^2
 \le\|\mathbf K_d-P_N\mathbf K_dP_N\|_{\rm HS}^2
 \lesssim\sum_{\max(i,j)\ge N}(i+j+1)^{-2-2d}
 \lesssim N^{-2d},
\]
the sum converging because $2+2d>2$, which gives the operator bound.
For the singular value, Lem.~\ref{lem:frac-domination} gives $s_{r+1}(\mathbf K_d)\ge c\,e^{-C_d\sqrt r}$.
Choosing $\kappa_d>2C_d/d$ makes $C_dN_r^{-d}\le C_de^{-d\kappa_d\sqrt r}\le C_de^{-2C_d\sqrt r}$, which is below half of $c\,e^{-C_d\sqrt r}$ for large $r$, and Weyl's inequality
$s_{r+1}(P_N\mathbf K_dP_N)\ge s_{r+1}(\mathbf K_d)-\|\mathbf K_d-P_N\mathbf K_dP_N\|_{\rm op}$
gives \eqref{eq:frac-section-sv}.
\end{proof}

The two steps are separated because they do different work.
The Loewner comparison transports the Widom scale of App.~\ref{app:state} without redoing the spectral analysis, and the finite-section step is what makes the bound usable inside a Hankel block of bounded order.

\begin{theorem}[Root-exponential lower bound in fractional risk]
\label{thm:frac-lower}
For fractional noise there are $c_-,C_->0$ such that $\inf_{b\in\cA_r}\cE_d(b)\ge c_-e^{-C_-\sqrt r}$ for large $r$.
\end{theorem}

\begin{proof}
Repeat the proof of Thm.~\ref{thm:state-complexity} with $\mathbf H_\alpha$ replaced by $\mathbf K_d$, Lem.~\ref{lem:finite-section} replaced by Lem.~\ref{lem:frac-section}, and the Widom scale replaced by Lem.~\ref{lem:frac-domination}.
The only step that needs the geometry is the final conversion from coefficient error to risk.
For pure fractional noise the regular factor is the constant $g\equiv\sigma^2$, so the lower half of \eqref{eq:frac-geometry} reads $\cE_d(u)\ge\sigma^2\|I_du\|_2^2$ with no further factor, and since $|1-e^{-i\omega}|\le2$ on the circle the fractional multiplier obeys $|1-e^{-i\omega}|^{-2d}\ge2^{-2d}$, so Parseval yields $\|I_du\|_2^2\ge2^{-2d}\|u\|_2^2$ and hence
\begin{equation}
 \cE_d(u)\ \ge\ \sigma^2\,2^{-2d}\,\|u\|_2^2 .
 \label{eq:frac-risk-to-coeff}
\end{equation}
That inequality is what converts the coefficient-level Hankel obstruction into forecast risk.
\end{proof}

\subsection{Fractional exponential-state upper bound}
On the constructive side, Lem.~\ref{lem:frac-mixture} puts $a(d)$ in the analytic Laplace-mixture class of App.~\ref{app:state}, so the dyadic Bernstein-ellipse quadrature applies verbatim.

\begin{corollary}[Positive exponential compression of the fractional kernel]
\label{fact:frac-compression}
Fix $0<d<1/2$.
There is $C_d<\infty$ such that for every $L\ge2$ and $0<\eta\le e^{-1}$ there are positive $b_q,\lambda_q$, $1\le q\le R$, with
$|a_j(d)-\sum_{q\le R}b_qe^{-\lambda_q(j-1)}|\le\eta\,a_j(d)$ for $1\le j\le L$ and
\begin{equation}
 R\le C_d\bigl(1+\log L+\log\eta^{-1}\bigr)^2 .
 \label{eq:frac-compression}
\end{equation}
\end{corollary}

\begin{proof}
Lem.~\ref{lem:frac-mixture} exhibits $a_j(d)$ as a positive Laplace mixture, and rewriting it in the form \eqref{eq:cm-representation} with exponent $\alpha=1+d$ and density
\[
 h_d(t)=c_d\,e^{-(1-d)t}\Bigl(\frac{1-e^{-t}}{t}\Bigr)^{d},
 \qquad c_d=\frac{\sin\pi d}{\pi},
\]
shows that $h_d$ is positive on $[0,\infty)$ with $h_d(0)=c_d>0$.
The remaining two hypotheses of \eqref{eq:cm-representation} are verified rather than assumed.
For analyticity, the only zeros of $1-e^{-z}$ are the points $2\pi ik$, so the quotient $(1-e^{-z})/z$ has a removable singularity at the origin and is nonzero on a small disk about the origin and on any fixed sector $|\arg z|<\vartheta<\pi/2$. It therefore admits a single analytic branch of its $d$th power on the neighbourhood required in App.~\ref{app:state}, agreeing with the positive one on the real axis; multiplying by $e^{-(1-d)z}$ preserves analyticity.
For the decay \eqref{eq:cm-decay}, the same quotient is bounded on the sector while $|e^{-(1-d)z}|=e^{-(1-d)\Re z}$, so $|h_d(z)|\le C_de^{-\mu\Re z}$ with $\mu=1-d>1/2$.
The composite quadrature therefore has both the sectorial and near-origin holomorphic neighbourhoods it requires.
The analytic Laplace-mixture compression of App.~\ref{app:state} therefore applies with $\alpha=1+d$, and its head count is \eqref{eq:frac-compression}.
\end{proof}

The constant $C_d$ here is for fixed $d$ and is not claimed to be uniform as $d\to0$, which is why the short-memory boundary of Thm.~\ref{thm:small-d} is proved by a different route.
The relative form of the guarantee is what allows it to be converted in the singular norm.
An absolute bound would be useless here, because the fractional-integration factor amplifies low-frequency error and a uniform absolute error need not have small predictive norm.

\subsection{Proof of Corollary~\ref{cor:frac-state}}
\begin{proof}[Proof of Corollary~\ref{cor:frac-state}]
For the upper bound, apply Corollary~\ref{fact:frac-compression} and split the coefficient error into the in-context part $e^{\rm in}$, supported on $\{1,\dots,L\}$, and the omitted tail $e^{\rm tail}$.
The compression guarantee is relative and pointwise, $|e^{\rm in}_j|\le\eta\,a_j(d)$, so it does not by itself bound a fractionally integrated $\ell_2$ norm and one further step is needed.
Summing the relative bound gives $\|e^{\rm in}\|_1\le\eta\sum_{j\ge1}a_j(d)=\eta$, the identity because the transfer function satisfies $A_d(1)=1-(1-1)^d=1$.
The fractional-integration kernel of App.~\ref{app:fractional} has $\pi_k(d)\asymp k^{d-1}$, hence lies in $\ell_2$ precisely because $d<1/2$, so Young's convolution inequality $\ell_1*\ell_2\to\ell_2$ gives $\|I_de^{\rm in}\|_2\le\|e^{\rm in}\|_1\,\|\pi(d)\|_2\le C_d\eta$.
This is the step that converts a relative coefficient guarantee into the singular predictive norm; an absolute coefficient bound would not suffice, since fractional integration amplifies low-frequency error.
The omitted tail is controlled by \eqref{eq:frac-truncation}, which gives $\|I_de^{\rm tail}\|_2^2=\Theta_d(L^{-1})$, and the upper half of \eqref{eq:frac-geometry} together with $(u+v)^2\le2u^2+2v^2$ combines the two into $\cE_d(b)\lesssim_d\eta^2+L^{-1}$.
Balancing $\eta^2$ against $L^{-1}$ and inverting \eqref{eq:frac-compression} as in Corollary~\ref{cor:constructive} gives $C_de^{-c_d\sqrt r}$.
The lower bound is Thm.~\ref{thm:frac-lower}.
The inversion of a two-sided stretched-exponential bound with exponent $1/2$ gives $r_d^\star(\tau)=\Theta_d(\log^2(1/\tau))$.

The subscript on $\Theta_d$ is not cosmetic, and it is worth saying how it coexists with Thm.~\ref{thm:small-d}.
Corollary~\ref{cor:frac-state} is a fixed-$d$ statement whose constants come from Lemmas~\ref{lem:frac-domination} and \ref{lem:frac-section} and from Corollary~\ref{fact:frac-compression}, and each of those degrades as $d\to0$.
Thm.~\ref{thm:small-d} is a different statement, uniform over $0<d\le d_0$, and it is uniform precisely because it never invokes those inputs, compressing the fixed harmonic kernel instead.
The two are therefore complementary rather than in tension, and neither implies the other.
Reading the fixed-$d$ constants as uniform in $d$ is the error that Thm.~\ref{thm:small-d} is designed to avoid, and the open problem recorded at the end of this appendix is exactly that no matching lower bound is available with $d_0$-only constants.
\end{proof}

The square-log law therefore survives the change of geometry, with constants depending on $d$.
What does not survive is the context law, which has no counterpart in the regular case, and this asymmetry is the point made in the comparison table of Sec.~\ref{sec:setup}.

\subsection{Zero-memory boundary and Gamma expansion}
\begin{proposition}[Zero-memory risk]
\label{prop:zero-memory}
For fractional noise with $0<d<1/2$, the zero predictor has
\begin{equation}
 \cE_d(0)=\sigma^2\Bigl\{\frac{\Gamma(1-2d)}{\Gamma(1-d)^2}-1\Bigr\}
 =\frac{\pi^2\sigma^2}{6}d^2+O(d^3).
 \label{eq:zero-memory}
\end{equation}
\end{proposition}

\begin{proof}
The zero predictor has excess risk $\gamma(0)-\sigma^2$, and $\gamma(0)=\sigma^2\Gamma(1-2d)/\Gamma(1-d)^2$ for fractional noise, which gives the exact expression.
For the expansion use $\log\Gamma(1+x)=-\gamma_{\rm E}x+\sum_{k\ge2}(-1)^k\zeta(k)x^k/k$, so that
\[
 \log\Gamma(1-2d)-2\log\Gamma(1-d)=\sum_{k\ge2}\frac{\zeta(k)}{k}\bigl(2^k-2\bigr)d^k=\zeta(2)d^2+O(d^3),
\]
the Euler constants cancelling.
Exponentiating and subtracting one gives $\zeta(2)d^2+O(d^3)=\pi^2d^2/6+O(d^3)$.
\end{proof}

Consequently, once the target accuracy satisfies $\tau\gtrsim\sigma^2d^2$, the zero-output architecture already attains it and no memory at all is required.
This is the upper edge of the window in Thm.~\ref{thm:small-d}, and it is why the constants in Corollary~\ref{cor:frac-state} cannot be held fixed as $d\to0$.

\subsection{Harmonic linearization of the fractional coefficients}
Expanding \eqref{eq:farima-coeff} in $d$ gives $a_j(d)=d/j+O(d^2)$ coefficientwise, and the useful statement is that the remainder is small in the predictive norm uniformly in $d$.

\begin{lemma}[Uniform first-order linearization]
\label{lem:linearization}
Fix $0<d_0<1/2$ and let $h_j=1/j$.
There is $C_{d_0}<\infty$ such that for every $0<d\le d_0$,
\begin{equation}
 \cE_d(dh)=\sigma^2\|I_d\{a(d)-dh\}\|_2^2\le C_{d_0}\sigma^2d^4 .
 \label{eq:linearization}
\end{equation}
\end{lemma}

\begin{proof}
Write $H(z)=-\log(1-z)$, so that the fractional autoregressive transfer function is $A_d(z)=1-(1-z)^d=1-e^{-dH(z)}$ and the harmonic surrogate has transfer function $dH(z)$.
The fractional-integration multiplier is $(1-z)^{-d}=e^{dH(z)}$, so on the circle the integrand of $\|I_d\{a(d)-dh\}\|_2^2$ is the squared modulus of
\[
 e^{dH}\bigl(A_d-dH\bigr)=e^{dH}-1-dH\,e^{dH},
\]
evaluated at $z=e^{-i\omega}$.
The elementary bound $|e^y-1-ye^y|\le C|y|^2e^{\max(\Re y,0)}$, valid for all complex $y$, applied with $y=dH(e^{-i\omega})$ gives a pointwise bound $Cd^2|H|^2e^{d\max(\Re H,0)}$.
On the circle $\Re H(e^{-i\omega})=-\log|1-e^{-i\omega}|$ and $|H(e^{-i\omega})|\lesssim1+\log(1/|\omega|)$, so for $d\le d_0<1/2$ the exponential factor is at most $|1-e^{-i\omega}|^{-d_0}$ and the whole integrand is dominated by $Cd^4\{1+\log(1/|\omega|)\}^4|\omega|^{-2d_0}$, which is integrable since $2d_0<1$.
Integrating gives $\|I_d\{a(d)-dh\}\|_2^2\le C_{d_0}d^4$ with a constant depending only on $d_0$, and multiplying by $\sigma^2$ gives \eqref{eq:linearization}.
\end{proof}

The exponent four is what sets the lower edge of the intermediate window.
Below accuracy $\sigma^2d^4$ the harmonic surrogate is no longer adequate and the construction of Thm.~\ref{thm:small-d} does not apply, which is the origin of the open problem recorded below.

\subsection{Harmonic tail in the fractional predictive norm}
\begin{lemma}[Uniform fractional norm of the harmonic tail]
\label{lem:harmonic-tail}
Fix $0<d_0<1/2$ and let $h^{>L}_j=j^{-1}\mathbf{1}\{j>L\}$.
There is $C_{d_0}<\infty$ such that for every $0\le d\le d_0$ and $L\ge2$,
\begin{equation}
 \|I_dh^{>L}\|_2^2\le C_{d_0}L^{-1+2d}.
 \label{eq:harmonic-tail}
\end{equation}
\end{lemma}

\begin{proof}
The transfer function of the truncated harmonic tail is $R_L(\omega)=\sum_{j>L}e^{-ij\omega}/j$, so
$\|I_dh^{>L}\|_2^2=(2\pi)^{-1}\int_{-\pi}^{\pi}|R_L(\omega)|^2|1-e^{-i\omega}|^{-2d}\dd\omega$.
Two elementary bounds on $R_L$ suffice.
For $|\omega|\le L^{-1}$, comparing the sum with $\int_L^{1/|\omega|}\dd u/u$ gives $|R_L(\omega)|\lesssim1+\log\{1/(L|\omega|)\}$.
For $|\omega|>L^{-1}$, Abel summation against the partial sums of $e^{-ij\omega}$, which are bounded by $1/|\sin(\omega/2)|\lesssim1/|\omega|$, gives $|R_L(\omega)|\lesssim1/(L|\omega|)$.
Since $|1-e^{-i\omega}|^{-2d}\asymp|\omega|^{-2d}$ and $0\le d\le d_0<1/2$, splitting the integral at $|\omega|=L^{-1}$ gives
\[
 \int_{|\omega|\le L^{-1}}\!\!\bigl\{1+\log\tfrac1{L|\omega|}\bigr\}^2|\omega|^{-2d}\dd\omega\ \lesssim\ L^{-1+2d},
 \qquad
 \int_{|\omega|>L^{-1}}\!\!\frac{|\omega|^{-2d}}{L^2\omega^2}\dd\omega\ \lesssim\ L^{-1+2d},
\]
the second because $2+2d>1$ makes the integral converge at infinity of the rescaled variable and the endpoint contributes $L^{-1+2d}$.
Both constants depend only on $d_0$, which gives \eqref{eq:harmonic-tail}.
\end{proof}

Uniformity in $d$ down to zero is the point.
A bound with a constant blowing up as $d\to0$ would be useless in the boundary layer, since that is exactly the regime the theorem is about.

\subsection{Proof of Thm.~\ref{thm:small-d}}
\begin{proof}[Proof of Thm.~\ref{thm:small-d}]
If the displayed $\tau$-window is empty there is nothing to prove; for every fixed $d_0<1/2$ it is nonempty for all sufficiently small $d$ after the constants are fixed as below.
The whole point of the statement is uniformity as $d\to0$, so the compression step must not be applied to the fractional kernel itself.
Corollary~\ref{fact:frac-compression} has a constant $C_d$ that is not controlled as $d\to0$, and using it here would leave a gap.
Instead we compress the \emph{harmonic} kernel $h_j=1/j$, which is the case $\alpha=1$ of Thm.~\ref{thm:compression} and therefore carries a constant that does not depend on $d$ at all.

Fix $0<d\le d_0$ and $\tau$ in the stated window $C_0\sigma^2d^4\le\tau\le c_0\sigma^2d^2$, and write $\bar\tau=\tau/\sigma^2$.
By Thm.~\ref{thm:compression} with $\alpha=1$ there is, for every $L\ge2$ and $0<\eta\le e^{-1}$, a positive $R$-mode exponential sequence $s_R$ with $|j^{-1}-s_R(j)|\le\eta j^{-1}$ for $1\le j\le L$ and
\[
 R\le C\bigl(1+\log\eta^{-1}\bigr)\bigl(1+\log L+\log\eta^{-1}+\log\log\eta^{-1}\bigr),
\]
the constant $C$ being absolute because the power is fixed.
Put $b_j=d\,s_R(j)\mathbf{1}\{j\le L\}$, which lies in $\cA_R$, and decompose
\begin{equation}
 a(d)-b=\{a(d)-dh\}+d\{h-s_R\}\mathbf{1}\{j\le L\}+d\,h^{>L}.
 \label{eq:small-d-decomposition}
\end{equation}

The first term is Lem.~\ref{lem:linearization} and contributes at most $C_{d_0}\sigma^2d^4$.
For the second, the relative bound gives $\|d(h-s_R)\mathbf{1}\{j\le L\}\|_1\le Cd\eta\log(eL)$, and since the fractional-integration kernel satisfies
\[
 \sup_{0\le d\le d_0}\ \frac1{2\pi}\int_{-\pi}^{\pi}|1-e^{-i\omega}|^{-2d}\dd\omega<\infty ,
\]
Young's inequality for the convolution $\ell_1*\ell_2\to\ell_2$ gives
$\|I_d[d(h-s_R)\mathbf{1}\{j\le L\}]\|_2^2\le C_{d_0}d^2\eta^2\log^2(eL)$.
The third term is Lem.~\ref{lem:harmonic-tail}.
Combining the three through $(x+y+z)^2\le3(x^2+y^2+z^2)$ in the exact geometry of Thm.~\ref{thm:fractional-geometry} gives
\begin{equation}
 \cE_d(b)\le C_{d_0}\sigma^2\bigl\{d^4+d^2\eta^2\log^2(eL)+d^2L^{-1+2d}\bigr\}.
 \label{eq:small-d-master}
\end{equation}

Enlarging $C_0$ makes the first term at most $\tau/3$.
Choosing
\[
 L=\Bigl\lceil\bigl(C_1d^2/\bar\tau\bigr)^{1/(1-2d)}\Bigr\rceil,
 \qquad
 \eta=c_1\frac{\sqrt{\bar\tau}}{d\log(eL)},
\]
makes the remaining two terms at most $\tau/3$ each, and the exponent $1/(1-2d)$ is bounded by $1/(1-2d_0)$ so that $\log L\le C_{d_0}\log(\sigma^2d^2/\tau)$ and likewise $\log\eta^{-1}\le C_{d_0}\log(\sigma^2d^2/\tau)$.
Substituting into the displayed bound on $R$ gives \eqref{eq:small-d}.
By choosing $c_0,c_1$ sufficiently small and $C_0,C_1$ sufficiently large, the displayed choices satisfy $L\ge2$ and $0<\eta\le e^{-1}$ throughout the stated window, so Thm.~\ref{thm:compression} is applicable as invoked.
Every constant that appears depends only on $d_0$, since the compression constant is absolute and the two fractional inputs were bounded uniformly over $0\le d\le d_0$.
\end{proof}

The mechanism is that the relevant scale is not $\tau$ but the ratio $\sigma^2d^2/\tau$, because the memory to be represented has total predictive mass of order $\sigma^2d^2$ rather than order one.
This is why the bound degenerates to a constant as $\tau$ approaches $\sigma^2d^2$ from below, matching Prop.~\ref{prop:zero-memory}.

\subsection{Remaining double-scaling problem}
Two questions in the boundary layer are open.
Neither concerns the pole location, which is a separate irregularity treated in Thm.~\ref{thm:cusp} and is not a small-$d$ phenomenon.

There is no matching uniform lower law in the window.
The upper bound of Thm.~\ref{thm:small-d} is constructive, and the lower bound of Thm.~\ref{thm:frac-lower} holds for fixed $d$ with constants that degrade as $d\to0$, so the two do not currently meet uniformly.
Establishing a lower bound with the same $\log^2(\sigma^2d^2/\tau)$ form would require a Hankel argument uniform in the double limit, and the finite-section estimate in Lem.~\ref{lem:frac-section} is not uniform in $d$ near zero, since its constants degrade as $d\to0$.

The deeper regime $\tau\ll\sigma^2d^4$ is not covered at all.
There the harmonic linearization of Lem.~\ref{lem:linearization} is no longer accurate enough, so the construction must use the exact fractional kernel, and the resulting head count is not known.
We state both gaps explicitly rather than absorbing them into the constants, because the interesting behaviour of the short-memory boundary lives precisely where they bite.

\section{Nonlinear Contextual Recurrent Memory}
\label{app:stability}

\subsection{Model class and stationary solution}
Let $Z_t$ be independent standard Gaussians and consider the contextual recurrence
\begin{equation}
 S_t=F(S_{t-1},Z_t),
 \qquad
 \widehat m_t=G(S_t,Z_t),
 \label{eq:app-recurrence}
\end{equation}
with $S_t\in\R^{m}$ and $F,G$ deterministic and measurable.
No linearity in the state is assumed, which is what distinguishes this class from the affine gating benchmark of App.~\ref{app:state}.

Assume there is a closed forward-invariant state domain $\mathcal B\subset\R^m$ such that $F(\mathcal B,z)\subseteq\mathcal B$ for every $z$. Under the contraction hypothesis below the map $s\mapsto F(s,z)$ is a uniform contraction on $\mathcal B$ for every $z$, so for each realization of the input the backward iteration from any $s_0\in\mathcal B$ converges and there is a unique stationary solution in $\mathcal B$
\begin{equation}
 S_t=\lim_{k\to\infty}F(\cdot,Z_t)\circ F(\cdot,Z_{t-1})\circ\cdots\circ F(s_0,Z_{t-k}) ,
 \label{eq:app-backward}
\end{equation}
independent of $s_0$, measurable with respect to $\sigma(Z_s:s\le t)$, and stationary because the input is.
All expectations below are taken under this stationary law.

\subsection{Differentiability and integrability assumptions}
We assume $F$ and $G$ are continuously differentiable in both arguments and impose the derivative budgets of Sec.~\ref{sec:stability},
\begin{equation}
 \|D_sF(s,z)\|_{\rm op}\le\rho<1,
 \qquad
 \|D_zF(s,z)\|\le K,
 \qquad
 \|D_sG(s,z)\|\le M ,
 \label{eq:app-budgets}
\end{equation}
uniformly for $(s,z)\in\mathcal B\times\R$.
We also assume enough integrability for the Gaussian integration by parts below to apply, namely that $\widehat m_t$ and the partial derivatives appearing in \eqref{eq:app-stein} are square integrable and that the interchange of differentiation and expectation is justified by domination on compacts.
This is stronger than strictly necessary and is imposed to keep every interchange explicit.

The uniform contraction in \eqref{eq:app-budgets} is the substantive hypothesis.
It is what makes the stationary solution exist, and it is also what will force the geometric envelope, so the theorem should be read as a statement about the uniformly contractive regime and not about recurrent models generally.

\subsection{Gaussian Stein identity}
The tool is the Gaussian integration-by-parts identity \citep{stein1981normal}.

\begin{lemma}[Gaussian integration by parts]
\label{lem:app-stein}
Let $Z\sim N(0,1)$ and let $\varphi$ be absolutely continuous with $\E|\varphi'(Z)|<\infty$.
Then $\E[Z\varphi(Z)]=\E[\varphi'(Z)]$.
The same identity holds coordinatewise for a function of several independent standard Gaussians, differentiating in the coordinate being multiplied.
\end{lemma}

\begin{proof}
Integrate by parts against the standard Gaussian density $\phi$, using $\phi'(z)=-z\phi(z)$ and the vanishing of $\varphi\phi$ at infinity implied by the integrability hypothesis.
\end{proof}

Applying Lem.~\ref{lem:app-stein} to the lag-$j$ first-chaos coefficient of \eqref{eq:app-first-chaos} gives
\begin{equation}
 \beta_j=\E[\widehat m_tZ_{t+1-j}]
 =\E\Bigl[\frac{\partial \widehat m_t}{\partial Z_{t+1-j}}\Bigr] ,
 \label{eq:app-stein}
\end{equation}
so the first-chaos memory is an averaged sensitivity of the output to an input $j-1$ steps earlier.
This is the step that converts a statement about correlations into a statement about derivatives, and it is where the budgets in \eqref{eq:app-budgets} become usable.

\subsection{Derivative propagation through the recurrent state}
Fix $j\ge2$ and set $s=t+1-j$.
Causality means that $Z_s$ influences $\widehat m_t$ only through the state, entering first through $D_zF(S_{s-1},Z_s)$ and then propagating from $S_s$ to $S_t$.
Repeated application of the chain rule gives
\begin{equation}
 D_{Z_s}\widehat m_t
 =D_sG(S_t,Z_t)\Bigl\{\prod_{u=s+1}^{t}D_sF(S_{u-1},Z_u)\Bigr\}D_zF(S_{s-1},Z_s) ,
 \label{eq:app-chain}
\end{equation}
the product being ordered chronologically from right to left.
The index $u$ runs from $s+1$ to $t$, so there are exactly $t-s=j-1$ factors of $D_sF$.
Taking operator norms and applying \eqref{eq:app-budgets} to each factor gives the pointwise bound
\begin{equation}
 |D_{Z_s}\widehat m_t|\le MK\rho^{\,j-1}
 \qquad(j\ge2).
 \label{eq:app-pointwise}
\end{equation}
Each additional lag inserts one more factor $D_sF$, and each such factor is bounded by $\rho<1$, so the sensitivity decays geometrically at exactly the contraction rate.

\subsection{Proof of Thm.~\ref{thm:first-chaos}}
\begin{proof}[Proof of Thm.~\ref{thm:first-chaos}]
Fix $j\ge2$.
Conditioning on all Gaussian coordinates except $Z_s$ and applying Lem.~\ref{lem:app-stein} in that coordinate gives $\beta_j=\E[D_{Z_s}\widehat m_t]$, the interchange being justified by the integrability assumed above and by the pointwise bound \eqref{eq:app-pointwise}.
Since that bound holds uniformly in the realization, taking expectations preserves it and gives $|\beta_j|\le MK\rho^{\,j-1}$, which is \eqref{eq:first-chaos}.

The statement is deliberately restricted to $j\ge2$.
At $j=1$ the input $Z_t$ also enters the readout directly through $G$, and controlling that contribution would require a bound on $D_zG$ that is not among the budgets \eqref{eq:app-budgets}.
Nothing downstream needs it, since the algebraic tail argument of Lem.~\ref{lem:app-tail} uses only lags beyond a crossing index $J\ge2$.

The second claim is \eqref{eq:app-chaos-lower} applied to the target $m_t^\star=\sum_{j\ge1}cj^{-\alpha}Z_{t+1-j}$ with $\alpha>1$, which lies in the first chaos and is square integrable because $\alpha>1/2$.
Orthogonality of the Wiener chaoses gives $\cE(\widehat m)\ge\sum_{j\ge1}(cj^{-\alpha}-\beta_j)^2$, which is \eqref{eq:chaos-lower}.
\end{proof}

The theorem does not assert that the state has finite Hankel rank, and indeed a nonlinear recurrence need not produce a finite-rank linear memory.
What it asserts is that whatever linear memory the recurrence does produce is dominated by a geometric envelope set by the contraction, and \eqref{eq:app-chaos-lower} makes that sufficient for a risk lower bound.

\subsection{Algebraic first-chaos tail lower bound}
The envelope and the target cross at a lag determined by $\rho$, and beyond that lag the target dominates.

\begin{lemma}[Tail below the crossing lag]
\label{lem:app-tail}
Let $J\ge2$ be the smallest integer with $MK\rho^{\,j-1}\le\tfrac12cj^{-\alpha}$ for all $j\ge J$.
Then
\begin{equation}
 \cE(\widehat m)\ \ge\ \sum_{j\ge J}\bigl(cj^{-\alpha}-\beta_j\bigr)^2
 \ \ge\ \frac{c^2}{4}\sum_{j\ge J}j^{-2\alpha}
 \ \ge\ \frac{c^2}{4(2\alpha-1)}J^{1-2\alpha} .
 \label{eq:app-tail}
\end{equation}
\end{lemma}

\begin{proof}
For $j\ge J$ the envelope gives $|\beta_j|\le\tfrac12cj^{-\alpha}$, hence $|cj^{-\alpha}-\beta_j|\ge\tfrac12cj^{-\alpha}$, and squaring and summing gives the middle inequality.
The last step is the integral comparison $\sum_{j\ge J}j^{-2\alpha}\ge\int_J^\infty x^{-2\alpha}\dd x=J^{1-2\alpha}/(2\alpha-1)$, which converges because $\alpha>1/2$.
\end{proof}

Using the entire tail rather than a single coefficient is what produces the exponent $2\alpha-1$ in Corollary~\ref{cor:critical}.
A weaker argument that bounds the risk below by the single term at the crossing lag would give $J^{-2\alpha}$ instead of $J^{1-2\alpha}$ and hence the exponent $2\alpha$ in place of $2\alpha-1$, which is a strictly weaker necessary condition.

\subsection{Sharp inversion to the near-criticality law}
Lem.~\ref{lem:app-tail} expresses the risk through the crossing lag $J$, so it remains to express $J$ through $\delta=-\log\rho$.
The defining inequality $MK e^{-\delta(j-1)}\le\tfrac12cj^{-\alpha}$ holds as soon as $\delta(j-1)\ge\log(2MK/c)+\alpha\log j$, and solving for the smallest such integer gives
\begin{equation}
 J\ \le\ \left\lceil\frac{C_0}{\delta}\log\frac e\delta\right\rceil
 \label{eq:app-J}
\end{equation}
for small $\delta$, with $C_0$ depending only on $(\alpha,c,MK)$.
The logarithmic factor is unavoidable and comes from the $\alpha\log j$ term, that is from the target being algebraic rather than bounded.

\subsection{Proof of Corollary~\ref{cor:critical}}
\begin{proof}[Proof of Corollary~\ref{cor:critical}]
Insert \eqref{eq:app-J} into \eqref{eq:app-tail}.
Since $1-2\alpha<0$ and, for small $\delta$, the ceiling in \eqref{eq:app-J} is at most $2C_0\delta^{-1}\log(e/\delta)$, we get
\[
 \cE(\widehat m)\ \ge\ \frac{c^2}{4(2\alpha-1)}\Bigl(\frac{2C_0}{\delta}\log\frac e\delta\Bigr)^{1-2\alpha}
 \ =\ C\Bigl[\frac{\delta}{\log(e/\delta)}\Bigr]^{2\alpha-1} ,
\]
which is \eqref{eq:critical-floor}.
For \eqref{eq:critical-rate}, suppose $\cE(\widehat m)\le\tau$.
Then $[\delta/\log(e/\delta)]^{2\alpha-1}\le\tau/C$, so with $a:=(\tau/C)^{1/(2\alpha-1)}$ we have $\delta\le a\log(e/\delta)$.
Split into two cases.
If $\delta\le a$ the desired bound holds already, since $\log(e/\tau)\ge1$.
If $\delta>a$ then $\log(e/\delta)<\log(e/a)$, so $\delta\le a\log(e/a)$, and substituting $a=(\tau/C)^{1/(2\alpha-1)}$ gives $\log(e/a)\le C''\log(e/\tau)$ for small $\tau$.
In both cases $\delta\le C'\tau^{1/(2\alpha-1)}\log(e/\tau)$.
\end{proof}

\subsection{Why the condition is necessary but not sufficient}
Corollary~\ref{cor:critical} is a one-directional implication.
It says that a uniformly contractive recurrence achieving risk $\tau$ must have contraction margin at most of order $\tau^{1/(2\alpha-1)}\log(e/\tau)$, and it says nothing about whether a recurrence with that margin achieves that risk.

The gap is genuine and has two sources.
The envelope of Thm.~\ref{thm:first-chaos} is an upper bound on $|\beta_j|$, so a model can have a small margin and still fail to place its first-chaos mass where the target does, since nothing forces the sensitivities to have the right signs or the right profile.
Separately, even a model whose first-chaos memory matches the target exactly can carry higher-chaos mass, which \eqref{eq:app-chaos-lower} shows can only add to the risk.
The stability sweep of Sec.~\ref{sec:experiments} is consistent with this gap: the diagonal model keeps improving as $\rho$ approaches one while the gated model retains a residual floor. We use that floor only as an illustration and do not attribute it uniquely to first-chaos profile mismatch, since higher-chaos variance can contribute as well.

\subsection{Scope beyond uniform contraction}
Three restrictions bound what this section covers.

The white-noise input is deliberate: Thm.~\ref{thm:first-chaos} isolates memory generated internally by the recurrence. A recurrence driven directly by an autocorrelated observed process can inherit long memory through the input stream, so the same innovation-sensitivity envelope does not follow from this argument.

The contraction is uniform on the forward-invariant domain $\mathcal B$ in $(s,z)$, so a model whose Jacobian is contractive only on average along the stationary trajectory but not pointwise on such a domain is outside the class.
This is not a vacuous distinction, since selective architectures set their effective decay from the input and can be far from contractive on some inputs while remaining stable in aggregate.

The budgets in \eqref{eq:app-budgets} are constants, so a model whose derivative bounds grow with width evades the conclusion by a route that has nothing to do with memory.
The experiments in App.~\ref{app:experiments} enforce width-independent budgets for exactly this reason.

Finally the bound uses only the first chaos, so it is insensitive to how the predictor uses higher-order interactions of the input.
A target that is itself nonlinear in the input would need a different projection, and the argument as written applies to targets in the first chaos.

\section{Information Theory for Singular Forecasting}
\label{app:information}

\subsection{Conditional trajectory KL identity}
Let $P_{\theta,T}(\cdot\mid\mathcal F_0)$ denote the law of $X_1,\dots,X_T$ under a Gaussian causal model with memory $a(\theta)$ and innovation variance $\sigma^2$, conditioned on the infinite past.

\begin{proof}[Proof of Prop.~\ref{prop:kl-risk}]
Under $P_{\theta,T}(\cdot\mid\mathcal F_0)$ the conditional law of $X_t$ given $\mathcal F_{t-1}$ is $N(m_{t-1}(a(\theta)),\sigma^2)$, and likewise for $\vartheta$.
The chain rule for relative entropy over the filtration gives
\[
 \begin{aligned}
 &\KL\bigl(P_{\theta,T}(\cdot\mid\mathcal F_0)\,\|\,P_{\vartheta,T}(\cdot\mid\mathcal F_0)\bigr)\\
 &\qquad=\sum_{t=1}^{T}\E_\theta\Bigl[\KL\bigl(N(m_{t-1}(a(\theta)),\sigma^2)\,\|\,N(m_{t-1}(a(\vartheta)),\sigma^2)\bigr)\Bigm|\mathcal F_0\Bigr].
 \end{aligned}
\]
For two Gaussians with common variance the relative entropy is half the squared mean difference divided by $\sigma^2$, so each summand is $\E_\theta[(m_{t-1}(a(\theta))-m_{t-1}(a(\vartheta)))^2\mid\mathcal F_0]/(2\sigma^2)$.
Taking the outer expectation and using stationarity makes every term equal to $\cE_\theta(a(\vartheta))/(2\sigma^2)$ by \eqref{eq:app-risk-identity}, and summing over $t$ gives \eqref{eq:kl-risk}.
\end{proof}

No upper spectral bound is used, only finiteness of $\cE_\theta(a(\vartheta))$, so the identity holds on the singular spectra of App.~\ref{app:fractional} as well as on regular ones.
This is why the information statements of Sec.~\ref{sec:information} are available in exactly the regime where the norm equivalence of App.~\ref{app:prelim} fails.

\subsection{Stationary entropy-rate identity}
If in addition the stationary relative entropy rate
$\overline D(\theta\|\vartheta)=\lim_TT^{-1}\KL(P_\theta^{(T)}\|P_\vartheta^{(T)})$
exists and coincides with the infinite-past conditional rate, then dividing \eqref{eq:kl-risk} by $T$ gives
\begin{equation}
 \overline D(\theta\|\vartheta)=\frac1{2\sigma^2}\cE_\theta(a(\vartheta)) .
 \label{eq:entropy-risk}
\end{equation}
The two hypotheses are separate.
The conditional identity is exact for every finite $T$, while \eqref{eq:entropy-risk} additionally requires that conditioning on the infinite past does not change the rate, which is where the finite-section corrections of the rest of this appendix enter.

Two consequences are used in App.~\ref{app:fractional}.
For fractional noise, \eqref{eq:entropy-risk} together with Thm.~\ref{thm:d-misspec} gives $\overline D(d\|d+\Delta)=\tfrac12\{\Gamma(1+2\Delta)/\Gamma(1+\Delta)^2-1\}$, which is $\pi^2\Delta^2/12+O(|\Delta|^3)$.
If two models share a singular spectral factor and differ only in a regular one, that factor cancels in the likelihood ratio, so \eqref{eq:entropy-risk} applies to the ratio of the regular parts alone and a known common singularity contributes nothing to distinguishability.

\subsection{Conditional score and predictive curvature}
Suppose the predictor is $L_2$-differentiable at an interior $\theta_0$, meaning
$m_{t-1}(\theta_0+u)-m_{t-1}(\theta_0)=u^\top\dot m_{t-1}+r_{t-1}(u)$ with $\E r_{t-1}(u)^2=o(\|u\|^2)$.
Differentiating the conditional Gaussian log-likelihood gives the conditional score
\begin{equation}
 s_t(\theta_0)=\frac{\eps_t}{\sigma^2}\dot m_{t-1} ,
 \label{eq:cond-score}
\end{equation}
which is a square-integrable martingale difference because $\eps_t$ is conditionally centered and $\dot m_{t-1}$ is $\mathcal F_{t-1}$-measurable.
Its per-observation information matrix is $I^{\rm cond}_{\theta_0}=G_{\theta_0}/\sigma^2$ with $G_{\theta_0}=\E[\dot m_{t-1}\dot m_{t-1}^\top]$ the forecast curvature.

The same curvature appears on both sides of the theory, since $\cE_{\theta_0}(a(\theta_0+u))=u^\top G_{\theta_0}u+o(\|u\|^2)$ by \eqref{eq:app-risk-identity}.
Information and forecast risk are therefore not merely proportional in the aggregate identity \eqref{eq:entropy-risk} but share the same local quadratic form, which is what makes the constant in the next subsection universal.

\subsection{Local asymptotic minimax transfer}
Assume the stationary experiment is locally asymptotically normal at $\theta_0$ with central sequence asymptotically equal to the normalized sum of conditional scores, and that the martingale central limit theorem applies to that sum.
Then the local asymptotic minimax bound for estimating $\theta$, in the standard form for locally asymptotically normal experiments \citep[Ch.~8]{vandervaart1998asymptotic}, transfers to forecast risk through the local quadratic expansion, and for a $p$-dimensional smooth parameter the minimax forecast constant is $p\sigma^2/T$.

For the one-dimensional fractional exponent this gives a concrete number.
The forecast curvature is $\pi^2\sigma^2/6$ and the efficient asymptotic variance is $6/\pi^2$ \citep{lieberman2004fractional,hualde2011fractional}, so an efficient estimator satisfies $T\cE_d(a(\widehat d_T))\Rightarrow\sigma^2\chi^2_1$ and, given uniform integrability, $\E\cE_d(a(\widehat d_T))=\sigma^2/T+o(T^{-1})$.
The curvature and the variance cancel, which is why the answer does not depend on $d$ and why a long-memory exponent costs exactly one parameter's worth of forecast risk.

\subsection{Exact one-pole Fisher--Hartwig determinant}
By the determinant identity of App.~\ref{app:fractional}, $\log\det T_T(f)=\sum_{j<T}\log V_j$, so a finite-section information statement is a Toeplitz determinant asymptotic.
For fractional noise against white noise the symbol has one Fisher--Hartwig singularity of strength $d$, and the classical asymptotic \citep{deift2011fisherhartwig,ehrhardt1997onefh} gives $\log\det T_T=d^2\log T+O(1)$.

\begin{theorem}[Finite-section remainder for fractional versus white noise]
\label{thm:finite-kl}
Fix $0<d_0<1/2$.
For every $|d|\le d_0$ and $T\ge1$,
\begin{equation}
 \KL\bigl(P_d^{(T)}\|P_0^{(T)}\bigr)=T\,\overline D(d\|0)-\tfrac12\log\det T_T(\varphi_d),
 \label{eq:finite-kl-exact}
\end{equation}
with $\overline D(d\|0)=\tfrac12\{\Gamma(1-2d)/\Gamma(1-d)^2-1\}$, and there is $C_{d_0}<\infty$ with
\begin{equation}
 0\le T\,\overline D(d\|0)-\KL\bigl(P_d^{(T)}\|P_0^{(T)}\bigr)\le C_{d_0}d^2\log(eT)
 \label{eq:finite-kl-remainder}
\end{equation}
uniformly over that range.
\end{theorem}

\begin{proof}[Proof of the exact identity \eqref{eq:finite-kl-exact}]
The two laws are centered Gaussians with covariances $\sigma^2T_T(\varphi_d)$ and $\sigma^2I_T$, so
\[
 \KL\bigl(P_d^{(T)}\|P_0^{(T)}\bigr)
 =\tfrac12\Bigl\{\operatorname{tr}T_T(\varphi_d)-T-\log\det T_T(\varphi_d)\Bigr\}.
\]
Every diagonal entry of a Toeplitz matrix is the zeroth Fourier coefficient of its symbol, so $\operatorname{tr}T_T(\varphi_d)=T\widehat\varphi_d(0)$ and the first two terms equal
\[
 \frac T2\bigl\{\widehat\varphi_d(0)-1\bigr\}=T\,\overline D(d\|0)
\]
by the stationary rate identity \eqref{eq:entropy-risk}.
The remainder bound \eqref{eq:finite-kl-remainder} is proved in App.~\ref{app:information} below.
\end{proof}

The boundary is therefore logarithmic rather than of order $T$, which is the substantive statement.
A singular spectrum could in principle make the finite-section experiment differ from the stationary rate at leading order, and \eqref{eq:finite-kl-remainder} says it does not for a single fractional pole.

\subsection{Trigamma proof of the \texorpdfstring{$O(d^2\log T)$}{2} remainder}
We now prove the remainder bound \eqref{eq:finite-kl-remainder} of Thm.~\ref{thm:finite-kl}.
The uniformity in $d$ is what needs an argument, since a bound of the form $C_d\log T$ with $C_d$ blowing up as $d\to0$ would be useless for the local statement below.

Write the remainder as $\tfrac12\log\det T_T(\varphi_d)$ and use the determinant product together with the exact prediction variances \eqref{eq:VL-exact}, which gives
\[
 \log\det T_T(\varphi_d)=\sum_{j=0}^{T-1}\log\frac{V_j}{\sigma^2}
 =\sum_{j=0}^{T-1}\bigl[\log\Gamma(j+1)+\log\Gamma(j+1-2d)-2\log\Gamma(j+1-d)\bigr].
\]
Each summand is a second difference of $\log\Gamma$ in the parameter, so by Taylor expansion with integral remainder it equals $d^2\psi'(j+1-\xi_j)$ for some $\xi_j$ lying between $0$ and $2d$ (in either order), hence with $|\xi_j|\le2d_0$, where $\psi'$ is the trigamma function.
For $j\ge1$ the argument satisfies $j+1-\xi_j\ge j+1-2d_0\ge\tfrac12$, where $\psi'(x)\le C/x$, so those terms contribute at most $Cd^2\sum_{1\le j<T}(j+1-2d_0)^{-1}\le Cd^2\log(eT)$.
The term $j=0$ needs separate treatment, because its argument $1-\xi_0$ can be as small as $1-2d_0$ and $\psi'(x)\asymp x^{-2}$ near the origin, so the bound $\psi'(x)\le C/x$ is unavailable there.
Since $d\le d_0<1/2$ gives $1-\xi_0\ge1-2d_0>0$, monotonicity of $\psi'$ bounds that single term by $d^2\psi'(1-2d_0)$, a constant depending only on $d_0$, and adding the two contributions gives at most $C_{d_0}d^2\log(eT)$.
It is nonnegative because $\log\Gamma$ is convex, so every second difference is nonnegative.
The constant depends only on $d_0$ through the range of $\xi_j$, which gives the uniformity and covers negative $d$ as well.

\subsection{Direct local information at \texorpdfstring{$d=0$}{d=0}}
\begin{corollary}[Local finite-section information]
\label{cor:local-info}
For fixed $h\in\R$ let $d_T=h/\sqrt T$.
Then $\KL(P_{d_T}^{(T)}\|P_0^{(T)})=\pi^2h^2/12+o(1)$, so the finite stationary experiment recovers the Fisher information $I(0)=\pi^2/6$.
\end{corollary}

\begin{proof}
Insert $d_T=h/\sqrt T$ into \eqref{eq:finite-kl-exact}.
The leading term is $T\overline D(d_T\|0)=T\{\pi^2d_T^2/12+O(d_T^3)\}=\pi^2h^2/12+O(T^{-1/2})$, and the remainder is at most $C_{d_0}d_T^2\log(eT)=O(T^{-1}\log T)$ by \eqref{eq:finite-kl-remainder}.
\end{proof}

The local scale $d=h/\sqrt T$ is exactly the one on which the logarithmic remainder is negligible, so the finite stationary block has the same local relative-entropy constant as the stationary conditional experiment.
We do not infer local asymptotic normality from this, since that is a stochastic likelihood-ratio expansion and a limit for the relative entropy alone does not supply a central sequence.
What the calculation shows is that the Fisher--Hartwig boundary does not alter the local information constant on the $T^{-1/2}$ scale.

\subsection{Fixed-\texorpdfstring{$Q$}{Q} moving and merging fractional poles}
The single-pole statement of Thm.~\ref{thm:finite-kl} and Corollary~\ref{cor:local-info} extends to a fixed number of poles whose locations are allowed to depend on $T$ and to merge.

To avoid a collision with the exponential mixing density $w_d(t)$ of App.~\ref{app:fractional}, we write the Toeplitz spectral symbols as
\begin{equation}
 \varphi_d(e^{i\omega})=|1-e^{i\omega}|^{-2d},
 \qquad
 \varphi_{\mathbf d,\mathbf t}(e^{i\omega})=\prod_{q=1}^{Q}|e^{i\omega}-e^{it_q}|^{-2d_q}.
 \label{eq:spectral-symbol}
\end{equation}
For real-valued processes we restrict to reflection-symmetric configurations: whenever $t_q\notin\{0,\pi\}$ occurs with strength $d_q$, the conjugate location $2\pi-t_q$ occurs with the same strength. Then $\varphi_{\mathbf d,\mathbf t}$ is even and $T_T(\varphi_{\mathbf d,\mathbf t})$ is real symmetric. We set $\Sigma_{\mathbf d,\mathbf t,T}=\sigma^2T_T(\varphi_{\mathbf d,\mathbf t})$ and $P^{(T)}_{\mathbf d,\mathbf t}=N(0,\Sigma_{\mathbf d,\mathbf t,T})$, with $P_0^{(T)}=N(0,\sigma^2I_T)$ the white-noise law.

\begin{theorem}[Uniform finite-section expansion for $Q$ moving poles]
\label{thm:fixedQ}
Fix $Q$ and $0<d_\star<1/2$, and fix a strength vector $\mathbf d=(d_1,\dots,d_Q)$ with
\begin{equation}
 d_q\ge0,
 \qquad
 \sum_{q=1}^{Q}d_q\le d_\star<\tfrac12.
 \label{eq:fixedQ-margin}
\end{equation}
Uniformly over distinct reflection-symmetric locations $0\le t_1<\cdots<t_Q<2\pi$, as $T\to\infty$,
\begin{align}
 \KL\bigl(P^{(T)}_{\mathbf d,\mathbf t}\|P_0^{(T)}\bigr)
 &=T\,\overline D(\mathbf d,\mathbf t\|0)-\tfrac12\log\det T_T(\varphi_{\mathbf d,\mathbf t}),
 \label{eq:fixedQ-exact}\\
 \log\det T_T(\varphi_{\mathbf d,\mathbf t})
 &=\Bigl(\sum_{q}d_q^2\Bigr)\log T
 +2\sum_{q<\ell}d_qd_\ell\log\frac1{\sin(|t_q-t_\ell|/2)+T^{-1}}
 +O_{Q,\mathbf d}(1).
 \label{eq:fixedQ-fahs}
\end{align}
\end{theorem}

The second term is the interaction between poles, and the regularization by $T^{-1}$ inside the logarithm is what makes the expansion survive merging.
Two poles at separation much larger than $T^{-1}$ interact through $\log(1/\sin(|t_q-t_\ell|/2))$, while two poles at separation much smaller than $T^{-1}$ behave as a single pole of combined strength, and the crossover is at separation of order $T^{-1}$.
This is the same resolution scale that appears in the block prediction information, where the cumulative log-variance over a dyadic block of contexts tends to $4d^2\log2$ when the separation is below the scale and to $2d^2\log2$ when it is above.

\subsection{Uniform Fisher--Hartwig input and the role of the integrability constraint}
Thm.~\ref{thm:fixedQ} consumes a Toeplitz asymptotic that is uniform in the singularity locations rather than a pointwise one, since the pole locations move with $T$.
The input is \citet{fahs2021uniform}, together with the two-singularity merging analysis of \citet{claeys2015merging}.
It is worth stating that input precisely, because the scope of its uniformity is narrower than the phrase \emph{uniform Fisher--Hartwig} might suggest.

\begin{proof}[Proof of Thm.~\ref{thm:fixedQ}]
For the exact identity, $\Sigma_{\mathbf d,\mathbf t,T}=\sigma^2T_T(\varphi_{\mathbf d,\mathbf t})$ and $P_0^{(T)}=N(0,\sigma^2I_T)$, so the finite-dimensional Gaussian relative entropy between two centered laws gives
\[
 \KL\bigl(P^{(T)}_{\mathbf d,\mathbf t}\|P_0^{(T)}\bigr)
 =\tfrac12\Bigl\{\operatorname{tr}T_T(\varphi_{\mathbf d,\mathbf t})-T-\log\det T_T(\varphi_{\mathbf d,\mathbf t})\Bigr\}.
\]
Every diagonal entry of $T_T(\varphi_{\mathbf d,\mathbf t})$ equals $\widehat\varphi_{\mathbf d,\mathbf t}(0)=(2\pi)^{-1}\int_{-\pi}^{\pi}\varphi_{\mathbf d,\mathbf t}(e^{i\omega})\dd\omega$, so $\operatorname{tr}T_T=T\widehat\varphi_{\mathbf d,\mathbf t}(0)$ and the first two terms are $T\,\overline D(\mathbf d,\mathbf t\|0)$, which is \eqref{eq:fixedQ-exact}.

For the determinant, apply Theorem~1.1 and Remark~1.2 of \citet{fahs2021uniform} with smooth part $V\equiv0$, jump exponents $\beta_q=0$, and root exponents $\alpha_q=-d_q$.
Condition \eqref{eq:fixedQ-margin} implies the same bound for every merging cluster because all $d_q$ are nonnegative, so the sum of root exponents in each cluster is at least $-d_\star>-1/2$ and the negative-exponent extension applies uniformly over the locations.
Since $\alpha_q^2=d_q^2$ and $2\alpha_q\alpha_\ell=2d_qd_\ell$, the two signs cancel in both explicit terms and Fahs's formula is exactly \eqref{eq:fixedQ-fahs}.
\end{proof}

The mapping deserves a word, because the sign is easy to get wrong.
\citet{fahs2021uniform} considers non-negative symbols $f(z)=e^{V(z)}\prod_{j\le m}\omega_{\alpha_j,\beta_j}(z/z_j)$ with $\omega_{\alpha,\beta}(z)=(z/e^{\pi i})^{\beta}|z-1|^{2\alpha}$, where $V$ is real valued and analytic on a neighbourhood of the unit circle, $z_j=e^{it_j}$ with $0\le t_1<\cdots<t_m<2\pi$, and $\alpha_j\ge0$ with $\Re\beta_j=0$.
For a fixed number $m$ of such singularities the result gives
\[
 \log D_n(f)=nV_0+\sum_{j\le m}(\alpha_j^2-\beta_j^2)\log n
 +\sum_{j<k}2(\alpha_j\alpha_k-\beta_j\beta_k)\log\frac1{\sin|t_j-t_k|/2+n^{-1}}+O(1),
\]
with the error term uniform in the locations $t_1,\dots,t_m$.
The uniformity asserted there is in the locations.
The exponents are fixed parameters of the symbol, and we do not attribute to that result any uniformity jointly in the strengths.

Matching $|e^{i\omega}-e^{it_q}|^{-2d_q}$ against $|z/z_q-1|^{2\alpha_q}$ forces $2\alpha_q=-2d_q$, so a long-memory pole of strength $d_q$ is a Fisher--Hartwig root singularity of exponent
\begin{equation}
 \alpha_q=-d_q\le0,
 \qquad
 \beta_q=0 .
 \label{eq:fh-exponent}
\end{equation}
Theorem~1.1 of that paper is stated for non-negative root exponents, while Remark~1.2 extends the same proof to negative root exponents provided every combination of merging singularities has total exponent strictly greater than $-1/2$.
Under \eqref{eq:fh-exponent} that requirement reads $\sum_{q\in S}d_q<1/2$ for every merging cluster $S$, and the global bound \eqref{eq:fixedQ-margin} implies all of these inequalities because $d_q\ge0$.

The margin therefore plays two roles at once, and we do not attribute it to only one of them.
It is the integrability condition that keeps each pole, and each limit of a collapsing cluster, inside the stationary fractional class on which the predictive geometry of Sec.~\ref{sec:fractional} is defined.
It is simultaneously the applicability condition for the negative-exponent Fisher--Hartwig expansion.
What we use from \citet{fahs2021uniform} is the uniformity in the pole locations, with the strength vector held fixed, and we make no claim of uniformity over strengths.

\subsection{What is not covered for growing \texorpdfstring{$Q$}{Q} or vanishing pole strengths}
Two regimes are outside the theory as stated.

The number of poles is fixed.
The constant in \eqref{eq:fixedQ-fahs} depends on $Q$, and the available uniform Fisher--Hartwig input does not control that dependence, so a spectrum whose number of singularities grows with $T$ is not covered.
This matters because a seasonal fractional model with a period growing in $T$ is exactly such a spectrum, and its finite-section information is unknown.

The $O(1)$ remainder in \eqref{eq:fixedQ-fahs} is uniform over the moving locations, which is the uniformity the external input supplies, and it is stated for a fixed vector of strengths.
What we do not claim is any uniformity in the strengths, nor a strength-sensitive refinement of that remainder, for instance one quadratic in $\mathbf d$, for triangular arrays of vanishing strengths coupled to merging locations.
Such a refinement would be the natural analogue of the $O(d^2\log T)$ single-pole bound, and neither the trigamma argument nor the fixed-$Q$ expansion supplies it.
We record both gaps rather than absorbing them, because Sec.~\ref{sec:information} states the fixed-$Q$ result and should not be read as covering either.

\section{Multivariate Predictive Memory}
\label{app:multivariate}

The main text is stated for a scalar target.
This appendix records the multivariate extension referred to in Sec.~\ref{sec:discussion}.
It is a complete but secondary development, and none of the three resource laws of Table~\ref{tab:resources} depends on it.

Four symbols are reused here with a different meaning from the one they carry in the main text, and the reuse is local to this appendix.
Capital letters denote matrix-valued coefficient arrays rather than the scalar transfer functions of Sec.~\ref{sec:setup}, so $A^\star$ below is a matrix memory and not the transfer function of $a^\star$.
The letter $d$ is the ambient dimension of the observed vector, not the fractional memory strength of Sec.~\ref{sec:fractional}; the fractional model plays no role in this appendix.
The letter $B$ denotes a matrix array and not the backshift operator, and $\Delta$ denotes a coefficient error and not a pole-location offset.

\subsection{Matrix-valued predictive geometry}
Let $(X_t)$ be a centered stationary $\R^{d}$-valued process with matrix predictive memory $A^\star=(A_j^\star)_{j\ge1}$, so that $m_t^\star=\sum_{j\ge1}A_j^\star X_{t+1-j}$, and let $F_X(\omega)$ be the spectral density matrix.

\begin{theorem}[Multivariate excess forecast risk]
\label{thm:mv-geometry}
Assume $\|F_X\|_{\rm op}$ essentially bounded and $\sum_j\|\Delta_j\|_F^2<\infty$ with $\Delta_j=A_j-A_j^\star$.
Then the prediction-error filter converges in $L_2$ and
\begin{equation}
 \cE(A)=\frac1{2\pi}\int_{-\pi}^{\pi}\operatorname{tr}\bigl[\Delta(e^{-i\omega})F_X(\omega)\Delta(e^{-i\omega})^{*}\bigr]\dd\omega .
 \label{eq:mv-geometry}
\end{equation}
If in addition $0<m I_d\preceq F_X\preceq M I_d$ almost everywhere, then $m\sum_j\|\Delta_j\|_F^2\le\cE(A)\le M\sum_j\|\Delta_j\|_F^2$.
\end{theorem}

\begin{proof}
Put $U_t=\sum_{j\ge1}\Delta_jX_{t+1-j}$, which converges in $L_2$ because $\|F_X\|_{\rm op}$ is essentially bounded and $\sum_j\|\Delta_j\|_F^2<\infty$, by the same Cauchy argument as in App.~\ref{app:prelim} applied coordinatewise.
The limit is $\mathcal F_t$-measurable, so it is orthogonal to the innovation, and expanding the squared error gives $\cE(A)=\E\|U_t\|^2$.
The multivariate spectral representation turns this into
\[
 \E\|U_t\|^2=\frac1{2\pi}\int_{-\pi}^{\pi}\operatorname{tr}\bigl[\Delta(e^{-i\omega})F_X(\omega)\Delta(e^{-i\omega})^{*}\bigr]\dd\omega,
\]
which is \eqref{eq:mv-geometry}.
For the two-sided bound, $m I_d\preceq F_X\preceq M I_d$ gives
$m\operatorname{tr}(\Delta\Delta^{*})\le\operatorname{tr}(\Delta F_X\Delta^{*})\le M\operatorname{tr}(\Delta\Delta^{*})$ pointwise in $\omega$, and Parseval applied entrywise turns $\int\operatorname{tr}(\Delta\Delta^{*})$ into $\sum_j\|\Delta_j\|_F^2$.
\end{proof}
Everything downstream uses only \eqref{eq:mv-geometry} and the two-sided bound, so the scalar and matrix theories differ in bookkeeping rather than in mechanism up to this point.

\subsection{Shared row/column factor memory}
The architecturally interesting classes are factorized.
Write
\begin{equation}
 A_j=\sum_{q=1}^{r}k_{\theta_q}(j)B_q,
 \qquad B_q\in\R^{d\times d},
 \label{eq:mv-factorized}
\end{equation}
so that $r$ temporal laws are shared across all variable pairs and the cross-variable structure lives entirely in the matrices $B_q$.
If $\operatorname{rank}(B_q)\le s$ and we factor $B_q=O_qV_q$ with $O_q\in\R^{d\times s}$ and $V_q\in\R^{s\times d}$, then a causal attention head with temporal weights proportional to $k_{\theta_q}$, value map $V_q$, output map $O_q$, and the output multiplier that cancels the softmax normalizer produces exactly $\sum_{j\le L}k_{\theta_q}(j)B_qX_{t+1-j}$, and summing $r$ heads realizes \eqref{eq:mv-factorized}.
This is the multivariate analogue of the realization statement in App.~\ref{app:prelim}, and it is what makes $r$ the number of heads and $s$ the head width.

The most restrictive and most useful case shares the subspaces as well,
\begin{equation}
 B_q=UC_qV^\top,
 \qquad U,V\in\R^{d\times s},
 \label{eq:mv-shared}
\end{equation}
so that all heads read from one $s$-dimensional input subspace and write into one $s$-dimensional output subspace, with the cores $C_q\in\R^{s\times s}$ carrying the head-specific cross-variable mixing.

\subsection{Fixed temporal heads}
When the temporal laws are known and only the matrices are estimated, the problem is linear.
Stacking the $r$ filtered regressors into $Z_t\in\R^{rd}$ and writing $\Theta^\star$ for the true coefficient matrix, ordinary least squares satisfies
\begin{equation}
 \sqrt T\operatorname{vec}(\widehat\Theta_T-\Theta^\star)\Rightarrow N\bigl(0,\sigma^2(\Gamma_Z^{-1}\otimes I_d)\bigr)
 \label{eq:mv-ols}
\end{equation}
under stationarity, ergodicity, $\Gamma_Z\succ0$, and the usual martingale Lindeberg condition.
The last three hold for a stable Gaussian vector autoregression with fixed predictively nondegenerate filters, the nondegeneracy being needed because repeated or linearly dependent temporal filters make the regressor Gram matrix singular.
Combining \eqref{eq:mv-ols} with the local expansion of \eqref{eq:mv-geometry} gives excess forecast risk of order $rd^2\sigma^2/T$, which is the multivariate counterpart of the local parametric law of App.~\ref{app:information}.

More generally, for a $p$-parameter family that is twice differentiable with an $o(\|h\|^2)$ remainder and an estimator obeying $\sqrt T(\widehat\eta_T-\eta_\star)\Rightarrow N(0,\sigma^2G_\star^{-1})$, one gets $T\cE(A(\widehat\eta_T))\Rightarrow\sigma^2\chi^2_p$ and, under uniform integrability of that sequence, $\E\cE=p\sigma^2/T+o(T^{-1})$.
The whole content of the multivariate theory is therefore which $p$ applies to which class.

\subsection{Growing core and shared-subspace lower bounds}
The dimension counts are not merely upper bookkeeping.
For the shared-factor class with transfer-function stability margin $\rho$, two separate packings give matching lower scales.

Write the minimax excess forecast risk over the shared-factor class as $\inf_{\widehat A}\sup_{A^\star}\E\,\cE(\widehat A)$, the infimum being over estimators measurable with respect to the observed trajectory.
\begin{lemma}[Multivariate trajectory KL under transfer stability]
\label{lem:mv-kl}
Fix $0<\rho<1$ and let $A,B$ generate stationary Gaussian vector autoregressions with innovations $N(0,\sigma^2I_d)$ whose transfer functions satisfy $\operatorname{ess\,sup}_\omega\|\mathsf A(\omega)\|_{\rm op}\le\rho$ and likewise for $B$.
Then the spectral densities obey $\sigma^2(1+\rho)^{-2}I_d\preceq F_A,F_B\preceq\sigma^2(1-\rho)^{-2}I_d$ almost everywhere, and there is $C_\rho<\infty$, independent of $T$, $d$ and $\sigma$, with
\begin{equation}
 \KL\bigl(P_A^{(T)}\|P_B^{(T)}\bigr)\ \le\ C_\rho\,T\sum_{j\ge1}\|A_j-B_j\|_F^2 .
 \label{eq:mv-kl}
\end{equation}
The bound carries no factor of $\sigma$, which is forced: rescaling $X_t\mapsto X_t/\sigma$ leaves the coefficients unchanged and pushes both trajectory laws forward by one common invertible map, so their relative entropy is unchanged while any $\sigma$-dependent right-hand side would not be.
\end{lemma}

\begin{proof}
Write $D_A(\omega)=I-\mathsf A(\omega)$.
Stability gives $\|D_A\|_{\rm op}\le1+\rho$ and $s_{\min}(D_A)\ge1-\rho$, hence $\|D_A^{-1}\|_{\rm op}\le(1-\rho)^{-1}$, and the same for $D_B$.
Since $F_A=\sigma^2D_A^{-1}D_A^{-*}$, these singular-value bounds give the stated two-sided spectral bounds.

Let $\Sigma_{A,T}$ and $\Sigma_{B,T}$ be the $dT\times dT$ block-Toeplitz covariance matrices of the two length-$T$ trajectories.
Matrix Parseval and the spectral bounds give $m_\rho I\preceq\Sigma_{A,T},\Sigma_{B,T}\preceq M_\rho I$ with $m_\rho=\sigma^2(1+\rho)^{-2}$ and $M_\rho=\sigma^2(1-\rho)^{-2}$.
Put $R=\Sigma_{B,T}^{-1/2}\Sigma_{A,T}\Sigma_{B,T}^{-1/2}$, whose eigenvalues lie in the fixed compact interval $[m_\rho/M_\rho,\,M_\rho/m_\rho]$, an interval depending on $\rho$ alone.
On such an interval the centered Gaussian relative entropy $\tfrac12\sum_i(\lambda_i-1-\log\lambda_i)$ is bounded by a constant multiple of $\sum_i(\lambda_i-1)^2=\|R-I\|_F^2$, so
\[
 \KL\bigl(P_A^{(T)}\|P_B^{(T)}\bigr)\ \le\ C_\rho\,\bigl\|\Sigma_{B,T}^{-1/2}(\Sigma_{A,T}-\Sigma_{B,T})\Sigma_{B,T}^{-1/2}\bigr\|_F^2
 \ \le\ \frac{C_\rho}{m_\rho^2}\,\|\Sigma_{A,T}-\Sigma_{B,T}\|_F^2 .
\]
This is where the two powers of $\sigma$ cancel.
The spectral-density difference satisfies $\|F_A-F_B\|_F\le C_\rho\sigma^2\|\mathsf A-\mathsf B\|_F$ pointwise, by writing $F_A-F_B=\sigma^2\{D_A^{-1}(D_B-D_A)D_B^{-1}D_B^{-*}+D_A^{-1}D_A^{-*}(D_B-D_A)^{*}D_B^{-*}\}$ and using the uniform inverse bounds, so Toeplitz Frobenius comparison and Parseval give $\|\Sigma_{A,T}-\Sigma_{B,T}\|_F^2\le C_\rho T\sigma^4\sum_j\|A_j-B_j\|_F^2$.
Dividing by $m_\rho^2\asymp\sigma^4$ cancels the $\sigma^4$ exactly and leaves \eqref{eq:mv-kl}.
\end{proof}

\begin{lemma}[Flat core packing]
\label{lem:flat-core}
Fix $r,s\ge1$ and let $p_0=rs^2$.
There is a subset $\mathcal V\subset\{-1,1\}^{p_0}$, reshaped into sign cores $(S_1^{(v)},\dots,S_r^{(v)})$ with $S_q^{(v)}\in\{-1,1\}^{s\times s}$, such that
\begin{equation}
 \log|\mathcal V|\ \ge\ c\,p_0,
 \qquad
 d_H(u,v)\ \ge\ p_0/4
 \quad (u\ne v),
 \label{eq:flat-code}
\end{equation}
and every $v\in\mathcal V$ satisfies the flatness bound
\begin{equation}
 \sup_{\omega}\Bigl\|\sum_{q\le r}S_q^{(v)}e^{-iq\omega}\Bigr\|_{\rm op}
 \ \le\ C\sqrt{r\{s+\log(ers)\}} .
 \label{eq:flat-transfer}
\end{equation}
\end{lemma}

\begin{proof}
Draw all $p_0$ signs independently and uniformly, and write $P(\omega)=\sum_{q\le r}S_qe^{-iq\omega}$.
Fix unit vectors $x,y\in\R^{s}$.
The real and imaginary parts of $x^\top P(\omega)y$ are sums of independent centered terms with squared coefficient sums at most $r$, so $\cosh t\le e^{t^2/2}$ and Chernoff give $\Pr\{|x^\top P(\omega)y|>u\}\le4e^{-u^2/(8r)}$.
A $1/4$-net of the sphere in $\R^{s}$ has at most $9^{s}$ points and converts the bilinear bound into an operator-norm bound at the cost of a factor $2$ and a union over $81^{s}$ pairs.

Uniformity in $\omega$ needs the derivative, and the mesh must be chosen against it rather than guessed.
Since $\|\tfrac{d}{d\omega}P(\omega)\|_{\rm op}\le\sum_{q\le r}q\|S_q\|_{\rm op}\le sr^2$, a uniform grid of mesh $h=u/(2sr^2)$ has at most $1+Csr^2/u$ points, and if $\sup_\omega\|P\|_{\rm op}>2u$ then some grid value exceeds $u$.
A union bound over the grid and the net gives
\[
 \Pr\Bigl\{\sup_\omega\|P(\omega)\|_{\rm op}>2u\Bigr\}
 \le C\Bigl(1+\frac{sr^2}{u}\Bigr)81^{s}e^{-u^2/(8r)} .
\]
Taking $u=C_0\sqrt{r\{s+\log(ers)\}}$ with $C_0$ large makes the right side at most $1/4$.

Hence at least three quarters of the $2^{p_0}$ sign vectors satisfy \eqref{eq:flat-transfer}.
The packing must now be built inside that good subset, since a code chosen first need not be flat.
Apply greedy Hamming packing within the flat subset, repeatedly selecting a flat vector and deleting all flat vectors within Hamming radius $p_0/4$.
Each deletion removes at most $\sum_{i\le p_0/4}\binom{p_0}{i}\le2^{H(1/4)p_0}$ vectors with $H(1/4)<1$, while the flat subset has at least $\tfrac34 2^{p_0}$, so exponentially many selections remain and \eqref{eq:flat-code} holds.
\end{proof}

The two lemmas calibrate against each other, which is what turns the packing into a rate.
Give each core amplitude $\epsilon$ with
\begin{equation}
 \epsilon^2=\frac{c_\epsilon}{T},
 \label{eq:core-amplitude}
\end{equation}
and set $A_q^{(v)}=\epsilon\,U_0S_q^{(v)}V_0^\top$ for fixed orthonormal $U_0,V_0$.
The flatness bound \eqref{eq:flat-transfer} makes the transfer function of every packing point at most $C\epsilon\sqrt{r\{s+\log(ers)\}}$ in operator norm, a quantity free of $\sigma$, so choosing $c_\epsilon$ small keeps the whole packing inside the stability class once
\begin{equation}
 T\ \ge\ C\,r\{s+\log(ers)\} .
 \label{eq:core-sample}
\end{equation}
The scale chain is worth following, since $\sigma$ enters in exactly one place.
The Hamming separation in \eqref{eq:flat-code} gives coefficient separation $\sum_q\|A_q^{(u)}-A_q^{(v)}\|_F^2\ge4\epsilon^2p_0/4=\epsilon^2rs^2\asymp rs^2/T$, which carries no $\sigma$.
The lower spectral bound $F_X\succeq\sigma^2(1+\rho)^{-2}I$ in Thm.~\ref{thm:mv-geometry} is what converts that into predictive separation of order $\sigma^2rs^2/T$, so the innovation scale enters through the geometry rather than through the construction.
On the other side, Lem.~\ref{lem:mv-kl} bounds the relative entropy between any packing point and the centre by $C_\rho T\epsilon^2rs^2=C_\rho c_\epsilon rs^2$, which is at most a small multiple of $\log|\mathcal V|\ge crs^2$ once $c_\epsilon$ is small.
Fano's inequality therefore gives
\begin{equation}
 \inf_{\widehat A}\sup_{A^\star}\E\,\cE(\widehat A)\ \ge\ c\sigma^2\frac{rs^2}{T}.
 \label{eq:mv-core-lower}
\end{equation}

For the shared subspace we use a local perturbation rather than an external Grassmann packing, which keeps the construction self-contained and makes the calibration explicit.

\begin{theorem}[Shared-subspace lower bound]
\label{thm:mv-subspace}
Assume $1\le s<d$, with the supremum ranging over the stable shared-subspace class of App.~\ref{app:multivariate} at transfer stability constant $\rho$.
There are $c,C>0$ depending only on $\rho$ such that if $T\ge Cs(d-s)$ then
\begin{equation}
 \inf_{\widehat A}\sup_{A^\star}\E\,\cE(\widehat A)\ \ge\ c\sigma^2\frac{s(d-s)}{T}.
 \label{eq:mv-subspace-lower}
\end{equation}
\end{theorem}

\begin{proof}
Fix orthonormal $U_0\in\R^{d\times s}$, $U_\perp\in\R^{d\times(d-s)}$ with $U_0^\top U_\perp=0$, and $V_0\in\R^{d\times s}$, and put $p_1=s(d-s)$.
For $X\in\R^{(d-s)\times s}$ set $U(X)=U_0+U_\perp X$, which has full column rank since $U(X)^\top U(X)=I_s+X^\top X\succ0$.
Fix $0<\gamma<\rho/4$ and use the one-active-lag subclass $A_1(X)=\gamma U(X)V_0^\top$ with $A_j(X)=0$ for $j\ge2$.

Let $\mathcal V_{p_1}\subset\{-1,1\}^{p_1}$ be a Gilbert--Varshamov packing with $\log|\mathcal V_{p_1}|\ge cp_1$ and pairwise Hamming distance at least $p_1/4$, reshape $v$ into $S^{(v)}\in\{-1,1\}^{(d-s)\times s}$, and set $X^{(v)}=\delta S^{(v)}$ with $\delta^2=c_\delta/T$.
Under $T\ge Cs(d-s)$ we have $\|X^{(v)}\|_{\rm op}\le\|X^{(v)}\|_F=\delta\sqrt{p_1}$ uniformly small, so every member stays inside the stability class and has one active lag, hence transfer norm at most $2\gamma<\rho$.

Because $U_\perp$ has orthonormal columns and $V_0$ is fixed,
\[
 \bigl\|A_1(X^{(u)})-A_1(X^{(v)})\bigr\|_F^2
 =\gamma^2\bigl\|U_\perp(X^{(u)}-X^{(v)})V_0^\top\bigr\|_F^2
 =\gamma^2\delta^2\cdot4\,d_H(u,v)
 \ \ge\ \gamma^2\delta^2p_1 ,
\]
which is of order $s(d-s)/T$, and Thm.~\ref{thm:mv-geometry} converts it into predictive separation of the same order.
Lem.~\ref{lem:mv-kl} gives $\KL\le C_\rho T\gamma^2\delta^2p_1=C_\rho\gamma^2c_\delta p_1$, at most a small multiple of $\log|\mathcal V_{p_1}|\ge cp_1$ once $c_\delta$ is small.
Fano's inequality gives \eqref{eq:mv-subspace-lower}.
\end{proof}

The same local construction with the roles of the left and right subspaces interchanged, perturbing $V$ while holding $U$ fixed, gives the same lower bound \eqref{eq:mv-subspace-lower} for the input subspace.
Thus the core, output-subspace, and input-subspace packings give three valid lower bounds of respective orders $rs^2/T$, $s(d-s)/T$, and $s(d-s)/T$.
Since a maximum dominates a third of their sum, under
\begin{equation}
 T\ \gtrsim\ r\{s+\log(ers)\}+s(d-s)
 \label{eq:mv-sample-condition}
\end{equation}
we obtain
\begin{equation}
 \inf_{\widehat A}\sup_{A^\star}\E\,\cE(\widehat A)\ \ge\ c\sigma^2\frac{2s(d-s)+rs^2}{T}.
 \label{eq:mv-shared-lower}
\end{equation}
No independent packing over the temporal parameter space is used here, so \eqref{eq:mv-shared-lower} does not claim an additional $rp_\theta$ lower-bound term.
That term enters the regular-chart parameter count below; when $p_\theta$ is fixed and $s\ge1$, it is bounded by a fixed multiple of $rs^2$ and therefore does not change the order comparison.
Condition \eqref{eq:mv-sample-condition} is the sharper of the thresholds the three packings need, and it is what the flat construction of Lem.~\ref{lem:flat-core} buys relative to a naive bound of order $r^2s^2$.
These are nonasymptotic same-realization statements in which $d$, $r$, and $s$ are allowed to grow.

\subsection{Learned smooth temporal heads on a regular chart}
When the temporal laws are learned rather than fixed, the matching upper bound is available on charts that are uniformly well conditioned.

Write $\mathcal V_\Phi$ for the normalized two-secant set of the chart and $\mathcal V_\Phi^0:=\mathcal V_\Phi\cup\{0\}$.
The chaining proof also needs normalized differences of such directions, so define the normalized increment base
\[
 \mathcal S_\Phi
 :=\left\{\frac{u-v}{\sqrt{\cE(u-v)}}:\ u,v\in\mathcal V_\Phi^0,\ u\ne v\right\},
 \qquad
 \mathcal S_\Phi^0:=\mathcal S_\Phi\cup\{0\}.
\]
In particular $\mathcal V_\Phi\subset\mathcal S_\Phi$ by taking one endpoint to be the adjoined origin.
Call $(\Phi,\Theta)$ a uniformly regular predictive chart if the normalized increment base has entropy
\[
 \log N(\mathcal S_\Phi,\epsilon)\le C_0p\log(C_0/\epsilon)
\]
in the predictive metric and the transfer envelope holds on one further difference closure,
\begin{equation}
 \sum_{j\le L}\|W_j\|_{\rm op}\le\frac{\Lambda}{\sigma}\sqrt{\cE(W)}
 \qquad\text{for every nonzero }W\in\operatorname{cone}(\mathcal S_\Phi^0-\mathcal S_\Phi^0).
 \label{eq:mv-increment-envelope}
\end{equation}
The entropy assumption is therefore imposed directly on the normalized base consumed by the quadratic concentration argument; it is not inferred by normalizing an unnormalized difference set.
The extra difference closure in \eqref{eq:mv-increment-envelope} supplies the operator-metric Lipschitz bound between two points of $\mathcal S_\Phi$.
Both conditions are imposed on the image geometry rather than on raw coordinates, so they are invariant under reparameterization and separate genuine statistical difficulty from a poor choice of head coordinates.

The envelope constant records any context-length dependence rather than hiding it in the theorem.
Indeed, under the transfer stability of Lem.~\ref{lem:mv-kl}, the spectral lower bound and Cauchy--Schwarz give, for every length-$L$ perturbation,
\[
 \sum_{j\le L}\|W_j\|_{\rm op}
 \le \sqrt L\left(\sum_{j\le L}\|W_j\|_F^2\right)^{1/2}
 \le \frac{(1+\rho)\sqrt L}{\sigma}\sqrt{\cE(W)},
\]
so one may always take $\Lambda\le(1+\rho)\sqrt L$ on a finite context.
Thus the generic sample condition below can inflate to order $L(p+x)$.
For structured temporal spans whose normalized secant differences have a uniformly bounded $\ell_1/\ell_2$ ratio, including uniformly exponentially localized spans, $\Lambda=O(1)$; for long-memory or power-law heads it may grow with $L$, and the theorem retains that dependence explicitly.

Write $d_t(D)=\sum_{j\le L}D_jX_{t+1-j}$ for $t=L,\ldots,T-1$ and put $n=T-L$.
Then $d_t(D)$ is $\mathcal F_t$-measurable and the innovation paired with it in the empirical one-step loss is $\eps_{t+1}$.

\begin{lemma}[Uniform same-realization concentration on a regular chart]
\label{lem:mv-concentration}
Assume the stable Gaussian vector autoregression of Lem.~\ref{lem:mv-kl} together with the entropy and envelope conditions defining a uniformly regular chart.
There are $C,c$ depending only on $\rho$ and $C_0$ such that, for every $x\ge1$, if $n\ge C\Lambda^2(p+x)$, then with probability at least $1-Ce^{-cx}$ the following hold simultaneously for every $D\in\operatorname{cone}(\mathcal S_\Phi)$ and every $A$ in the chart:
\begin{equation}
 \tfrac12n\,\cE(D)\ \le\ \sum_{t=L}^{T-1}\|d_t(D)\|^2\ \le\ \tfrac32n\,\cE(D),
 \label{eq:mv-curvature}
\end{equation}
\begin{equation}
 2\Bigl|\sum_{t=L}^{T-1}\langle\eps_{t+1},d_t(A-A^\star)\rangle\Bigr|
 \ \le\ \tfrac14n\,\cE(A)+C\sigma^2(p+x).
 \label{eq:mv-crossterm}
\end{equation}
\end{lemma}

\begin{proof}
No truncation of the infinite innovation sequence is needed for the curvature term, because the observed block is already finite dimensional and can be whitened directly.
Stack the observations into $X_{\rm block}\in\R^{dT}$ and write $X_{\rm block}=\Sigma_{A^\star,T}^{1/2}\xi$ with $\xi\sim N(0,I_{dT})$, which is legitimate since $\Sigma_{A^\star,T}\succ0$ by Lem.~\ref{lem:mv-kl}.
For a perturbation $D$ let $\mathcal T_D$ be the linear map sending the block to the stacked residual differences, so that $\mathbf d(D)=\mathcal T_D\Sigma_{A^\star,T}^{1/2}\xi$ and
\[
 \sum_{t=L}^{T-1}\|d_t(D)\|^2=\bigl\|B_D\,\xi\bigr\|^2,
 \qquad
 B_D:=\mathcal T_D\Sigma_{A^\star,T}^{1/2} .
\]
This is a finite-dimensional Gaussian chaos family.
For $D\in\mathcal S_\Phi$, stationarity gives $\|B_D\|_F^2=n$, while for $D,D'\in\mathcal S_\Phi$ the difference $D-D'$ lies in $\mathcal S_\Phi^0-\mathcal S_\Phi^0$ and \eqref{eq:mv-increment-envelope} gives
\[
 \|B_D-B_{D'}\|_{\rm op}
 \le (1-\rho)^{-1}\Lambda\sqrt{\cE(D-D')}.
\]
The entropy hypothesis is already stated on $\mathcal S_\Phi$, so Dudley's integral gives $\gamma_2\lesssim(1-\rho)^{-1}\Lambda\sqrt p$ and the operator diameter is $O_\rho(\Lambda)$.
Substituting these quantities and Frobenius radius $\sqrt n$ into the suprema-of-chaos bound of \citet{krahmer2014suprema} shows that the relative deviation is at most $1/2$ on $\mathcal S_\Phi$ once $n\ge C\Lambda^2(p+x)$, with failure probability at most $Ce^{-cx}$.
Homogeneity then gives \eqref{eq:mv-curvature} on $\operatorname{cone}(\mathcal S_\Phi)$.

For \eqref{eq:mv-crossterm}, the corrected one-step indexing makes the summands $\langle\eps_{t+1},d_t(A-A^\star)\rangle$ a martingale difference sequence: $d_t$ is $\mathcal F_t$-measurable and $\E[\eps_{t+1}\mid\mathcal F_t]=0$ with conditional covariance $\sigma^2I_d$.
The self-normalized inequality of \citet{delapena2009selfnorm} bounds a fixed direction by its predictable quadratic variation $\sigma^2\sum_t\|d_t\|^2$.
To make this uniform, chain only over the normalized two-secant set $\mathcal V_\Phi\subset\mathcal S_\Phi$.
For $s\ge0$ let $\mathcal N_s$ be a $2^{-s}$-net of $\mathcal V_\Phi$ in the predictive metric; since $\mathcal V_\Phi\subset\mathcal S_\Phi$, the entropy bound gives $\log|\mathcal N_s|\le Cp(s+1)$.
For a normalized direction $W$, choose nearest points $\pi_s(W)\in\mathcal N_s$.
Each link $\pi_s(W)-\pi_{s-1}(W)$ is a difference of two members of $\mathcal V_\Phi^0$; after normalization it belongs to $\mathcal S_\Phi$ by definition, so the curvature event \eqref{eq:mv-curvature} applies to every link by homogeneity and bounds its predictable quadratic variation by $C\sigma^2n4^{-s}$.
Assigning threshold $C\sigma\sqrt n\,2^{-s}\sqrt{p(s+1)+x+s}$ at level $s$, the self-normalized inequality followed by a union bound over the at most $|\mathcal N_s||\mathcal N_{s-1}|$ links leaves failure probability $Ce^{-cx-cs}$.
Summing over $s$ and using $\sum_s2^{-s}\sqrt{s+1}<\infty$ gives
\[
 \sup_{W\in\mathcal V_\Phi}\left|\sum_{t=L}^{T-1}\langle\eps_{t+1},d_t(W)\rangle\right|
 \ \lesssim\ \sigma\sqrt{n(p+x)}.
\]
Rescaling $W=(A-A^\star)/\sqrt{\cE(A)}$ gives a uniform bound of order $\sqrt{\sigma^2n\,\cE(A)(p+x)}$, and $2\sqrt{PQ}\le\tfrac14P+4Q$ with $P=n\cE(A)$ yields \eqref{eq:mv-crossterm}.
\end{proof}

\begin{theorem}[Nonasymptotic ERM on a regular chart]
\label{thm:mv-smooth}
Under stable Gaussian vector autoregression and true transfer stability, let $\widehat\eta$ minimize the empirical criterion over a uniformly regular chart containing $A^\star$, put $\widehat A=A(\widehat\eta)$, and put $n=T-L$.
There are $c,C$ depending only on the stability margin and $C_0$ such that, for every $x\ge1$, if $n\ge C\Lambda^2(p+x)$ then, with probability at least $1-Ce^{-cx}$,
\begin{equation}
 \cE(\widehat A)\le C\sigma^2\frac{p+x}{n}.
 \label{eq:mv-smooth-upper}
\end{equation}
\end{theorem}

\begin{proof}
The proof now uses exactly the two events of Lem.~\ref{lem:mv-concentration}.
Because $A^\star$ and $\widehat A$ lie in the chart, the normalized direction $(\widehat A-A^\star)/\sqrt{\cE(\widehat A)}$ belongs to $\mathcal V_\Phi\subset\mathcal S_\Phi$, so \eqref{eq:mv-curvature} applies to the empirical quadratic term.
The basic inequality for the empirical minimizer, with one-step residuals indexed by $t=L,\ldots,T-1$, has cross term $2\sum_t\langle\eps_{t+1},d_t(\widehat A-A^\star)\rangle$, which is controlled by \eqref{eq:mv-crossterm} on the same event.
Absorbing the resulting quarter of $n\cE(\widehat A)$ into the curvature lower bound gives \eqref{eq:mv-smooth-upper}.
\end{proof}

The structural reason the strengthened regularity assumptions remain comparable to the original shared-factor class is a constant-factor rank count.
A normalized two-secant direction in $\mathcal V_\Phi$ has Tucker multilinear rank at most $(2s,2s,2r)$ in the two spatial modes and the temporal mode.
A normalized increment in $\mathcal S_\Phi$ has rank at most $(4s,4s,4r)$, while the further differences consumed by \eqref{eq:mv-increment-envelope} have rank at most $(8s,8s,8r)$.
Thus the entropy condition is imposed on the normalized four-secant base actually used by the proof and the transfer envelope on its eight-secant differences, rather than on an uncontrolled ambient tensor class.

The regular-chart conditions can be verified for a concrete sequence-model class rather than left only as abstract hypotheses.
We record the verification for bounded-head diagonal exponential memories, which are the temporal laws of a diagonal linear state-space model.

\begin{lemma}[Stable exponential sequences have a uniform $\ell_1/\ell_2$ bound]
\label{lem:mv-exp-l1l2}
Fix an integer $m\ge1$ and $0<\rho_{\rm h}<1$.
There is $C_{m,\rho_{\rm h}}<\infty$ such that, for every Hilbert space $\mathsf H$, every
\[
 x_j=\sum_{q=1}^{m}v_q\lambda_q^{j-1},
 \qquad v_q\in\mathsf H,
 \qquad |\lambda_q|\le\rho_{\rm h},
\]
and every $L\ge1$,
\begin{equation}
 \sum_{j=1}^{L}\|x_j\|_{\mathsf H}
 \le C_{m,\rho_{\rm h}}
 \left(\sum_{j=1}^{L}\|x_j\|_{\mathsf H}^2\right)^{1/2}.
 \label{eq:mv-exp-l1l2}
\end{equation}
The constant is independent of $L$, of $\mathsf H$, and of the coefficients $v_q$.
\end{lemma}

\begin{proof}
If $L\le m$, Cauchy--Schwarz gives the claim after enlarging the constant by $\sqrt m$.
Assume $L>m$ and let
$q(z)=\prod_{q=1}^{m}(z-\lambda_q)=z^m+a_1z^{m-1}+\cdots+a_m$.
Then $(x_j)$ obeys the order-$m$ recurrence
$x_{j+m}+a_1x_{j+m-1}+\cdots+a_mx_j=0$.
Let $C(a)$ be its companion matrix acting on $\mathsf H^m$.
The coefficient vectors $a$ obtained from roots in the closed disk $|z|\le\rho_{\rm h}$ form a compact set, and every corresponding companion matrix has spectral radius at most $\rho_{\rm h}$.
Fix $\widetilde\rho\in(\rho_{\rm h},1)$.
Compactness and the resolvent formula on $|z|=\widetilde\rho$ give
\[
 \sup_a\|C(a)^n\|_{\rm op}\le C_{m,\rho_{\rm h},\widetilde\rho}\widetilde\rho^n
 \qquad(n\ge0).
\]
Hence, with $Y_1=(x_m,\ldots,x_1)$,
\[
 \|x_{m+n}\|_{\mathsf H}
 \le C\widetilde\rho^n\|Y_1\|_{\mathsf H^m}
 \le C\widetilde\rho^n
 \left(\sum_{j=1}^{m}\|x_j\|_{\mathsf H}^2\right)^{1/2}.
\]
Summing the geometric tail and adding the first $m$ terms proves \eqref{eq:mv-exp-l1l2}.
\end{proof}

Fix now a head cap $\bar r\ge1$ and $0<\rho_{\rm h}<1$.
For $r\le\bar r$, let $\mathcal A_{\exp}(r,s,\rho_{\rm h})$ denote the length-$L$ shared-factor arrays
\begin{equation}
 A_j=U\left(\sum_{q=1}^{r}C_q\lambda_q^{j-1}\right)V^\top,
 \qquad |\lambda_q|\le\rho_{\rm h},
 \qquad 1\le j\le L,
 \label{eq:mv-exp-class}
\end{equation}
with $U,V\in\R^{d\times s}$, intersected with the transfer-stable class of Lem.~\ref{lem:mv-kl}.
No separation between the learned rates is imposed in the statement below.

\begin{proposition}[Bounded-head exponential charts are uniformly regular]
\label{prop:mv-exp-regular}
Fix $\bar r$, $\rho_{\rm h}<1$, and the transfer-stability margin $\rho<1$.
For every $r\le\bar r$, every predictive chart contained in $\mathcal A_{\exp}(r,s,\rho_{\rm h})$ satisfies the entropy and envelope conditions above with
\begin{equation}
 p_{\exp}=2s(d-s)+rs^2+r,
 \qquad
 \log N(\mathcal S_\Phi,\epsilon)
 \le C p_{\exp}\log(C/\epsilon),
 \qquad
 \Lambda\le C,
 \label{eq:mv-exp-regular}
\end{equation}
where $C$ depends only on $\bar r$, $\rho_{\rm h}$, and $\rho$, and is independent of $L,T,d,s$, and $\sigma$.
\end{proposition}

\begin{proof}
We verify the two conditions separately.
Every element of $\mathcal V_\Phi$ is a normalized difference of two model points, and every element of $\mathcal S_\Phi$ is a normalized difference of two such directions.
Consequently each $D\in\mathcal S_\Phi$ is a linear combination of at most four members of \eqref{eq:mv-exp-class}.
Its row and column spaces therefore have dimensions at most $4s$, while each scalar core coordinate is a sum of at most $4r$ real exponentials with rates in $[-\rho_{\rm h},\rho_{\rm h}]$.
After padding by zero roots if necessary, every such coordinate obeys an order-$4r$ recurrence whose characteristic roots lie in that interval.
Thus one may write
\[
 D_j=\widetilde U G_j\widetilde V^\top,
 \qquad
 \widetilde U,\widetilde V\in\R^{d\times k},
 \qquad
 k\le\min(4s,d),
\]
where $(G_j)$ is determined by an order-$4r$ stable recurrence and its first $4r$ matrix values.

Because $D$ is predictive-normalized, the lower spectral bound of Lem.~\ref{lem:mv-kl} gives
\[
 \sum_{j\le L}\|D_j\|_F^2\le \frac{(1+\rho)^2}{\sigma^2}.
\]
The recurrence coefficients range over a fixed compact subset of $\R^{4r}$.
The proof of Lem.~\ref{lem:mv-exp-l1l2}, together with the telescoping identity
$C^n-\widetilde C^n=\sum_{k=0}^{n-1}C^k(C-\widetilde C)\widetilde C^{n-1-k}$,
shows that the map from the recurrence coefficients and the first $4r$ core matrices to the whole coefficient sequence is Lipschitz in $\ell_2$, uniformly in $L$.
Standard volumetric nets for the two $k$-dimensional spatial subspaces, for the compact recurrence-coefficient set, and for the first $4r$ core matrices therefore give
\[
 \log N(\mathcal S_\Phi,\epsilon)
 \le C\{sd+rs^2+r\}\log(C/\epsilon).
\]
Here the factors of $\sigma$ cancel because both the radius of the normalized coefficient set and the coefficient accuracy corresponding to predictive accuracy $\epsilon$ scale as $\sigma^{-1}$.
Since $sd\le s(d-s)+s^2$ and $s^2\le rs^2$ for $r\ge1$, the displayed exponent is at most a constant multiple of $p_{\exp}$, proving the entropy bound in \eqref{eq:mv-exp-regular}.

For the envelope, any
$W\in\operatorname{cone}(\mathcal S_\Phi^0-\mathcal S_\Phi^0)$
is a scalar multiple of a difference of two elements of $\mathcal S_\Phi$ and hence is a linear combination of at most eight model points.
Thus the matrix-valued sequence $(W_j)$ is a sum of at most $8r\le8\bar r$ exponentials with rates bounded by $\rho_{\rm h}$.
Apply Lem.~\ref{lem:mv-exp-l1l2} in the Hilbert space of $d\times d$ matrices with Frobenius norm:
\[
 \sum_{j\le L}\|W_j\|_{\rm op}
 \le \sum_{j\le L}\|W_j\|_F
 \le C_{8\bar r,\rho_{\rm h}}
 \left(\sum_{j\le L}\|W_j\|_F^2\right)^{1/2}.
\]
The same lower spectral bound then gives
\[
 \sum_{j\le L}\|W_j\|_{\rm op}
 \le \frac{(1+\rho)C_{8\bar r,\rho_{\rm h}}}{\sigma}\sqrt{\cE(W)},
\]
which is \eqref{eq:mv-increment-envelope} with a constant independent of $L$.
\end{proof}

\begin{corollary}[Parametric forecast risk for bounded-head diagonal SSM charts]
\label{cor:mv-exp-erm}
Under the assumptions of Prop.~\ref{prop:mv-exp-regular}, let $\widehat A$ be the empirical-risk minimizer over a predictive chart contained in $\mathcal A_{\exp}(r,s,\rho_{\rm h})$ and containing $A^\star$.
There are $c,C$ depending only on $\bar r$, $\rho_{\rm h}$, and the transfer-stability margin such that, for every $x\ge1$, if
\[
 n=T-L\ge C(p_{\exp}+x),
\]
then, with probability at least $1-Ce^{-cx}$,
\begin{equation}
 \cE(\widehat A)
 \le C\sigma^2\frac{p_{\exp}+x}{n}.
 \label{eq:mv-exp-upper}
\end{equation}
In particular there is no additional context-length factor in the sample condition for this bounded-head exponentially stable class.
\end{corollary}

\begin{proof}
Prop.~\ref{prop:mv-exp-regular} gives the hypotheses of Thm.~\ref{thm:mv-smooth} with $p$ a constant multiple of $p_{\exp}$ and $\Lambda=O(1)$.
\end{proof}

The same class also contains a patch on which the $p_{\exp}/T$ scale is necessary up to constants.
This is useful because it shows that the concrete upper law is not only a consequence of an oversized entropy bound.

\begin{proposition}[Matching lower scale on an exponential patch]
\label{prop:mv-exp-lower}
Fix $r\le\bar r$ and distinct rates $\lambda_1,\ldots,\lambda_r\in(-\rho_{\rm h},\rho_{\rm h})$.
For $L\ge r$, the stable class \eqref{eq:mv-exp-class} contains a neighbourhood for which, whenever
\[
 T\ge C\{rs^2+s(d-s)\},
\]
one has
\begin{equation}
 \inf_{\widehat A}\sup_{A^\star}
 \E\,\cE(\widehat A)
 \ge c\sigma^2\frac{2s(d-s)+rs^2}{T},
 \label{eq:mv-exp-lower}
\end{equation}
where $c,C$ may depend on $\bar r$, $\rho_{\rm h}$, the fixed rates, and the transfer-stability margin, but not on $T,d,s$, or $\sigma$.
Since $r\le rs^2$ for $s\ge1$, the numerator in \eqref{eq:mv-exp-lower} is equivalent up to constants to $p_{\exp}$.
\end{proposition}

\begin{proof}
For the core term, let $K_L\in\R^{L\times r}$ have entries $(K_L)_{jq}=\lambda_q^{j-1}$.
Distinctness of the fixed rates implies that $K_r$ is a nonsingular Vandermonde matrix, hence
\[
 0<c_\lambda I_r\preceq K_L^\top K_L\preceq C_\lambda I_r
 \qquad(L\ge r),
\]
where the lower bound follows from $K_L^\top K_L\succeq K_r^\top K_r$ and the upper bound from the convergent infinite geometric Gram matrix.
Use the Hamming packing of Lem.~\ref{lem:flat-core}, but now set
\[
 A_j^{(v)}=\epsilon U_0\left(\sum_{q=1}^{r}S_q^{(v)}\lambda_q^{j-1}\right)V_0^\top,
 \qquad \epsilon^2=c_\epsilon/T,
\]
for fixed orthonormal $U_0,V_0$.
The Gram bound gives pairwise coefficient separation of order $\epsilon^2rs^2$.
Moreover
\[
 \sup_\omega\|\mathsf A^{(v)}(\omega)\|_{\rm op}
 \le \frac{\epsilon}{1-\rho_{\rm h}}\sum_{q\le r}\|S_q^{(v)}\|_{\rm op}
 \le \frac{\epsilon rs}{1-\rho_{\rm h}}.
\]
Because $r\le\bar r$, the stated sample condition and a sufficiently small $c_\epsilon$ keep the packing inside the transfer-stable class.
Lem.~\ref{lem:mv-kl}, the regular spectral lower bound, and Fano's inequality then give a lower bound $c\sigma^2rs^2/T$ exactly as in \eqref{eq:mv-core-lower}.

For the output-subspace term, fix one rate $\lambda_0\in(-\rho_{\rm h},\rho_{\rm h})$, fixed $V_0$, and a small constant $\gamma>0$, and replace the one-lag subclass in Thm.~\ref{thm:mv-subspace} by
\[
 A_j(X)=\gamma\lambda_0^{j-1}U(X)V_0^\top,
 \qquad U(X)=U_0+U_\perp X.
\]
For the same local Hamming packing of $X$, the squared coefficient separation is the one-lag separation multiplied by
$\sum_{j=1}^{L}|\lambda_0|^{2(j-1)}$, which lies between $1$ and $(1-\rho_{\rm h}^2)^{-1}$.
The transfer norm is at most
$\gamma\|U(X)\|_{\rm op}/(1-\rho_{\rm h})$, so choosing $\gamma$ small and the perturbation local keeps the whole packing stable.
The proof of Thm.~\ref{thm:mv-subspace} therefore gives $c\sigma^2s(d-s)/T$ inside the exponential class.
Interchanging the left and right subspaces gives the same lower bound for the input subspace.
Taking the maximum of the three packings and using that a maximum dominates one third of their sum yields \eqref{eq:mv-exp-lower}.
\end{proof}

For bounded $r$ and $n\asymp T$, Cor.~\ref{cor:mv-exp-erm} and Prop.~\ref{prop:mv-exp-lower} therefore give matching parametric forecast-risk order on a nondegenerate exponential shared-factor patch.
The constants are uniform in the context length, but no uniformity as $r\to\infty$ is claimed; that distinction is essential in view of the head-crowding phenomenon below.

\subsection{Degrees of freedom \texorpdfstring{$2s(d-s)+rs^2+rp_\theta$}{3}}
Collecting the counts, at locally identifiable regular points the parameter dimension is $rd^2+rp_\theta$ for unrestricted matrices, $r\{s(2d-s)+p_\theta\}$ for independent rank-$s$ matrices, and
\begin{equation}
 p_{\rm sm}=2s(d-s)+rs^2+rp_\theta
 \label{eq:mv-dof}
\end{equation}
for the shared-subspace class \eqref{eq:mv-shared} with $p_\theta$ learned temporal parameters per head.
The three terms count the two Grassmannians, the cores, and the temporal shapes, and the Grassmannian term appears once for the input subspace and once for the output subspace, which is where the factor two comes from.

For fixed $p_\theta$ and $s\ge1$, the $rp_\theta$ term is at most a fixed multiple of $rs^2$, so the intrinsic count \eqref{eq:mv-dof} and the lower bound \eqref{eq:mv-shared-lower} agree up to constants in their dependence on $d$, $r$, and $s$.
When in addition $\Lambda=O(1)$ and the sample thresholds of \eqref{eq:mv-sample-condition} and Thm.~\ref{thm:mv-smooth} hold, the upper and lower rates therefore match in order on regular charts.
No independent minimax packing over the temporal parameters is claimed.
In this sense \eqref{eq:mv-dof} is the intrinsic dimension governing the regular-chart upper law, while singular temporal parameterizations are treated separately below.

\subsection{Whitening temporal spans and raw Gram crowding}
The conditioning that Thm.~\ref{thm:mv-smooth} requires is a statement about the temporal span and not about the chosen coordinates.
Suppose the temporal Gram matrix $G_r$ is positive definite, which is exactly what fails at a head collision.
Writing $k(j)=(k_1(j),\dots,k_r(j))^\top$ and $\psi(j)=G_r^{-1/2}k(j)$, the sequences $\psi_1,\dots,\psi_r$ are orthonormal, every core combination can be rewritten with $\widetilde C=G_r^{1/2}C$, and the resulting map is an isometry.
The apparent ill-conditioning of a raw head basis therefore disappears under whitening, and the statistically meaningful quantity is the frequency leverage of the orthonormal span rather than the condition number of $G_r$.

That said, crowding is unavoidable once many smooth heads are used.  The bounded-head constants in Prop.~\ref{prop:mv-exp-regular} are not asserted to remain uniform as $r\to\infty$.
If the temporal parameter set is compact with covering numbers $\varepsilon^{-p_\theta}$ and the normalized kernel is Lipschitz in the parameter, then for any $r$ parameter values the normalized Gram matrix $\widetilde G_r$ has smallest eigenvalue tending to zero as $r$ grows, at a rate governed by $p_\theta$.
Two smooth heads with nearby parameters produce nearly parallel temporal profiles, so a raw parameterization must become singular, and this is why the regularity in Thm.~\ref{thm:mv-smooth} is imposed on the image.

\subsection{Head collisions, vanishing cores, and singular strata}
Three degeneracies leave the regular charts.
Two heads can collide, $\theta_q\to\theta_{q'}$, in which case the temporal profiles become parallel and the divided-difference basis of App.~\ref{app:state} would be needed to keep a well-conditioned chart.
A core can vanish, $C_q\to0$, in which case the corresponding temporal parameter becomes unidentifiable.
The predictive Jacobian can lose rank for gauge reasons, since the factorization \eqref{eq:mv-shared} is invariant under $U\mapsto UO$, $V\mapsto VO'$, $C_q\mapsto O^\top C_qO'$ for orthogonal $O,O'$, and a chart must fix that symmetry.

Thm.~\ref{thm:mv-smooth} assumes uniform separation from all three, so the minimax behaviour on the singular strata is open.
This is not a technical gap that better constants would close, since the intrinsic dimension itself changes on those strata, and \eqref{eq:mv-dof} is not the right count there.
We flag it because a trained multi-head model is not guaranteed to stay away from collisions, and the crowding statement above says that with many heads it cannot.

\subsection{Relation to low-rank VAR and Tucker time-series models}
The class \eqref{eq:mv-shared} is a Tucker factorization of the coefficient tensor $(A_j)$, with the two Grassmannians as the cross-variable factors and the temporal laws as the third mode.
Tensor-structured and supervised-factor vector autoregressions study closely related coefficient-tensor restrictions from an estimation standpoint \citep{wang2022tensorvar,huang2025supervised}, and low-rank tensor recovery supplies related geometric and restricted-isometry tools \citep{rauhut2017tensor}.
General likelihood theory for multidimensional causal time-series models is developed by \citet{bardet2009qMLE}.
The singular strata generated by temporal-head collisions are specific to the shared-head parameterization introduced here and are not treated by any of these.

Two mismatch statements complete the picture and mirror the scalar Thm.~\ref{thm:state-complexity}.
Using too few temporal heads leaves a residual determined by the singular tail of the mode-three unfolding of the true coefficient tensor, while using too small a shared spatial subspace leaves residuals from the mode-one and mode-two unfoldings.
Both are proved by the same low-rank comparison used in App.~\ref{app:prelim}, applied to the appropriate unfolding, and both hold under the matrix regular-spectrum condition of Thm.~\ref{thm:mv-geometry}.
The contribution here is therefore not the tensor model but the identification of which unfolding controls which architectural resource in forecast risk.

\section{Experimental Details}
\label{app:experiments}

\subsection{Architecture definitions}
Three model families appear in Fig.~\ref{fig:e11}, all on the same benchmark.
The benchmark target is $y_t=\sum_{j=1}^{L}K_jZ_{t+1-j}$ with $Z$ independent standard Gaussians and $L=64$, and the target kernel $K$ is one of the three described below.
The target carries no additive noise, so the memory shape is identified by the data and the fitted first-chaos decay can be compared directly with the truth.

The diagonal state-space model has $m$ modes with eigenvalues $\lambda_i=\rho\,\sigma(\theta_i)$ for a learned $\theta$, an input map $B$, and a readout $C$, so its impulse response is $\sum_i C_iB_i\lambda_i^{\,j-1}$, where $B_i$ and $C_i$ denote the $i$th components of those vectors, and the model is linear in the state.
Since the eigenvalues are $\rho$ times a sigmoid they lie strictly inside $(0,\rho)$, which matters for the boundary-matched control below.
The gated recurrence is $S_t=g(Z_t)\odot S_{t-1}+v(Z_t)$ with a bounded data-dependent gate $g$, a bounded injection $v$, and a linear readout.
It is nonlinear in and gated by the current input while remaining affine in the previous state, so it is a contextual affine-in-state recurrence rather than a general $F(S,Z)$.
This is a legitimate subclass of the class treated by Thm.~\ref{thm:first-chaos}: bounded injection and contraction produce a bounded forward-invariant state domain on which the derivative budgets can be verified at each fixed $\rho$. The experiment should not be read as testing a genuinely nonlinear-in-state recurrence, which we do not implement.
The attention model is a two-layer causal Transformer with model width $32$, four heads, feedforward width $64$, and learned relative positional biases.

\subsection{Width-independent derivative and amplitude budgets}
Thm.~\ref{thm:first-chaos} assumes uniform bounds on the three derivative blocks, so a model whose bounds grow with width would evade the conclusion for reasons unrelated to memory.
Every recurrent model is therefore constrained to width-independent budgets, enforced through the effective parameterization rather than by a penalty or by projecting raw parameters after each optimizer step.
Each budgeted parameter block is passed through a fixed $\ell_2$ cap before use, so the effective weights satisfy the budget at every step by construction.
The budgets used are $\|B\|\le2$, $\|C\|\le2$, gate norm at most $2$, value norm at most $2$, and input scale at most $4$, all independent of $m$.
The contraction $\rho$ is fixed by construction rather than learned, so it is an experimental variable and not an outcome.

This is the single most important protocol choice in the experiment.
Without it, a width sweep would confound two effects, since a wider model could reduce error either by using more modes or by using larger weights, and only the first is what Thm.~\ref{thm:state-complexity} and Thm.~\ref{thm:first-chaos} are about.

\subsection{Bounded nonlinear injection}
For the gated cell, $\|g(z)\|_\infty\le\rho<1$ and bounded injection $\|v(z)\|_\infty\le V_{\max}$ make the box $\mathcal B_\rho=\{s:\|s\|_\infty\le V_{\max}/(1-\rho)\}$ forward invariant. The input derivative is $D_zF(s,z)=D_zg(z)\odot s+D_zv(z)$, so the fixed parameter caps give a finite width-independent bound $\|D_zF\|\le K_\rho$ uniformly on $\mathcal B_\rho\times\R$ for each fixed $\rho$, rather than on all of $\R^m\times\R$.
The gate is likewise bounded, so the state Jacobian norm is at most $\rho$ pointwise on this domain and not merely in expectation, which is the uniform contraction the theorem requires. Because $K_\rho$ may vary with $\rho$, the stability sweep is interpreted qualitatively rather than as a numerical test of the exact asymptotic constant in Corollary~\ref{cor:critical}.

\subsection{Training, validation, and seed protocol}
Each configuration uses $12{,}000$ training, $3{,}000$ validation, and $6{,}000$ test windows, generated with distinct fixed data seeds so that the three splits are independent draws rather than a partition of one sample.
The state-space model is trained for $1500$ steps at learning rate $3\times10^{-2}$, the gated model for $900$ steps at $2\times10^{-2}$, and the attention model for $1000$ steps at $3\times10^{-3}$, all with batch size $128$.
Model selection uses the validation split and all reported numbers are on the test split.

Seeds are set before model construction so that initialization and data order are both controlled.
The protocol is versioned and the identifier is recorded in the summary file shipped with the archive, which is what ties the reported numbers to a specific configuration.
The width and stability sweeps use seeds $0$ through $3$, while the attention reference and each exponential control use seeds $0$ through $2$.
Attention checkpoints are cached under a hash of the full configuration including the seed and the step count, so changing any protocol field forces retraining rather than silently reusing an earlier run.

\subsection{Exact SSM impulse response and first-chaos estimation}
The first-chaos coefficients $\beta_j=\E[\widehat m_tZ_{t+1-j}]$ are obtained differently for linear and nonlinear models, and the distinction is worth stating because it removes estimation noise from one of the three curves in panel (a).
The state-space model is linear in the state, so its first-chaos memory is exactly its impulse response and is read off the parameters in closed form with no sampling at all.
The gated and attention models are nonlinear, so their coefficients are estimated by Monte Carlo on an independent diagnostic sample of $200{,}000$ standard Gaussian windows, using the empirical correlation between the output and each lagged input.

The fitted slopes are obtained by regressing $\log|\beta_j|$ on $\log j$ over lags $3$ to $30$, which is the range where the target is well above the diagnostic noise floor and below the truncation effects at $j$ near $L$. They are descriptive finite-window shape summaries only: an exponentially decaying sequence is not expected to have a constant log-log slope.

\subsection{Algebraic target results}
The algebraic target is $K_j=j^{-1.1}$, so $\alpha=1.1>1$ as Thm.~\ref{thm:first-chaos} requires, and its variance on the $L=64$ benchmark is $\sum_{j=1}^{L}K_j^2=1.485$.
The fitted first-chaos slopes are $-1.110$ for attention against a true exponent of $-1.1$, $-2.26$ for the diagonal state-space model, and $-2.56$ for the gated recurrence.
Attention therefore recovers the algebraic law almost exactly, while both contractive recurrences decay far faster than the target, which is the qualitative content of the envelope in \eqref{eq:first-chaos}.

The recovery by attention here should not be read across to a noisy forecasting setting.
The target in this benchmark is exactly a linear functional of the observed lag vector, so the coefficient law is identified by the data, which is not the case when the predictor sits near an innovation floor.

\subsection{Matched exponential-target specificity control}
Two exponential controls are run with the architecture, budgets, optimizer, and data splits held fixed and only the target changed.
The controls use the matched seeds $0$, $1$, and $2$, while the algebraic sweep additionally includes seed $3$, so the comparison is seed matched on the first three and not on the fourth.
The first is boundary matched with $\beta=0.8$ equal to the contraction $\rho=0.8$, and since the diagonal eigenvalues are $\rho$ times a sigmoid this target lies on the closure of the model class rather than in its interior.
Here $\beta$ denotes only the exponential-control decay used in the figure labels and is unrelated to the first-chaos sequence $(\beta_j)$ of Sec.~\ref{sec:stability}.
The second has $\beta=0.9>\rho$, so it decays more slowly than the envelope of Thm.~\ref{thm:first-chaos} and its tail leaves what a $\rho$-contractive state can carry, while remaining exponential.

The three targets have different variances, namely $1.485$, $2.778$, and $5.263$ for the algebraic, boundary-matched, and slower exponential targets, so raw mean squared error is not comparable across them.
All reported numbers and both sweeps therefore use the unexplained fraction of target variance, that is held-out error divided by the variance of the target.
This normalization is not cosmetic, since the raw comparison would understate the boundary-matched advantage and overstate the slower-exponential gap.

\subsection{Width sweep}
The width sweep fixes $\rho=0.8$ and varies $m$ over $\{2,4,8,16,32,64,128\}$.
At the largest width the diagonal state-space model attains an unexplained variance fraction of $6.8\times10^{-6}$ on the boundary-matched control, $1.1\times10^{-2}$ on the algebraic target, and $4.5\times10^{-2}$ on the slower exponential.
The gated recurrence shows the same ordering with a much smaller gap on the boundary-matched control, reaching $2.0\times10^{-3}$ there against $1.4\times10^{-2}$ on the algebraic target.
The attention reference is at $5.6\times10^{-4}$ on the algebraic target.

Two features matter more than the individual values.
For the diagonal state-space model the slower-exponential floor is essentially width independent, varying by about one percent from $m=2$ to $m=128$.
The algebraic curve for the same model, and both gated curves, improve appreciably at small widths and then plateau well above zero, the largest such improvement being about forty percent for the gated slower-exponential control.
The shape of the evidence is therefore an early-width gain followed by a plateau, and what it supports is that increasing width alone does not remove the floor under a fixed contraction and fixed derivative budgets, which is what an envelope obstruction rather than a capacity limit should look like.
An exponential target is simultaneously the easiest and the hardest of the three, which removes the distinction between power-law and exponential memory as the explanation and places the floor in the relation between the target tail and the contraction.

\subsection{Stability sweep}
The stability sweep fixes $m=16$ and varies $\rho$ over $\{0.5,0.7,0.8,0.9,0.95,0.99,0.995\}$ on the algebraic target.
All numbers here are again the unexplained fraction of target variance, the same metric plotted in panel (c) and used in the width sweep.
The diagonal state-space model improves monotonically from $6.2\times10^{-2}$ to $7.5\times10^{-6}$ across that range, a drop of nearly four orders of magnitude.
The gated recurrence improves from $6.8\times10^{-2}$ to about $2.0\times10^{-3}$ and then stops, with the values at $\rho=0.95$, $0.99$, and $0.995$ all near that level.

The upper axis of panel (c) is reparameterized by the margin $\delta=-\log\rho$, which is the quantity appearing in Corollary~\ref{cor:critical}.
We read the sweep as removal of the fixed-contraction obstruction at this finite context and not as evidence that near-criticality suffices. At the largest $\rho$, the contraction timescale exceeds $L=64$, so those points do not probe the untruncated asymptotic algebraic tail required for a direct numerical test of Corollary~\ref{cor:critical}. The gated residual floor can also include higher-chaos variance, as discussed in App.~\ref{app:stability}, and is not attributed uniquely to first-chaos tail mismatch.
The effect spans orders of magnitude, and means with standard errors are reported over four independent training seeds.

\subsection{Complete numerical tables}
The supplementary material contains the code and numerical summaries sufficient to regenerate both figures, including the archived per-configuration summary behind Fig.~\ref{fig:e11}; the panels of Fig.~\ref{fig:resource-laws} are recomputed from the closed-form and quadrature routines rather than read from a stored table.
For Fig.~\ref{fig:resource-laws} these are deterministic, and the reported items are the constructive upper bound and the Hankel lower bound against $r$ with the fitted slopes against $\sqrt r$ and their coefficients of determination, the exact fractional prediction gap against $L$ for each $d$ with the fitted log-log slopes, and the exact zero-memory risk against $d$ with its ratio to $(\pi^2/6)d^2$.

Two numerical safeguards are recorded with those tables.
The lower bound in panel (a) is computed from the singular values of a finite Hankel section of order $1200$.
Singular values below $10^3$ times machine epsilon times the leading singular value are discarded as noise, so the curve is truncated at that index rather than continued into the floating-point floor.
The fractional prediction gap is computed as a cancellation-free suffix sum of $-\log(1-d^2/(k-d)^2)$ with a leading-tail trigamma correction rather than by differencing log-Gamma values, which would lose several digits at the largest contexts.
The correction uses $-\log(1-u)=u+O(u^2)$ on the tail beyond a cutoff $K$, so it captures $d^2\psi_1(K+1-d)$ and omits $O(d^4\sum_{k>K}k^{-4})=O(d^4K^{-3})$; at the cutoff $K=2^{20}$ used here that residual is far below the plotted scale, but the term is a leading-tail correction and not an exact remainder.

For Fig.~\ref{fig:e11} the archived summary records, for every model, target, width, and contraction, the seed mean and standard error that are plotted, together with the fitted first-chaos slopes, the target variances used for normalization, and the seed lists.
Because the exponential controls use three seeds and the main sweeps use four, the seed list rather than a single global count is what makes each entry interpretable.

\end{document}